\documentclass{article}

     \PassOptionsToPackage{numbers, compress}{natbib}
\newcommand{\indep}{\perp \!\!\! \perp}

\usepackage[final]{neurips_2024}

\usepackage{paralist}
\usepackage[utf8]{inputenc} 
\usepackage[T1]{fontenc}    
\usepackage{hyperref}       
\usepackage{url}            
\usepackage{booktabs}       
\usepackage{amsfonts}       
\usepackage{nicefrac}       
\usepackage{microtype}      
\usepackage{xcolor}         
\usepackage{multirow} 
\usepackage{graphicx}
\usepackage{subcaption}
\usepackage{caption}
\usepackage{amsmath}
\usepackage{amssymb}
\usepackage{mathtools}
\usepackage{amsthm}
\usepackage{dsfont}
\usepackage{algorithmic, algorithm}
\usepackage{wrapfig}
\usepackage{floatflt}
\theoremstyle{plain}
\newtheorem{theorem}{Theorem}[section]
\newtheorem{proposition}[theorem]{Proposition}
\newtheorem{lemma}[theorem]{Lemma}

\theoremstyle{definition}

\newtheorem{assumption}[theorem]{Assumption}
\theoremstyle{remark}

\title{Causal Contrastive Learning \\ 
            for Counterfactual Regression Over Time}

\usepackage{authblk}
\author[1,2]{Mouad El Bouchattaoui
\thanks{\texttt{mouad.el-bouchattaoui@centralesupelec.fr}} 
}
\author[1]{Myriam Tami}
\author[2]{Benoit Lepetit}
\author[1]{Paul-Henry Cournède}

\affil[1]{Paris-Saclay University, CentraleSupélec, MICS Lab, Gif-sur-Yvette, France}
\affil[2]{Saint-Gobain, Paris, France}

\begin{document}

\maketitle

\begin{abstract}
Estimating treatment effects over time holds significance in various domains, including precision medicine, epidemiology, economy, and marketing. This paper introduces a unique approach to counterfactual regression over time, emphasizing long-term predictions. Distinguishing itself from existing models like Causal Transformer, our approach highlights the efficacy of employing RNNs for long-term forecasting, complemented by Contrastive Predictive Coding (CPC) and Information Maximization (InfoMax). Emphasizing efficiency, we avoid the need for computationally expensive transformers. Leveraging CPC, our method captures long-term dependencies in the presence of time-varying confounders. Notably, recent models have disregarded the importance of invertible representation, compromising identification assumptions. To remedy this, we employ the InfoMax principle, maximizing a lower bound of mutual information between sequence data and its representation. Our method achieves state-of-the-art counterfactual estimation results using both synthetic and real-world data, marking the pioneering incorporation of Contrastive Predictive Encoding in causal inference.
\end{abstract}
\section{Introduction}
\label{sect: intro}
It's vital in real-world applications to estimate potential responses, i.e., responses under hypothetical treatment strategies. Individuals show diverse responses to the same treatment, emphasizing the need to quantify individual response trajectories. This enables personalized interventions, enhancing decision-making efficacy. In medical contexts, precise response estimation enables tailored treatments for patients \citep{atan2018Deep-Treat, shalit2020ITPolicyClinical, mueller2023personalized}. This paper focuses on \emph{counterfactual regression over time}, estimating responses under hypothetical treatment plans based on individual records, including past covariates, responses, and treatment sequences up to the current prediction time \citep{robins2000MSM, robins-time-varying-exposures}. In addressing the challenges of this time-varying setting, we tackle:
\begin{inparaenum}[(1)]
\item \textbf{Time-dependent confounding }\citep{platt2009time}: confounders influenced by past treatment, impacting subsequent treatments and responses.
\item \textbf{Selection bias:} imbalanced covariate distributions across treatment regimes in observational data, requiring time-aware handling beyond methods in static settings \citep{robins2000MSM, schisterman2009overadjustment, lim2018RMSM}.
\item \textbf{Long-term dependencies}: enduring interdependencies among covariates, treatments, and responses, enabling long-range interactions \citep{choi2016retain, pham2017predicting}. 
\end{inparaenum}

Recent advancements in neural networks, such as Recurrent Marginal Structural Networks (RMSNs) \citep{lim2018forecasting}, Counterfactual Recurrent Networks (CRN) \citep{bica2020estimatingCRN}, and G-Net \cite{Li2021GNetAR}, have tackled these causal inference challenges. However, their reliance on RNNs limits their ability to capture long-term dependencies. Recent studies \cite{Melnychuk2022CausalTF} propose integrating transformers to better represent temporal dynamics. Rather than viewing this as a limitation of RNNs, we see it as an opportunity to emphasize their strengths. We design specific architectures for counterfactual regression over large horizons, avoiding complex, hard-to-interpret models like transformers. Our approach leverages the computational efficiency of RNNs, incorporating Contrastive Predictive Coding (CPC) \citep{oord2018representation, henaff2020data} for learning data history representations. This enhances model performance while maintaining efficiency, offering a compelling alternative to transformer-based approaches. Furthermore, we usually formulate identification assumptions of counterfactual responses over the original process history space (Appendix \ref{subsect: identif_ass_CCPC}). However, these assumptions may not hold in the representation space for arbitrary functions. Since identification often involves conditional independence, it applies when using an invertible representation function. Current models for time-varying settings \citep{lim2018forecasting, bica2020estimatingCRN, Melnychuk2022CausalTF} do not enforce representation invertibility. To address this, instead of adding complexity with a decoder, we \emph{implicitly} push the history process to be "reconstructable" from the encoded representation by maximizing Mutual Information (MI) between representation and input, following the InfoMax principle \cite{linsker1988selforga}, akin to Deep InfoMax \citep{DeepInfoMax2019}.

\section{Contributions}
Our approach is inspired by self-supervised learning using MI objectives \citep{DeepInfoMax2019}. We aim to maximize MI between different views of the same input, introducing counterfactual regression over time through CPC to capture long-term dependencies. Additionally, we propose a tractable lower bound to the original InfoMax objectives for more efficient representations. This is challenging due to the sequential nature and high dimensionality, marking a novelty. We demonstrate the importance of regularization terms via an ablation study. Previous work leveraging contrastive learning for causal inference applies only to the static setting with no theoretical grounding \citep{infomax_static2022}. To our knowledge, we frame for the first time the representation balancing problem from an information-theoretic perspective and show that the suggested adversarial game (Theorem \ref{thm: blancing_iclub}) yields theoretically balanced representations using the Contrastive Log-ratio Upper Bound (CLUB) of MI, computed efficiently. Key innovations of our Causal CPC model include:
\begin{inparaenum}[(1)]
    \item We showcase the capability of leveraging CPC to capture long-term dependencies in the process history using InfoNCE \citep{gutmann2010nce, gutmann2012nce, oord2018representation}, an unexplored area in counterfactual regression over time where its integration into process history modeling is not straightforward in causality.
    \item We enforce input reconstruction from representation by contrasting the representation with its input. Such quality is generally overlooked in baselines, yet it ensures that confounding information is retained, preventing biased counterfactual estimation.
    \item Applying InfoMax to process history while respecting its dynamic nature is challenging. We provide a tractable lower bound to the original InfoMax problem, also bringing theoretical insights on the bound's tightness.    
    \item We suggest minimizing an upper bound on MI between representation and treatment to make the representation non-predictive of the treatment, using the CLUB of MI \citep{cheng2020club}. This novel information-theoretic perspective results in a theoretically balanced representation across all treatment regimes.    
    \item By using a simple Gated Recurrent Unit (GRU) layer \citep{cho2014GRU} as the model backbone, we demonstrate that well-designed regularizations can outperform more complex models like transformers.
\end{inparaenum}
Finally, our experiments on synthetic data (cancer simulation \citep{PK-PD}) and semi-synthetic data based on real-world datasets (MIMIC-III \citep{johnson2016mimic}) show the superiority of Causal CPC at accurately estimating counterfactual responses.

\section{Related Work}
\label{sect: related_work}
\paragraph{Models for Counterfactual Regression Through Time} Traditionally, causal inference addresses time-varying confounders using Marginal Structural Models (MSMs) \citep{robins2000MSM}, which rely on inverse probability of treatment weighting \citep{robins-time-varying-exposures}. However, MSMs can yield high variance estimates, especially with extreme values, and are limited to pooled logistic regression, impractical for high-dimensional, dynamic data. RMSNs \citep{lim2018RMSM} enhance MSMs by integrating RNNs for propensity and outcome modeling. CRN \citep{bica2020estimatingCRN} employs adversarial domain training with a gradient reversal layer \citep{ganin2015unsupervised} to establish a treatment-invariant representation space, reducing bias induced by time-varying confounders. Similarly, G-Net \citep{Li2021GNetAR} combines g-computation and RNNs to predict counterfactuals in dynamic treatment regimes. Causal Transformer (CT) \citep{Melnychuk2022CausalTF} uses transformers to estimate counterfactuals over time and handles selection bias by learning a treatment-invariant representation via Counterfactual Domain Confusion loss (CDC) \citep{tzeng2015simultaneous}. These models, like ours, assume \emph{sequential ignorability} \cite{robins-time-varying-exposures}, in contrast to a body of work which does not fully verify our assumptions \citep{lopezmultipletreatments, Soleimani2017TreatmentResponseMF, schulam2017reliable, ranganath2018multiple, wang2019blessings,veitch2020adapting, Bica2020TimeSDecounf, hatt2021sequential, LuChengCMAproxynConf20121, kuzmanovic2021deconfounding,qian2021synctwin, seedat2022continuous, de2022predicting, cao2023estimating, hizli2023causal, JiangCF-GODE2023, berrevoets2021disentangled, chen2023multi, wu2023counterfactual, frauen2023estimating}, which we discuss in detail in Appendix \ref{appendix: Methods_viol_assp}. In contrast, we introduce a contrastive learning approach to capture long-term dependencies while maintaining a simple model and ensuring high computational efficiency in both training and prediction. This demonstrates that simple models with well-designed regularization terms can still achieve high prediction quality. Additionally, previous works \citep{robins-time-varying-exposures, robins2000MSM, lim2018RMSM, Li2021GNetAR, Melnychuk2022CausalTF} did not consider the role of invertible representation in improving counterfactual regression. Here, we introduce an InfoMax regularization term to make our encoder easier to invert. Appendix \ref{subsect: extended_work_cf} provides a detailed overview of counterfactual regression models.

\paragraph{InfoMax Principle} The InfoMax principle aims to learn a representation that maximizes MI with its input \citep{linsker1988selforga, bell1995information}. Estimating MI for high-dimensional data is challenging, often addressed by maximizing a simple and tractable lower bound on MI \citep{DeepInfoMax2019, VarBoundsMI2019, liang2023factorizedCL}. Another approach involves maximizing MI between two lower-dimensional representations of different views of the same input \citep{bachman2019learning, henaff2020data, tian2020contrastive, Tschannen2020OnMIMAX}, offering a more practical solution. We adopt this strategy by dividing our process history into two views, past and future, and maximizing a tractable lower bound on MI between them. This encourages a "reconstructable" representation of the process history. To our knowledge, the only work applying an InfoMax approach to counterfactual regression, albeit in static settings, is \cite{infomax_static2022}. They propose maximizing MI between an individual's representation and a global representation, aggregating information from all individuals into a single vector. However, the global representation lacks clarity and interpretability, raising uncertainties about its theoretical underpinnings in capturing confounders. Furthermore, there's a lack of theoretical analysis on how minimizing MI between individual and treatment-predictive representations could yield a treatment-invariant representation. As a novelty, we extend the InfoMax principle to longitudinal data, providing a theoretical guarantee of learning balanced representations. Appendix \ref{subsect: extended_work_mi_ss} discusses self-supervision and MI, with all proofs in Appendix \ref{sect: proofs}.

\section{Problem Formulation}
\label{sect: prob_formulation} 
\begin{wrapfigure}[12]{r}{0.35\textwidth}
     \centering
    \includegraphics[width = 0.33\textwidth, height = 0.13\textheight]{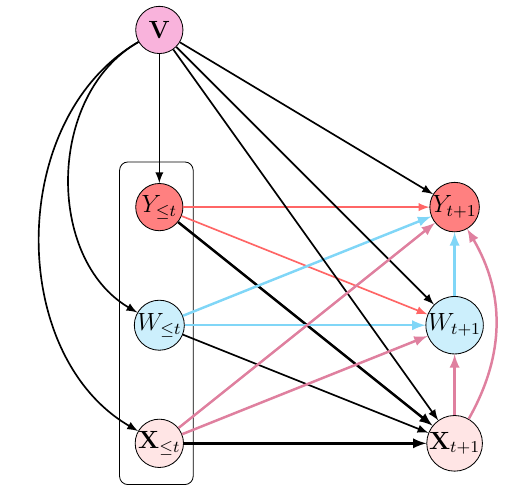}
    \caption{Causal graph over $\mathbf{H}_{t+1}$}
  \label{fig: causal_graph}
\end{wrapfigure}
\paragraph{Setup} In the framework of Potential Outcomes (PO) \citep{rubin2005causal}, and following \cite{robins2009estimation}, we track a cohort of individuals (units) $i \in \{1,2,\dots, N\}$ over $t_{max}$ time steps. At each time $t \in \{1,2,\dots,t_{max}\}$, we observe the following:
\begin{inparaenum}[(1)]
    \item \textbf{Discrete treatment} $W_{it} \in \mathcal{W} = \{0,1, \dots, K-1\}$, e.g., in medical contexts, $W_{it}$ may represent treatments like radiotherapy or chemotherapy.
    \item \textbf{Outcome of interest} $Y_{it} \in \mathcal{Y} \subset \mathbb{R}$, such as tumor volume.
    \item \textbf{Time-varying context} $\mathbf{X}_{it} \in \mathcal{X} \subset \mathbb{R}^{d_x}$, containing information about the individual that may influence treatment decisions and outcomes. $\mathbf{X}_{it}$ is a $d_x$-dimensional vector of confounders, such as health records or clinical measurements.
    \item \textbf{Static confounders} $\mathbf{V} \in \mathcal{V} \subset \mathbb{R}^{d_v}$, such as gender, which remain constant over time.
    \item \textbf{Partially observed potential outcomes} $Y_{it}(\omega_{i, \leq t})$, representing the outcomes that \emph{would have been} observed for individual $i$ at time $t$ under treatment sequence $\omega_{i, \leq t} = (\omega_{i,1}, \dots, \omega_{i,t}) \in \mathcal{W}^t$.
\end{inparaenum}
We define the history process up to time $t+1$ as $\mathbf{H}_{t+1} = [\mathbf{V}, \mathbf{X}_{\leq t+1}, W_{\leq t}, Y_{\leq t}]$, capturing all information prior to the assignment of treatment $W_{t+1}$. This history is illustrated in the causal graph shown in Figure \ref{fig: causal_graph}.

\paragraph{Goal} Given a training dataset $\{\mathbf{H}_{i, t+1}, i=1,\dots, N \}$ sampled from the empirical distribution $\mathbb{P}_{\mathbf{H}_{t+1}}$, we address the following causal inference problem: \emph{Given a history process $\mathbf{H}_{t+1}$, how can we efficiently estimate counterfactual responses up to time $t + \tau$ (where $\tau \geq 1$ is the prediction horizon) for a potential treatment sequence $\omega_{t+1:t+\tau} = (\omega_{t + 1}, \dots, \omega_{t + \tau})$?} The goal is to estimate the causal quantity: \(\mathbb{E}(Y_{t+\tau}(\omega_{t+1:t+\tau})\mid \mathbf{H}_{t+1}),\) i.e., the expected outcome at time $t + \tau$, given the history $\mathbf{H}_{t+1}$ and a sequence of treatments $\omega_{t+1:t+\tau}$. We identify this causal quantity from observational data using the assumption of sequential ignorability \citep{robins-time-varying-exposures, lim2018forecasting, Li2021GNetAR, Melnychuk2022CausalTF}, which is implicitly assumed in Figure \ref{fig: causal_graph} and explicitly discussed in Appendix \ref{subsect: identif_ass_CCPC}. This allows us to express the counterfactual as:
\[
\mathbb{E}(Y_{t+\tau}(\omega_{t+1:t+\tau})\mid \mathbf{H}_{t+1}) = \mathbb{E}\left(Y_{t+\tau}\mid \mathbf{H}_{t+1}, W_{t+1:t+\tau} = \omega_{t+1:t+\tau}\right).
\]
\section{Causal CPC}
\label{sect: causal_cpc_def}
\subsection{Representation Learning}
\paragraph{Contrastive Predictive Coding} We employ contrastive learning to efficiently represent the process history $\mathbf{H}_t$. Causal forecasting over multiple time horizons requires representations that capture variability in $\mathbf{H}_t$. For short-term predictions, local features and smooth signal variations are critical, while long-term predictions rely on capturing global structures and long-term dependencies, as shared information between history and future points diminishes.

To achieve this, we learn a representation of $\mathbf{H}_t$ that predicts future components over multiple time steps. For each horizon $j = 1, \dots, \tau$, future components are defined as $\mathbf{U}_{t+j} = [\mathbf{V}, \mathbf{X}_{t+j}, W_{t+j-1}, Y_{t+j-1}] \in \mathcal{U} \subset \mathbb{R}^{(d_v+d_x+K+1)}$. First, local features are extracted by encoding $[\mathbf{V}, \mathbf{X}_{t}, W_{t-1}, Y_{t-1}]$ into $\mathbf{Z}_t = \Phi_{\theta_1}([\mathbf{V}, \mathbf{X}_{t}, W_{t-1}, Y_{t-1}])$. Then, the full process history $\mathbf{H}_t$ is summarized into a \emph{context representation} $\mathbf{C}_t$, given by an autoregressive model: $\Phi_{\theta_2}^{ar}(\mathbf{Z}_{\leq t}) = \mathbf{C}_t$, where $\Phi_{\theta_2}^{ar}$ is implemented with a GRU \citep{cho2014GRU}. This results in the representation function: $\Phi_{\theta_1, \theta_2}(\mathbf{H}_t) = \mathbf{C}_t$.

To train the model, we use a contrastive loss that encourages the context $\mathbf{C}_t$ to predict future local features $\mathbf{Z}_{t+1}, \dots, \mathbf{Z}_{t+\tau}$ while distinguishing them from the features of other individuals. This is done by minimizing the InfoNCE loss $\mathcal{L}^{(InfoNCE)}_j$ for each horizon $j$:
\begin{equation}
\label{eq: infonce_loss_cpc}
\resizebox{0.6\textwidth}{!}{$
\mathcal{L}^{(InfoNCE)}_j(\theta_1, \theta_2, \Gamma_j) \coloneqq -\mathbb{E}_{\mathcal{B}} \left[ \log \frac{\exp(T_j(\mathbf{U}_{t+j}, \mathbf{C}_t))}{\sum_{l=1}^{|\mathcal{B}|} \exp(T_j(\mathbf{U}_{l,t+j},\mathbf{C}_t))} \right],
$}
\end{equation}
where $\mathcal{B}$ is a batch containing individual histories, and $\Gamma_j$ is a weight matrix. The discriminator $T_j(.,.)$ classifies the correct future feature among negative samples from other individuals:
\begin{equation}
\label{eq: discrim_cpc}
\resizebox{0.5\textwidth}{!}{$
T_j(\mathbf{U}_{t+j}, \Phi_{\theta_1, \theta_2}(\mathbf{H}_t)) = \Phi_{\theta_1}(\mathbf{U}_{t+j})^T \Gamma_j \mathbf{C}_t = \mathbf{Z}_{t+j}^T \Gamma_j \mathbf{C}_t.
$}
\end{equation}
In practice, $\mathbf{C}_t$ predicts $\mathbf{\hat{Z}}_{t+j}$, and prediction quality is measured by the dot product $\mathbf{Z}_{t+j}^\top \mathbf{\hat{Z}}_{t+j}$. Minimizing the InfoNCE loss $\mathcal{L}^{(InfoNCE)}_j$ provides a lower bound on the MI between the context and future features $I(\mathbf{U}_{t+j}, \mathbf{C}_{t})$ \citep{oord2018representation}:
\begin{equation}
 \label{eq: LB_MI_InfoNCE}
    I(\mathbf{U}_{t+j}, \mathbf{C}_{t}) \geq \log(|\mathcal{B}|) - \mathcal{L}^{(InfoNCE)}_j.
\end{equation}
For multiple forecasting horizons $j = 1, 2, \dots, \tau$, we learn long-term dependencies by minimizing the InfoNCE loss across all horizons:
\begin{equation}
    \label{eq: cpc_loss}
    \resizebox{0.6\textwidth}{!}{$\mathcal{L}^{CPC}(\theta_1, \theta_2, \{\Gamma_j\}_{j=1}^{\tau}) \coloneqq \frac{1}{\tau} \sum_{j=1}^{\tau} \mathcal{L}^{(InfoNCE)}_j(\theta_1, \theta_2, \Gamma_j).$}
\end{equation}

Thus, minimizing $\mathcal{L}^{CPC}$ maximizes the \emph{shared information} between the context and future components as shown in Eq. (\ref{eq: MI_lb_avg}), pushing the model to capture the global structure of the process over large horizons—crucial for counterfactual regression across multiple time steps:
\begin{equation}
\label{eq: MI_lb_avg}
\resizebox{0.4\textwidth}{!}{$
\frac{1}{\tau} \sum_{j=1}^{\tau} I(\mathbf{U}_{t+j}, \mathbf{C}_{t}) \geq \log(|\mathcal{B}|) - \mathcal{L}^{CPC}.
$}
\end{equation}

\paragraph{InfoMax Principle} We introduce a regularization term to make the context representation of the process history $\mathbf{H}_{t}$ "reconstructable." We leverage the InfoMax principle to maximize the MI between $\mathbf{H}_{t}$ and the context $\mathbf{C}_t$. However, we avoid computing the contrastive loss between $\mathbf{C}_t$ and $\mathbf{H}_t$ for two main reasons. First, $\mathbf{H}_{t}$ is a high-dimensional sequence, making the loss computation very demanding. Secondly, we are still interested in incorporating inductive bias toward capturing global dependencies, this time by pushing any subsequence to be predictive of any future subsequence within $\mathbf{H}_t$. Hence, we divide the process history into two non-overlapping views, $\mathbf{H}_t^h\coloneqq\mathbf{U}_{1:t_0}$, $\mathbf{H}_t^f\coloneqq\mathbf{U}_{t_0+1:t}$ representing a \emph{historical subsequence} and a \emph{future subsequence} within the process history $\mathbf{H}_{t}$, with $t_0$ randomly chosen per batch. We then maximize the MI between the representations of these views, $\mathbf{C}_t^h$ and $\mathbf{C}_t^f$, resulting in a lower bound to the InfoMax objective as formulated below:
\begin{proposition}    
    $ I(\mathbf{C}_t^h,\mathbf{C}_t^f)\leq I(\mathbf{H}_{t}, (\mathbf{C}_t^h, \mathbf{C}_t^f))$.
    \label{prop: lb_infomax}
\end{proposition}

We provide an intuitive discussion of the inequality by providing an exact writing of the gap in \ref{prop: lb_infomax}: 
\begin{theorem}
\label{thm: tightness_lb_info}
 \begin{equation}
 \resizebox{0.8\textwidth}{!}{$
     I(\mathbf{H}_{t}; (\mathbf{C}_t^f, \mathbf{C}_t^h)) - I(\mathbf{C}_t^f, \mathbf{C}_t^h) = I(\mathbf{H}_t; \mathbf{C}_t^f \mid \mathbf{C}_t^h) + \mathbb{E}_{\mathbf{h}_t \sim \mathbb{P}_{\mathbf{H}_t}} \mathbb{E}_{\mathbf{c}_{t}^f \sim \mathbb{P}_{\mathbf{C}_{t}^f \mid \mathbf{h}_t}} \left[ D_{KL}[\mathbb{P}_{\mathbf{C}_{t}^h \mid \mathbf{h}_t} || \mathbb{P}_{\mathbf{C}_{t}^h \mid \mathbf{c}_{t}^f }] \right].
     \label{eq: tightness_lb_infomax}
     $}
 \end{equation}
\end{theorem}
Both terms on the RHS of Eq. (\ref{eq: tightness_lb_infomax}) are positive, providing an alternative proof to Proposition \ref{prop: lb_infomax}. When equality holds, it implies \( I(\mathbf{H}_t; \mathbf{C}_t^f \mid \mathbf{C}_t^h) = 0 \), indicating that \(\mathbf{H}_t\) is independent of \(\mathbf{C}_t^f\) given \(\mathbf{C}_t^h\). This suggests that \(\mathbf{C}_t^h\) retains sufficient information from \(\mathbf{H}_t\) that is predictive of \(\mathbf{C}_t^f\). The symmetry of MI also leads to the occurrence of the second term on the RHS when conditioning on \(\mathbf{C}_t^f\). The equality in Proposition \ref{prop: lb_infomax} implies \( \mathbb{P}_{\mathbf{C}_{t}^h \mid \mathbf{h}_t} = \mathbb{P}_{\mathbf{C}_{t}^h \mid \mathbf{c}_{t}^f } \), suggesting that \(\mathbf{C}_t^f\) efficiently encodes its subsequence while sharing maximum information with \(\mathbf{C}_t^h\).

By considering the proposed variant of the InfoMax principle, we can compute a contrastive bound to $I(\mathbf{C}_t^h,\mathbf{C}_t^f)$ more efficiently, as the random vectors reside in a low-dimensional space thanks to the encoding. We define a contrastive loss using InfoNCE similar to Eq. (\ref{eq: infonce_loss_cpc}): 
\begin{equation}
\label{eq: infomax_loss_cpc}
\resizebox{0.6\textwidth}{!}{$
\mathcal{L}^{(InfoMax)}(\theta_{1},\theta_{2}, \eta)  \coloneqq -\mathbb{E}_{\mathcal{B}} \left[ \log \frac{\exp(T_{\eta}(\mathbf{C}_t^f, \mathbf{C}_t^h))}{\sum_{l = 1}^{|\mathcal{B}|} \exp(T_{\eta}(\mathbf{C}_{l,t}^f, \mathbf{C}_t^h))} \right].
$}
\end{equation}
We use a non-linear discriminator $T_{\eta}$ parametrized by $\eta$ (detailed in Appendix \ref{sect: archi_ccpc}). The representation of the past subsequence $\mathbf{C}_t^{h}$ is mapped to a prediction of the future subsequence $\mathbf{\hat{C}}_t^{f}\coloneqq F_{\eta}(\mathbf{C}_t^h)$ and $T_{\eta} = \mathbf{C}_t^{f^{\top}} \mathbf{\hat{C}}_t^f$.

Theorem \ref{thm: tightness_lb_info} and Proposition \ref{prop: lb_infomax} justify using the loss in Eq. (\ref{eq: infomax_loss_cpc}) by showing that our InfoMax simplification provides a valid lower bound. Thus, the contrastive loss in Eq. (\ref{eq: infomax_loss_cpc}) is valid, as:
\begin{equation}
\log(|\mathcal{B}|) - \mathcal{L}^{(InfoMax)} \leq I(\mathbf{C}_t^h, \mathbf{C}_t^f) \leq I(\mathbf{H}_t, (\mathbf{C}_t^h, \mathbf{C}_t^f)).
\end{equation}
The "mental model" behind our regularization term comes from the MI, \( I(\mathbf{H}_t, (\mathbf{C}_t^h, \mathbf{C}_t^f)) = H(\mathbf{H}_t) - H(\mathbf{H}_t \mid (\mathbf{C}_t^h, \mathbf{C}_t^f)) \), which can be written using entropy.
Since the entropy term is constant and parameter-free, minimizing the conditional entropy \( H(\mathbf{H}_t \mid (\mathbf{C}_t^h, \mathbf{C}_t^f)) \geq 0 \) ensures that \( \mathbf{H}_t \) is almost surely a function of \( (\mathbf{C}_t^h, \mathbf{C}_t^f) \) (Appendix \ref{appendix: zero_cond_H}, Proposition \ref{prop: zero_cond_H}). When MI is maximized, the theoretical existence of such a function suggests that the learned context \( \mathbf{C}_t \) can decode and reconstruct \( \mathbf{H}_t \).
  
Beyond the idea of reconstruction, it was shown that the InfoNCE objective implicitly learns to invert the data's generative model under mild assumptions \citep{zimmermann2021contrastive}. Recent works \citep{daunhawer2023identifiability, liu2024revealing} extend this insight to multi-modal settings, which can reframe our InfoMax problem: $\mathbf{H}_t^h$ and $\mathbf{H}_t^f$ can be seen as two coupled modalities, allowing us to identify latent generative factors up to some mild indeterminacies (e.g rotations, affine mappings). We plan to extend multi-modal causal representation learning to the longitudinal setting, where we anticipate minimizing our InfoMax objective, in the limit of infinite data, will effectively invert the data generation process up to a class of indeterminacies that we conjecture to be broader and under weaker assumptions than those in current causal representation learning literature, given our focus is on causal inference rather than the identification of causal latent variables. We initiate a formal basis for this claim in Appendix \ref{appendix: infomax_inverting_dgp}.

\subsection{Balanced representation learning}
\paragraph{Motivation} Our goal is counterfactual regression, specifically estimating \( \mathbb{E}(Y_{t+\tau}(\omega_{t+1:t+\tau}) \mid \mathbf{H}_{t+1}) \). For simplicity, with \( \tau = 1 \), we estimate the potential outcome for a given treatment \( W_{t+1} = \omega_{t+1} \), where \( W_{t+1} \in \{0,1, \dots, K-1\} \), expressed as \( \mathbb{E}(Y_{t+1}(\omega_{t+1}) \mid \mathbf{H}_{t+1}) \), which under standard assumptions is identified as:
\[
\mathbb{E}(Y_{t+1}(\omega_{t+1}) \mid \mathbf{H}_{t+1}) = \mathbb{E}(Y_{t+1} \mid \mathbf{H}_{t+1}, W_{t+1} = \omega_{t+1}).
\]
The RHS can be estimated from data as \( \mathbb{E}(Y_{t+1} \mid \mathbf{H}_{t+1}, W_{t+1}) = f(\mathbf{H}_{t+1}, W_{t+1}) \). Since only one treatment is observed per individual at each time step, \( W_{i,t+1} = \omega_{i,t+1} \), our model \( \hat{f} \) generates counterfactual responses by switching treatments \( \hat{f}(\mathbf{h}_{i,t+1}, \omega_{t+1}') \), where \( \omega_{t+1}' \neq \omega_{t+1} \) (e.g., chemotherapy vs. radiotherapy). The challenge is that \( \mathbf{H}_{t+1} \) and \( W_{t+1} \) are not independent, introducing potential bias in counterfactual estimation \citep{robins1999association}, leading to covariate shift or selection bias. To address this, we learn a representation \( \Phi(\mathbf{H}_{t+1}) \) that enforces distributional balance during decoding.

\paragraph{Setup} To mitigate selection bias, we leverage the context representation \( \mathbf{C}_t \) of \( \mathbf{H}_t \) and introduce two sub-networks: one for response prediction and one for treatment prediction, both using a mapping of the context representation:
\[
    \mathbf{\Phi}_t = \mathrm{SELU}(\mathrm{Linear}(\mathbf{C}_t)) = \Phi_{\theta_R}(\mathbf{H}_t),
\]
where SELU represents the Scaled Exponential Linear Unit \citep{klambauer2017self}, and \( \theta_R \) denotes all parameters of the representation learner, i.e., \( \theta_R = [\theta_1, \theta_2] \). Following \citep{bica2020estimatingCRN, Melnychuk2022CausalTF}, our objective is to learn a representation that accurately predicts outcomes while remaining \emph{distributionally balanced} across all possible treatment choices \( W_t = 0,1,\dots, K-1 \). To achieve this, we frame the problem as an adversarial game: one network learns to predict the next treatment from the representation, while a regularization term ensures that the representation is non-predictive of the treatment.
\setlength{\textfloatsep}{5pt} 
\begin{figure*}
\begin{center}
\includegraphics[width=\linewidth]{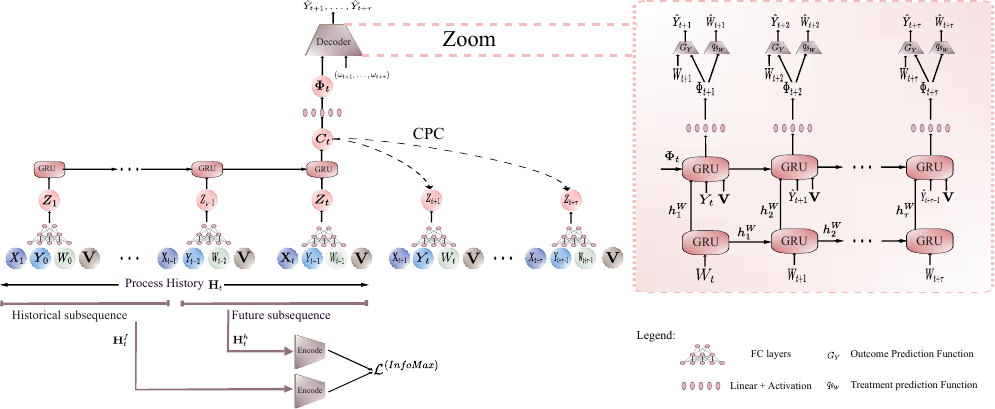}
\caption{Causal CPC architecture: The left shows the encoder, which learns context $\mathbf{C}_t$ from process history $\mathbf{H}_t$, with CPC and InfoMax objectives used for pretraining. The right shows the decoder, which autoregressively predicts the future outcome sequence from $\mathbf{C}_t$.}
\label{fig: ccpc_archi}
\end{center}
\end{figure*}

\paragraph{Factual response prediction} Since we intend to predict counterfactual responses for \(\tau\) steps ahead in time, we train a decoder to predict the factual responses \( Y_{t+1}, \dots, Y_{t+\tau} \) given the sequence of treatments \( (W_{t+1}, \dots, W_{t+\tau}) \). We minimize the negative conditional likelihood
\begin{align*}
 \mathcal{L}_Y(\theta_R, \theta_Y)  &= -\log p_{\theta_Y}(y_{t+1:t+\tau} \mid  \mathbf{\Phi}_t, \omega_{t+1:t+\tau}) \\
 &=  - \sum_{j=1}^{\tau}\log p_{\theta_Y}(y_{t+j} \mid y_{t+1:t+j-1}, \mathbf{\Phi}_t, \omega_{t+1:t+j}). 
\end{align*}

We denote $\mathcal{I}_t^j \coloneqq [Y_{t+1:t+j-1}, \mathbf{\Phi}_t, W_{t+1:t+j}]$ and assume a Gaussian distribution for the conditional responses $Y_{t+j} \mid \mathcal{I}_t^j \sim \mathcal{N}(G_Y(\mathcal{I}_t^j), \sigma^2)$, where $G_Y$ models the mean of the conditional response (see right side of Figure \ref{fig: ccpc_archi}). We set $\sigma = 0.05$ throughout our experiments. The response sequence is estimated autoregressively using a GRU-based decoder without teacher forcing \citep{williams1989learning} to ensure model training's consistency with testing in real-world scenarios (Figure \ref{fig: ccpc_archi} and Algorithm \ref{alg: dec_training} in Appendix \ref{sect: alg_ccpc}).

\paragraph{Treatment prediction} We learn a treatment prediction sub-network parameterized by $\theta_W$ that takes as input the representation $\mathbf{\Phi}_{t+1}$ and predicts a distribution \( q_{\theta_W}(\omega_{t+1} \mid \mathbf{\Phi}_{t+1}) \) over the treatment \( W_{t+1} \) by minimizing the negative log-likelihood, \( \mathcal{L}_{W} = -\log q_{\theta_W}(\omega_{t+1} \mid \mathbf{\Phi}_{t+1}) \). To assess the quality of the representation in predicting the treatment, the gradient from \( \mathcal{L}_{W} \) only updates the treatment network parameters \( \theta_W \) and is not backpropagated through the response of the parameters for the representation \( \mathbf{\Phi}_{t+1} \) (Algorithm \ref{alg: dec_training}, Appendix \ref{sect: alg_ccpc}).

\paragraph{Adversarial learning} To create an adversarial game, we update the representation learning parameters, and in the next step, the treatment network \( q_{\theta_W}(\cdot \mid \mathbf{\Phi}_{t+1}) \) with adverse losses such that the representation \( \mathbf{\Phi}_{t+1} \) becomes invariant with respect to the assignment of \( W_{t+1} \). Different from SOTA models (as highlighted in related work) and in line with our information guidelines principles, learning a balanced representation \( \mathbf{\Phi}_{t+1} \) amounts to ensuring \( \mathbf{\Phi}_{t+1} \indep W_{t+1} \), which is equivalent to \( I(\mathbf{\Phi}_{t+1}, W_{t+1}) = 0 \). Hence, we minimize the MI as a way to confuse the treatment classifier. Specifically, we minimize an upper bound on \( I(\mathbf{\Phi}_{t+1}, W_{t+1}) \), namely the CLUB of MI \citep{cheng2020club}.
\begin{equation}
\begin{aligned}
I_{\text{CLUB}}(\Phi(\mathbf{H}_t), W_{t+1}; q_{\theta_W}) &\coloneqq  \mathbb{E}_{\mathbb{P}_{(\Phi(\mathbf{H}_t), W_{t+1})}} 
\left[ \log q_{\theta_W}(W_{t+1} \mid \Phi(\mathbf{H}_{t+1})) \right] \\
&\quad - \mathbb{E}_{\mathbb{P}_{\Phi(\mathbf{H}_t)}}\mathbb{E}_{\mathbb{P}_{W_{t+1}}} 
\left[ \log q_{\theta_W}(W_{t+1} \mid \Phi(\mathbf{H}_{t+1})) \right].
\end{aligned}
\label{eq: ICLUB}
\end{equation}
We use the objective in Eq. (\ref{eq: ICLUB}) to update the representation learner $\Phi(.)$ \citep{brakel2017learning, DeepInfoMax2019}. This update aims to minimize the discrepancy between the conditional likelihood of treatments for units sampled from $\mathbb{P}_{(\mathbf{H}_t, W_{t+1})}$ and the conditional likelihood of treatments under the assumption of independent sampling from the product of marginals $\mathbb{P}_{\mathbf{H}_{t+1}} \otimes \mathbb{P}_{W_{t+1}}$. In practice, we generate samples from the product of marginals by shuffling the treatment $W_{t+1}$ across the batch dimension similar to \cite{brakel2017learning, DeepInfoMax2019}.

When minimizing $\mathcal{L}_W$, $q_{\theta_W}(\omega_{t+1}\mid \mathbf{\Phi}_{t+1})$ gets closer to the true conditional distribution $p(\omega_{t+1}\mid \mathbf{\Phi}_{t+1})$, and, in this case, the objective in Eq. (\ref{eq: ICLUB}) provides an upper bound of the MI between representation and treatment. We formalize the intuition by adapting the result of \cite{cheng2020club}:
\begin{theorem}\cite{cheng2020club} Let $q_{\theta_W}(\mathbf{\Phi}_{t+1}, \omega_{t+1})  \coloneqq q_{\theta_W}(\omega_{t+1}|\mathbf{\Phi}_{t+1})p(\mathbf{\Phi}_{t+1})$ be the joint distribution induced by $q_{\theta_W}(\omega_{t+1}|\mathbf{\Phi}_{t+1})$ over the representation space of $\mathbf{\Phi}_{t+1}$. If:
\begin{equation*}
\scalebox{0.9}{$
D_{KL}(p(\mathbf{\Phi}_{t+1}, \omega_{t+1})||q_{\theta_W}(\mathbf{\Phi}_{t+1}, \omega_{t+1})) 
\leq D_{KL}(p(\mathbf{\Phi}_{t+1}) p(\omega_{t+1})||q_{\theta_W}(\mathbf{\Phi}_{t+1}, \omega_{t+1})),
$}
\end{equation*}  
then $I(\mathbf{\Phi}_{t+1}, W_{t+1}) \leq I_{CLUB}(\mathbf{\Phi}_{t+1}, W_{t+1}; q).$
\label{thm: ICLUB_bound}
\end{theorem}
Based on Theorem \ref{thm: ICLUB_bound}, our adversarial training is interpretable and can be explained as follows: the treatment classifier seeks to minimize $\mathbb{E}_{\mathbb{P}_{(\mathbf{H}_t, W_{t+1})}}\left[\mathcal{L}_W\right]$, which is equivalent to minimizing Kullback-Leibler divergence $D_{KL}(p(\mathbf{\Phi}_{t+1}, \omega_{t+1})||q_{\theta_W}(\mathbf{\Phi}_{t+1}, \omega_{t+1}))$. Therefore, $q_{\theta_W}(\mathbf{\Phi}_{t+1}, \omega_{t+1})$ could get closer to $p(\mathbf{\Phi}_{t+1}, \omega_{t+1})$ than, ultimately, to $p(\mathbf{\Phi}_{t+1}) p(\omega_{t+1})$, as we train the network to predict $W_{t+1}$ from $\mathbf{\Phi}_{t+1}$. In such a case and by Theorem \ref{thm: ICLUB_bound}, $I_{CLUB}$ provides an upper bound on the MI. Hence, in a subsequent step, we minimize $I_{CLUB}$ w.r.t the representation parameters, minimizing the MI $I(\mathbf{\Phi}_{t+1}, W_{t+1})$ and achieving balance. We theoretically formulate such behavior by proving in the following theorem that, at the Nash equilibrium of this adversarial game, the representation is exactly balanced across the different treatment regimes provided by $W_{t+1}$.
\begin{theorem}
Let $t\in \{1, 2, \dots, t_{max}\}$, $\Phi=\Phi_{\theta_R}$ and $q = q_{\theta_W}$ are, respectively, any representation and treatment network. Let $\mathbb{P}_{\Phi(\mathbf{H}_t)}$ be the probability distribution over the representation space and $\mathbb{P}_{\Phi(\mathbf{H}_t)\mid W_{t+1}}$ its conditional counterpart. Then, there exist $\Phi^*$ and $q^*$ such that:
\begin{equation}
\begin{aligned}
 \Phi^* &= \arg\min_{\Phi} I_{\text{CLUB}}(\Phi(\mathbf{H}_t), W_{t+1}; q^*)\\
 q^* &= \arg\max_{q} \mathbb{E}_{\mathbb{P}_{\Phi^*(\mathbf{H}_t)}} \left[ \log q(W_{t+1} \mid \Phi^*(\mathbf{H}_t)) \right].                  
\end{aligned}
\end{equation}
Such an equilibrium holds if and only if $\mathbb{P}_{\Phi(\mathbf{H}_t)\mid W_{t+1}=0} \!=\! \mathbb{P}_{\Phi(\mathbf{H}_t)\mid W_{t+1}=1} \!=\! \dots \!=\! \mathbb{P}_{\Phi(\mathbf{H}_t)\mid W_{t+1}=k-1}$, almost surely.
\label{thm: blancing_iclub}
\end{theorem}
Since we target multi-timestep forecasting, covariate balancing in the representation space extends beyond \( t+1 \). For simplicity, we presented it for \( t+1 \), but in practice, the adversarial game applies the balancing across all forecasting horizons (Algorithm \ref{alg: dec_training}). The theorem also holds for other horizons by replacing \( \Phi(\mathbf{H}_t) \) with \( \Phi(\mathbf{H}_{t + j -1}) \) and \( W_{t+1} \) with \( W_{t+j} \), for $2 \leq j \leq \tau$.

\paragraph{Causal CPC Training} The Causal CPC model is trained in two stages:
\begin{inparaenum}[(1)]
    \item \textbf{Encoder pretraining:} We first learn an efficient representation of the process history by minimizing loss:
    \begin{equation*}
    \resizebox{0.5\textwidth}{!}{
    $\mathcal{L}_{enc} = \mathcal{L}^{CPC}(\theta_{1}, \theta_{2}, \{\Gamma_j\}_{j=1}^{\tau}) + \mathcal{L}^{(InfoMax)}(\theta_{1}, \theta_{2}, \gamma)$}.
    \end{equation*}
    \item \textbf{Decoder training:} After pretraining, we fine-tune the encoder by optimizing the factual outcome and treatment networks in the adversarial game from Theorem \ref{thm: blancing_iclub}. Formally:
    \begin{equation*}
    \resizebox{0.6\textwidth}{!}{
    $\begin{aligned}
    & \min_{\theta_{R}, \theta_{Y}} \mathcal{L}_{dec}(\theta_{R}, \theta_{Y}, \theta_{W}) = \mathcal{L}_Y(\theta_{R}, \theta_{Y}) + I_{\text{CLUB}}(\Phi_{\theta_{R}}(\mathbf{H}_t), W_{t+1}; q_{\theta_{W}}), \\
    & \min_{\theta_{W}} \mathcal{L}_{W}(\theta_{W}, \theta_{R}) = -\mathbb{E}_{\Phi_{\theta_{R}}(\mathbf{H}_t)} \left[ \log q_{\theta_{W}}(W_{t+1} \mid \Phi_{\theta_{R}}(\mathbf{H}_t)) \right].
    \end{aligned}
    $}
    \end{equation*}
\end{inparaenum}

\section{Experiments}
\label{sect: experiments}
We compare Causal CPC with SOTA baselines: MSMs \cite{robins2000MSM}, RMSN \cite{lim2018RMSM}, CRN \cite{bica2020estimatingCRN}, G-Net \cite{Li2021GNetAR}, and CT \cite{Melnychuk2022CausalTF}. All models are fine-tuned via a grid search over hyperparameters, including architecture and optimizers. Model selection is based on mean squared error (MSE) on factual outcomes from a validation set, and the same criterion is used for early stopping. Further details on hyperparameters and training procedures are provided in Appendices \ref{sect: huperparams_details} and \ref{sect: exp_protocol}.

\subsection{Experiments with Synthetic Data}
\paragraph{Tumor Growth} 
We use the PharmacoKinetic-PharmacoDynamic (PK-PD) model \cite{PK-PD} to simulate responses of non-small cell lung cancer patients, following previous works \citep{lim2018RMSM, bica2020estimatingCRN, Melnychuk2022CausalTF}. We evaluate our approach on simulated counterfactual trajectories, varying the confounding level via the parameter \( \gamma \) (Appendix \ref{subsect: cancer_sim_descrip}). Unlike \cite{Melnychuk2022CausalTF}, who used larger datasets (10,000 for training, 1,000 for testing), we use a smaller, more challenging dataset (1,000 for training, 500 for testing) to reflect real-world data limitations. For long-horizon forecasting, we set the prediction horizon to 10 and evaluate two training sequence lengths: 60 and 40, with covariates of dimension 4.
\setlength{\textfloatsep}{5pt} 
\begin{figure}[ht]
  \centering
  \begin{subfigure}[b]{0.3\textwidth}
    \centering
    \includegraphics[width=\linewidth]{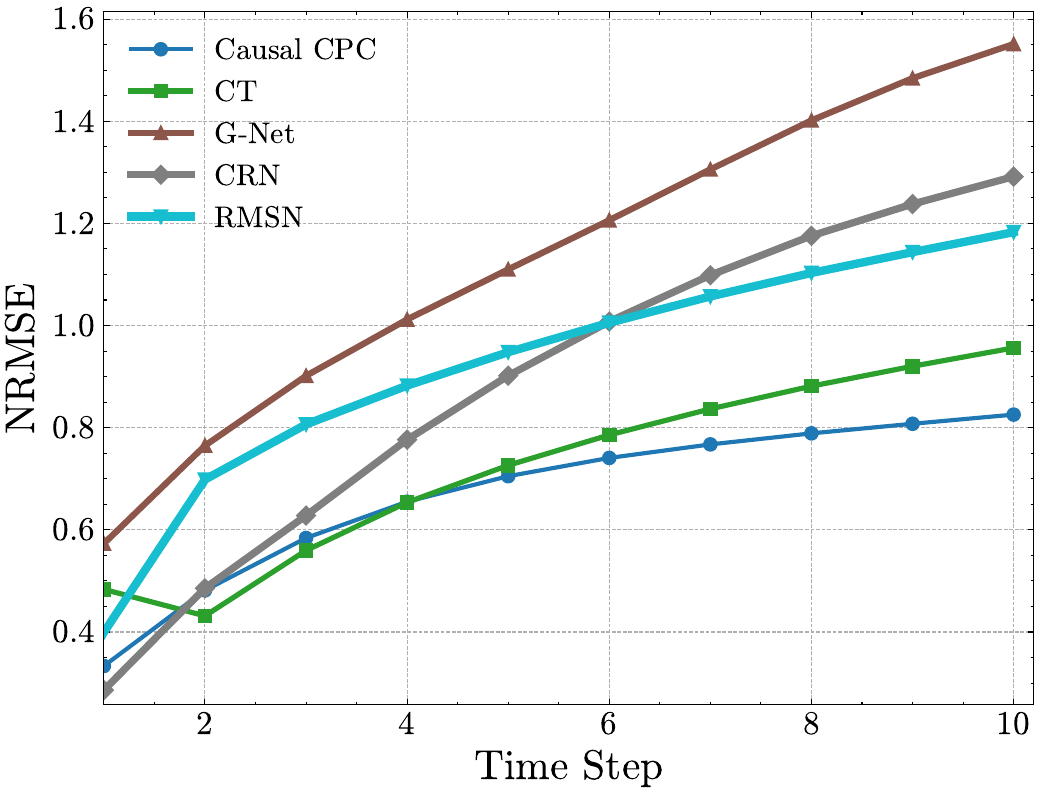}
    \caption{$\gamma =1$, sequence length 60}
    \label{fig:gamma1_seq60}
  \end{subfigure}%
  \hfill
  \begin{subfigure}[b]{0.3\textwidth}
    \centering
    \includegraphics[width=\linewidth]{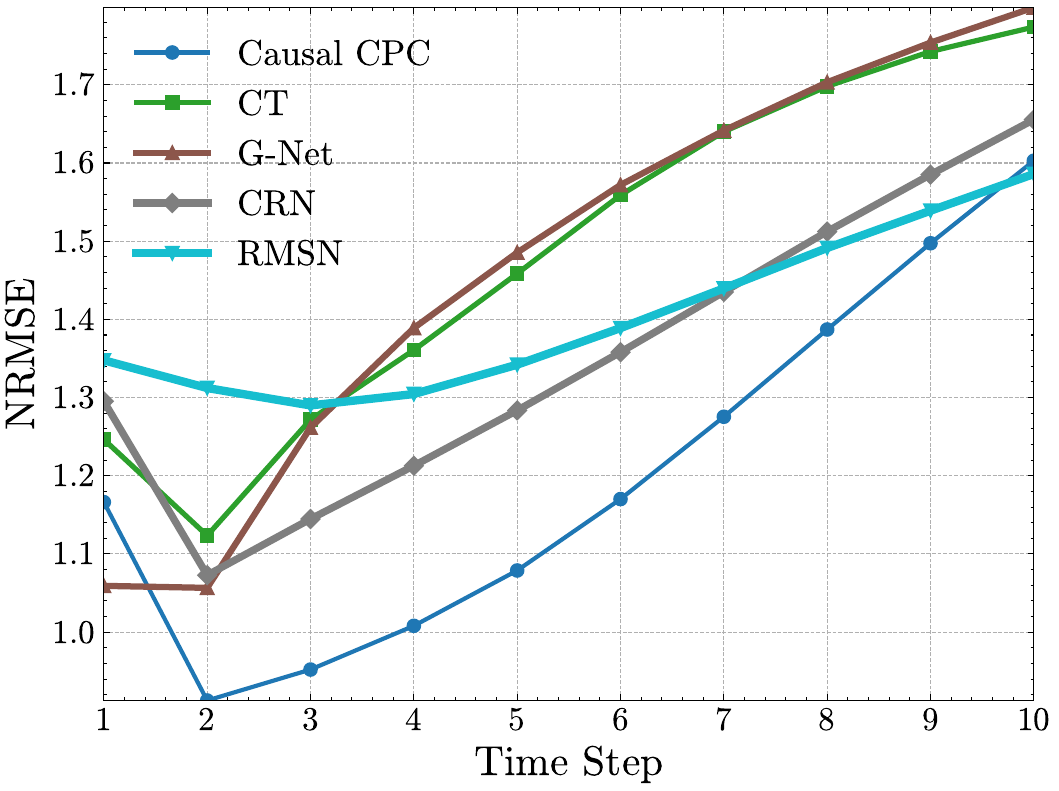}
    \caption{$\gamma =2$, sequence length 60}
    \label{fig:gamma2_seq60}
  \end{subfigure}%
  \hfill
  \begin{subfigure}[b]{0.3\textwidth}
    \centering
    \includegraphics[width=\linewidth]{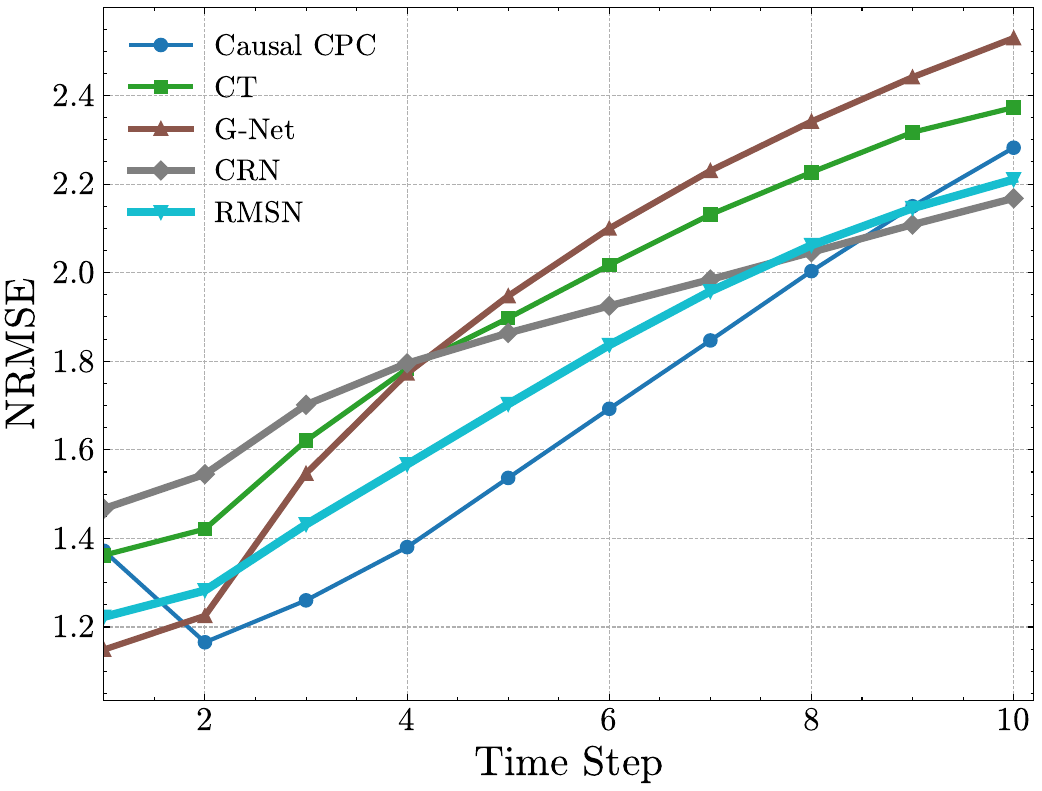}
    \caption{$\gamma =3$, sequence length 60}
    \label{fig:gamma3_seq60}
  \end{subfigure}
  
  \vspace{0.5cm} 

  \begin{subfigure}[b]{0.3\textwidth}
    \centering
    \includegraphics[width=\linewidth]{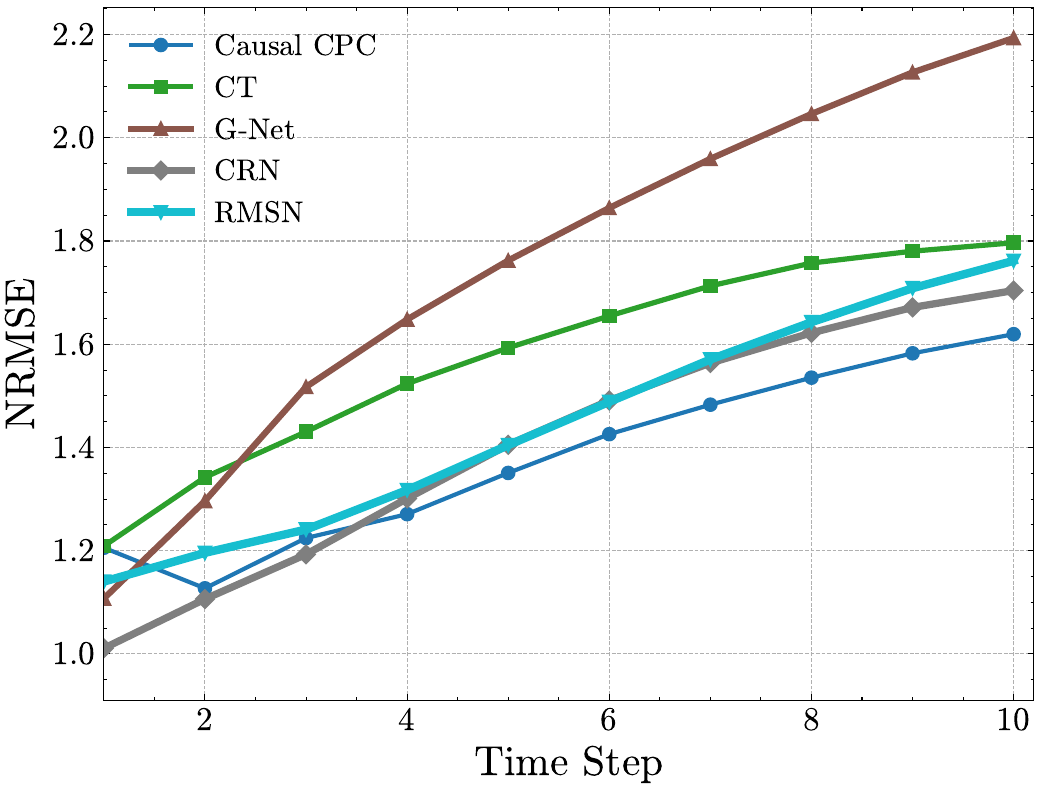}
    \caption{$\gamma =1$, sequence length 40}
    \label{fig:gamma1_seq40}
  \end{subfigure}%
  \hfill
  \begin{subfigure}[b]{0.3\textwidth}
    \centering
    \includegraphics[width=\linewidth]{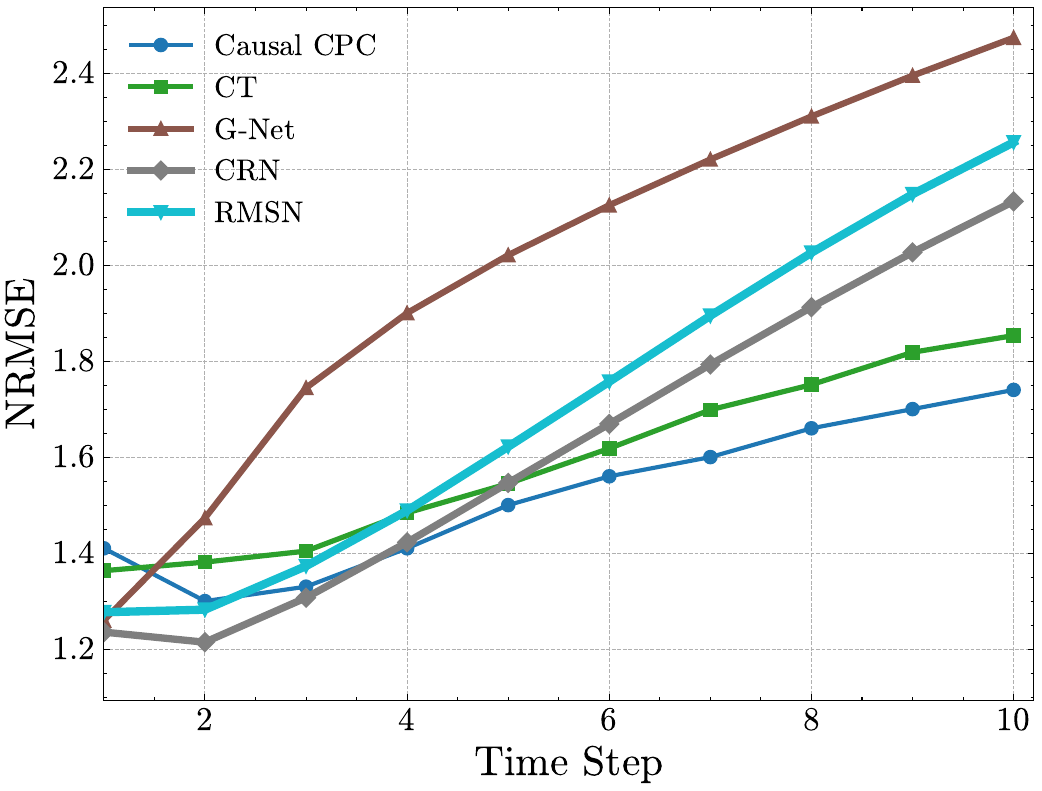}
    \caption{$\gamma =2$, sequence length 40}
    \label{fig:gamma2_seq40}
  \end{subfigure}%
  \hfill
  \begin{subfigure}[b]{0.3\textwidth}
    \centering
    \includegraphics[width=\linewidth]{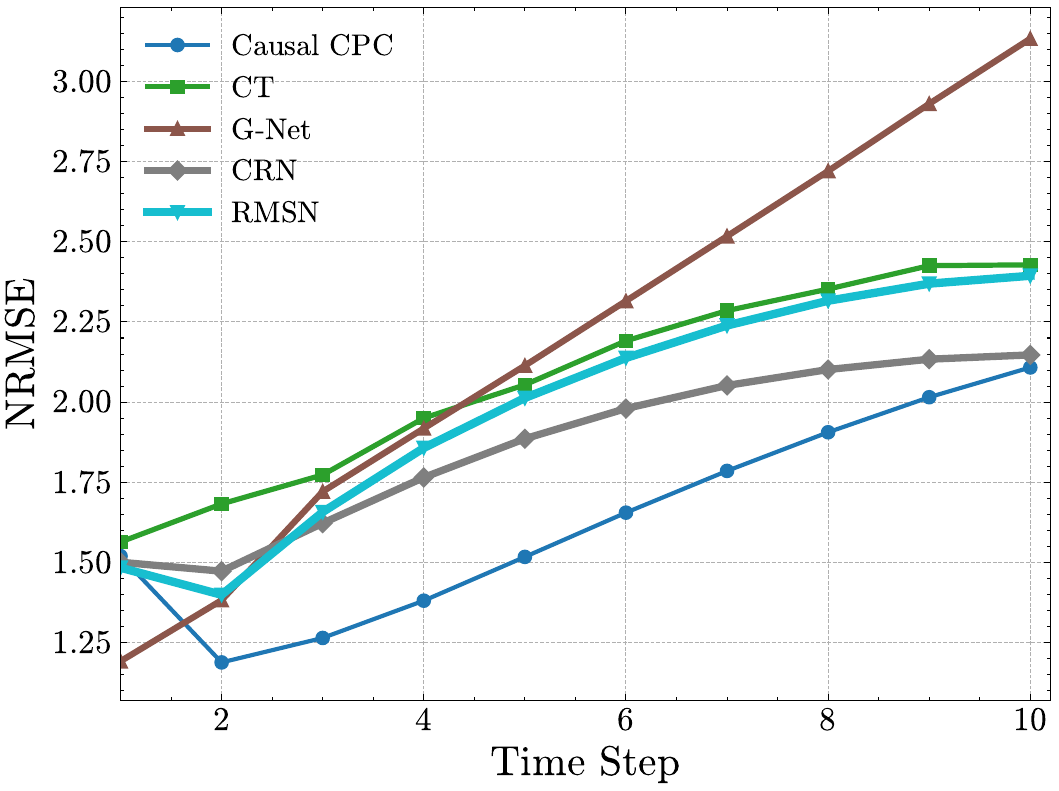}
    \caption{$\gamma =3$, sequence length 40}
    \label{fig:gamma3_seq40}
  \end{subfigure}
\caption{Evolution of error (NRMSE) in estimating counterfactual responses for cancer simulation data. Top: training sequence length 60. Bottom: training sequence length 40. In both cases, $\tau=10$. MSM is excluded due to high prediction errors.}
  \label{fig:combined_cancer_sim}
\end{figure}
\paragraph{Results} We tested all models on the cancer simulation data across three confounding levels, \( \gamma = 1, 2, 3 \). Figures \ref{fig:gamma1_seq60}, \ref{fig:gamma2_seq60} and \ref{fig:gamma3_seq60} show the evolution of Normalized Root Mean Squared Error (NRMSE) over counterfactual tumor volume as the prediction horizon increases. Causal CPC consistently outperforms all baselines at larger horizons, demonstrating its effectiveness in long-term predictions. This confirms the quality of \( \mathbf{C}_{t} \) in predicting future components across multiple time steps, capturing the global structure of the process as discussed in Eq. (\ref{eq: MI_lb_avg}). Extended results are provided in Appendix \ref{subsect: benchmark_few_data_cancer}. 
\begin{wrapfigure}[12]{r}{0.40\textwidth}
     \centering
    \includegraphics[width = 0.37\textwidth, height = 0.13\textheight]{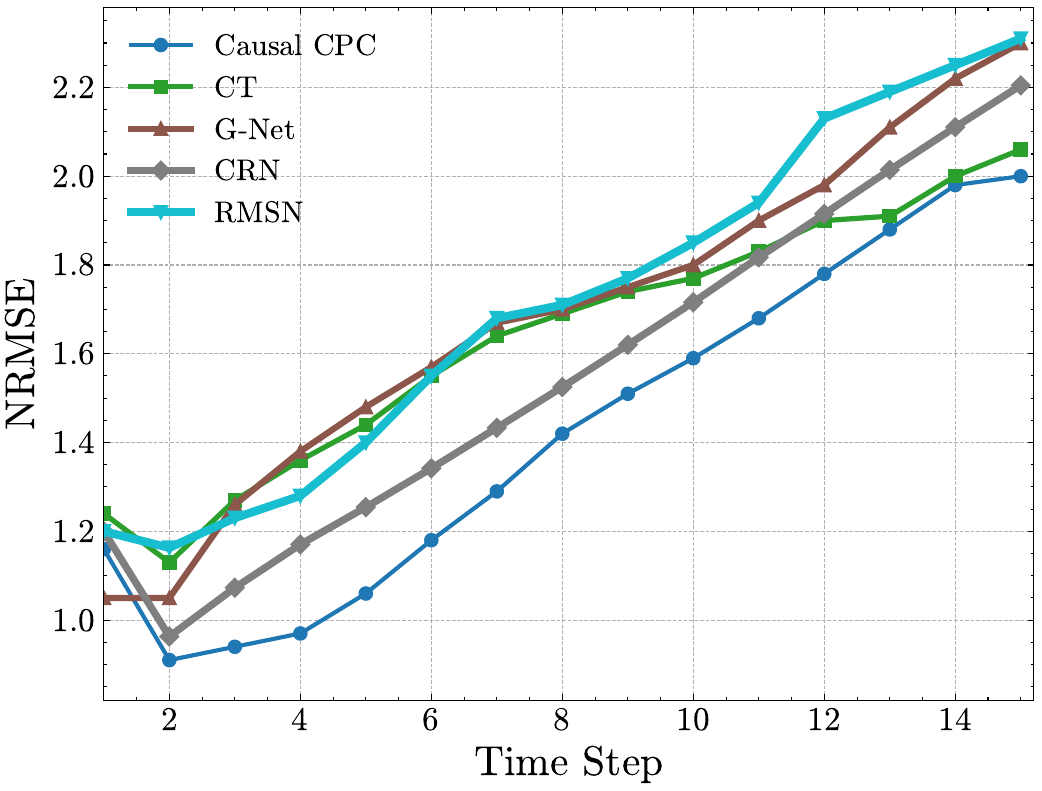}
    \caption{Models' performance for cancer simulation, $\gamma=2$, $\tau=15$.}
  \label{fig: gamma2_seq60_tau15}
\end{wrapfigure}
In the more challenging case where the maximum sequence length is 40 (Figures \ref{fig:gamma1_seq40}, \ref{fig:gamma2_seq40} and \ref{fig:gamma3_seq40}), the error evolution remains similar to Figure \ref{fig:gamma1_seq60}, \ref{fig:gamma2_seq60} and \ref{fig:gamma3_seq60}. Our model maintains its advantage, outperforming most baselines in long-term forecasting. However, Causal CPC does not outperform other models in short-term forecasting, a consistent limitation across experiments. Still, the model's superior long-term performance highlights its potential in applications requiring long-term accuracy. At higher levels of confounding, the model does not always outperform SOTA models at certain time steps. This may be due to the low dimensionality of the time-varying components and static covariates, \( \mathbf{U}_{t} = [\mathbf{V}, \mathbf{X}_{t}, W_{t-1}, Y_{t-1}] \), which have only four dimensions. Our model leverages contrastive learning-based regularization to excel on datasets with higher confounding dimensions, as demonstrated on MIMIC-III where \( \mathbf{U}_{t} \) has 72 dimensions. In this setting, our model consistently outperforms baselines at longer prediction horizons. The occasional underperformance of Causal CPC at the final horizon is due to $\tau=10$ being the last contrasted horizon, not an issue specific to $\tau=10$. To support this, we reran all models with a sequence length of 60, $\tau=15$, and $\gamma=2$. As shown in Figure \ref{fig: gamma2_seq60_tau15}, Causal CPC still outperforms SOTA for horizons beyond $\tau=10$ due to the encoder's retraining, where the InfoNCE loss is computed over all 15 time steps. The last prediction error remains close to SOTA, suggesting that training over larger horizons than initially intended may be beneficial.

\subsection{Experiments with semi-synthetic and real data}
\label{subsect: core_exp_mimic}
\paragraph{Semi-synthetic MIMIC-III} We used a semi-synthetic dataset constructed by \cite{Melnychuk2022CausalTF} based on the MIMIC-III dataset \citep{johnson2016mimic}, incorporating both endogenous temporal dependencies and exogenous dependencies from observational patient trajectories, as detailed in Appendix \ref{subsect: details_sim_mimic}. The patient trajectories are high-dimensional and exhibit long-range dependencies. Similar to the cancer simulation, the training data consisted of relatively few sequences (500 for training, 100 for validation, and 400 for testing). Table \ref{tab: perf_mimic_few_data} presents the mean and standard deviation of counterfactual predictions across multiple horizons ($\tau = 10$). We test two maximum sequence lengths, 100 and 60, to assess the models' robustness for long-horizon forecasting.
\begin{table*}[t]
 \caption{Evolution of RMSEs for the semi-synthetic MIMIC III, sequence length 100.}
 \centering
\resizebox{\textwidth}{!}{%
\begin{tabular}{|c |c |c |c |c |c |c |c |c |c |c|} 
\hline
Model       & $\tau = 1$   &   $\tau = 2$& $\tau = 3$&   $\tau = 4$&  $\tau = 5$&   $\tau = 6$& $\tau = 7$&  $\tau = 8$&   $\tau = 9$ & $\tau = 10$\\ \hline  
\textbf{Causal CPC (ours)}  &  0.32$\pm 0.04$&   0.45$ \pm 0.08$& 0.54$\pm 0.06$& 0.61 $\pm 0.10$&   \textbf{0.66$ \pm$ 0.10}&  \textbf{0.69$\pm$0.11}&   \textbf{0.71$ \pm$ 0.11}& \textbf{0.73$\pm $ 0.06}& \textbf{0.75 $\pm$ 0.05}& \textbf{0.77$ \pm$ 0.10} \\ 
\hline  
\textbf{CT} &  $0.42\pm0.38$&   \textbf{0.40$\pm$ 0.06}& \textbf{0.52$\pm$ 0.08}& \textbf{0.60$\pm$ 0.005}&  0.67$ \pm 0.10$  & 0.72 $\pm0.12$&   0.77$ \pm 0.13$ & 0.81$\pm 0.14$ &0.85 $\pm 0.16$&  0.88 $\pm $0.17   \\ 
\hline  
\textbf{G-Net} &  $0.54\pm0.13$&   0.72$ \pm 0.14$& 0.85 $\pm$0.16& 0.96 $\pm$ 0.17&  1.05 $\pm$ 0.18& 1.14 $\pm$0.18&   1.24$ \pm$ 0.17& 1.33$\pm 0.16$ & 1.41 $\pm$ 0.16&  1.49$ \pm $0.16\\ 
\hline  
\textbf{CRN} &  \textbf{0.27 $\pm$0.03}&   0.45$ \pm 0.08$& 0.58 $\pm$ 0.09& 0.72$\pm$ 0.11&   0.82$\pm$ 0.15& 0.92 $\pm$ 0.20&   1.00 $\pm$ 0.25& 1.06 $\pm$ 0.28& 1.12 $\pm$ 0.32&  1.17 $\pm$ 0.35\\ 
\hline  
\textbf{RMSN} &  0.40 $\pm$ 0.16&   0.70 $\pm$ 0.21& 0.80$\pm$ 0.19& 0.88 $\pm$ 0.17&   0.94 $\pm$ 0.16& 1.00 $\pm$ 0.15&  1.05 $\pm$ 0.14& 1.10 $\pm$ 0.14& 1.14 $\pm$ 0.13& 1.18 $\pm$ 0.13 \\ 
\hline  
 \end{tabular}
 }
\label{tab: perf_mimic_few_data}
\end{table*}
\vspace{-0.5cm} 
\begin{wrapfigure}[10]{r}{0.40\textwidth}
     \vspace{10pt}
     \centering
    \includegraphics[width = 0.37\textwidth, height = 0.10\textheight]{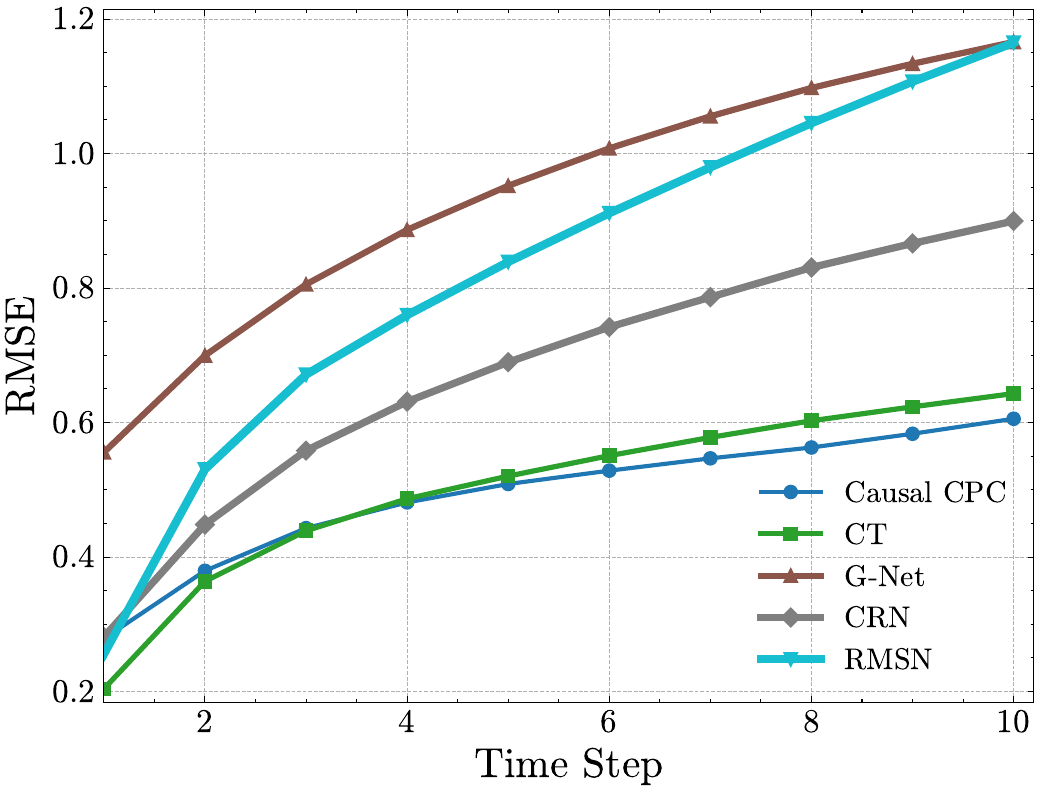}
    \caption{Performance for MIMIC III semi-synthetic, sequence length 60.}
  \label{fig: prf_mimic_seq_60}
\end{wrapfigure}
\paragraph{Results} Causal CPC consistently outperformed the baselines, especially at larger horizons, both with a sequence length of 100 and a reduced length of 60 (Figures \ref{tab: perf_mimic_few_data}, \ref{fig: prf_mimic_seq_60}). Its superior performance at longer horizons is likely due to the high number of covariates, making it well-suited to contrastive-based training. We also tested the models with 800/200/200 individuals for training/validation/testing, as in \cite{Melnychuk2022CausalTF} (Appendix \ref{subsect: results_orig_mimic}), where Causal CPC achieved state-of-the-art (SOTA) results comparable to CT but with much shorter training and prediction times.
\begin{table}[htbp]
\begin{minipage}{0.55\textwidth}
\centering
\caption{Models complexity and the running time averaged over five seeds. Results are reported for tumor growth simulation ($\gamma =1$). Hardware: GPU-1xNVIDIA Tesla M60.}
\label{tab: complexity_running_time_cancer}
\resizebox{\textwidth}{!}{%
\begin{tabular}{lccccr}
\toprule
Model & Trainable parameters (k) & Training time (min) & Prediction time (min) \\
\midrule
\textbf{Causal CPC (encoder + decoder)} & 8.2 & 16$\pm$ 3 &  4 $\pm$ 1 \\
\textbf{CT} & 11 & 12$\pm$ 2 & 30$\pm$ 3 \\
\textbf{G-Net} & 1.2 & 2 $\pm$ 0.5 & 35 $\pm$ 3 \\
\textbf{CRN} & 5.2 & 13$\pm$ 2 & 4$\pm$ 1 \\
\textbf{RMSN} & 1.6 & 22$\pm$ 2 & 4$\pm$ 1 \\
\textbf{MSM} & \textbf{$<$0.1} & \textbf{1$\pm$0.5} & \textbf{1$\pm$0.5} \\
\bottomrule
\end{tabular}
}
\end{minipage}
\hfill
\begin{minipage}{0.40\textwidth}
\centering
\caption{Ablation study with NRMSE averaged across ($1\leq \tau \leq 10$) for cancer simulation ($\gamma = 1$) and MIMIC III.}
\label{tab: ablation_study}
\resizebox{\textwidth}{!}{%
\begin{tabular}{lcccr}
\toprule
Model & Cancer\_Sim & MIMIC III \\
\midrule
Causal CPC (Full) & \textbf{1.05} & \textbf{0.62} \\ 
Causal CPC (w/o $\mathcal{L}^{(InfoNCE)}$) & 1.07 & 0.68 \\
Causal CPC (w/o $\mathcal{L}^{(InfoMax)}$) & 1.13 & 0.74 \\\hline
Causal CPC (w CDC loss) & 1.07 & 0.73 \\
Causal CPC (w/o balancing) & 1.08 & 0.69 \\
\bottomrule
\end{tabular}
}
\end{minipage}
\end{table}
\paragraph{Computational Efficiency and Model Complexity} Efficient execution is crucial for practical deployment, especially with periodic retraining. Beyond training, challenges arise in evaluating multiple counterfactual trajectories per individual, which grow exponentially with the forecasting horizon as $K^{\tau}$, where $K$ is the number of possible treatments. This is particularly relevant when generating multiple treatment plans, such as minimizing tumor volume. Table \ref{tab: complexity_running_time_cancer} shows the models' complexity (number of parameters) and running time, split between model fitting and prediction. Causal CPC is highly efficient during prediction due to its simple 1-layer GRU, similar to CRN (1-layer LSTM), while providing better ITEs estimation. In contrast, CT is less efficient due to its transformer architecture and teacher forcing, which requires recursive data loading during inference. G-Net also has longer prediction times due to Monte Carlo sampling. Overall, Causal CPC strikes a strong balance between accuracy and efficiency, making it well-suited under constrained resources.

\paragraph{Real MIMIC-III Data}
\begin{wrapfigure}[10]{r}{0.37\textwidth}
     \vspace{-10pt}  
     \centering
     \includegraphics[width=0.34\textwidth, height=0.12\textheight]{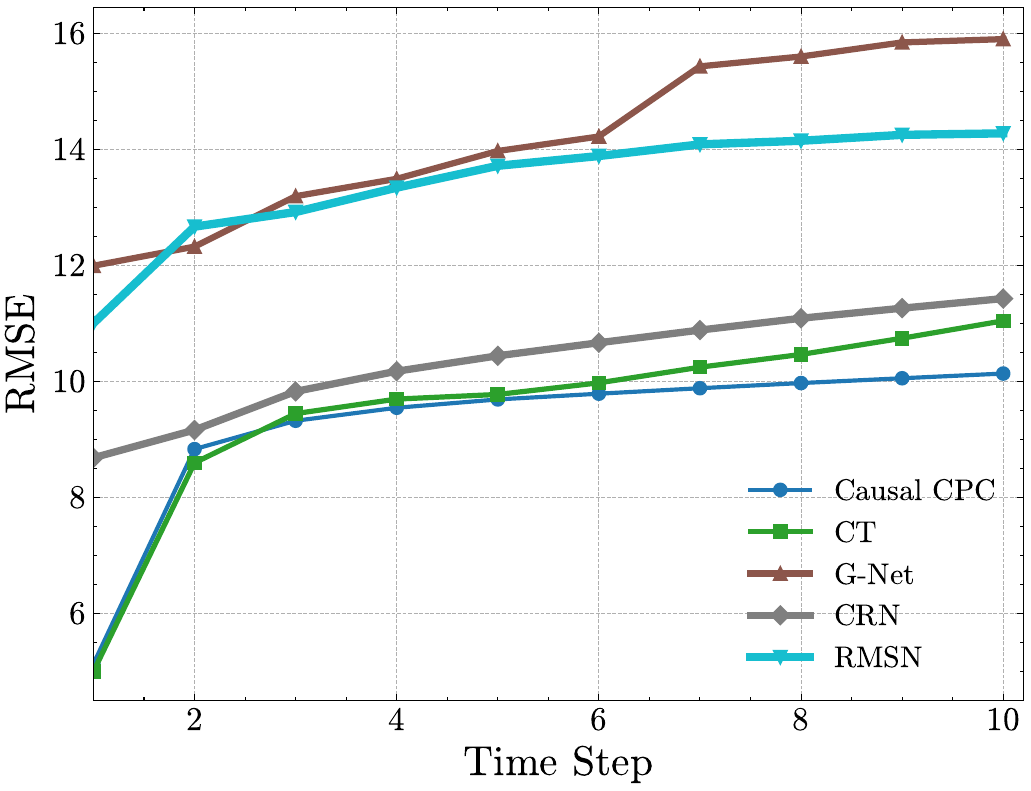}
     \caption{Evolution of RMSEs, Real MIMIC III, sequence length 100.}
     \label{fig: prf_mimic_real}
\end{wrapfigure}
We evaluated our model on real MIMIC-III data, where counterfactual trajectories cannot be assessed due to the absence of observed counterfactual responses. However, performance can still be measured by forecasting factual (observed) responses over time. Our model estimates responses for each individual based on their observed treatment trajectory. As shown in Figure \ref{fig: prf_mimic_real}, Causal CPC consistently outperforms all baselines, especially at larger horizons, demonstrating its robustness and effectiveness in real-world settings.
\section{Discussion}
\paragraph{Why does Causal CPC outperform SOTA at large horizons?}  
Our context \(\mathbf{C}_t\) is designed to capture shared information across future representations, particularly covariates, by minimizing the InfoNCE loss over multiple time steps (Eq. \ref{eq: infonce_loss_cpc}). As shown in Eq. (\ref{eq: MI_lb_avg}), minimizing \(\mathcal{L}^{CPC}\) maximizes shared information between the context and future components, helping capture the \emph{global structure} of the process. This is especially beneficial for counterfactual regression over long horizons, explaining the model's superior performance. However, it may not always outperform SOTA in shorter-term predictions due to its focus on long-term dependencies.

\paragraph{Short-term Counterfactual Regression}  
While our model is designed for long-term predictions, it may not consistently outperform SOTA for short-horizon tasks. However, the use of contrastive loss, particularly InfoNCE (Eq. \ref{eq: cpc_loss}), suggests potential adaptability to balance both short- and long-term predictions without retraining. A trade-off could be achieved by adjusting the contrastive term weights across time steps in Eq. (\ref{eq: cpc_loss}), which we leave for future work.
\begin{wrapfigure}[11]{r}{0.40\textwidth}
     \vspace{8pt}  
     \centering
    \includegraphics[width = 0.37\textwidth, height = 0.12\textheight]{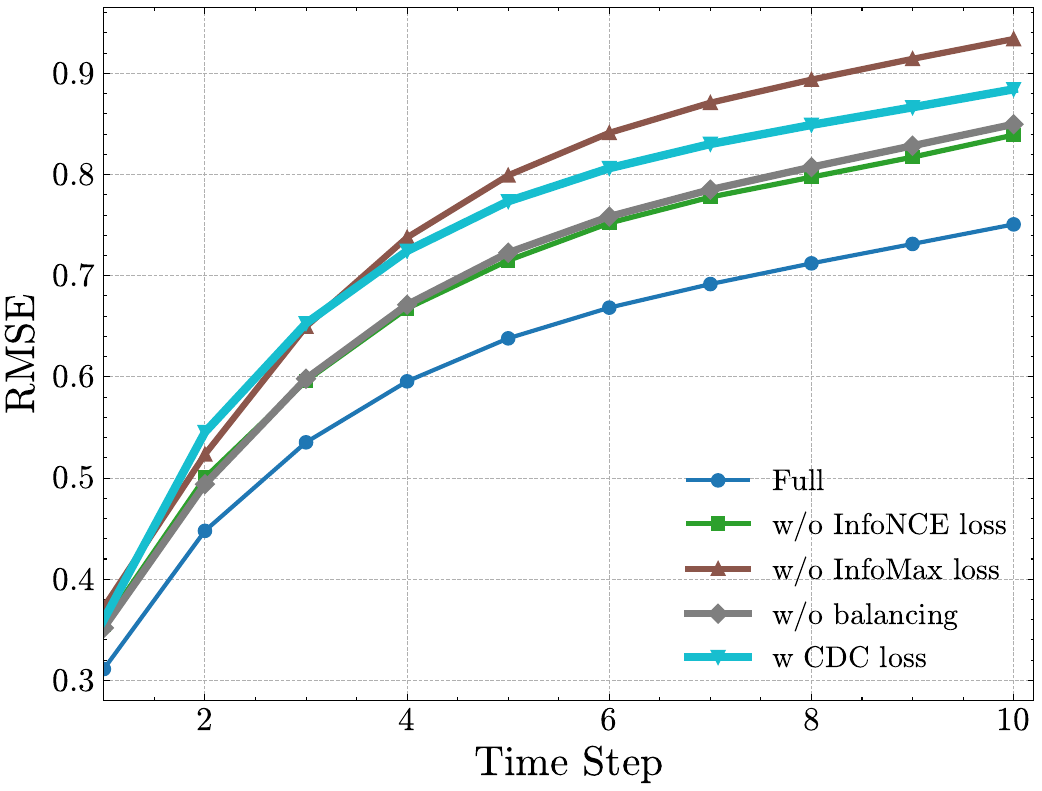}
    \caption{Ablation study of Causal CPC on MIMIC III.}
  \label{fig: prf_mimic_ablation_core}
\end{wrapfigure}

\paragraph{Ablation study} 
We examined the model's performance in various configurations—full model, without CPC, and without InfoMax. Table \ref{tab: ablation_study} and Figure \ref{fig: prf_mimic_ablation_core} show that removing either term reduces counterfactual accuracy across all horizons, underscoring their significance. Additionally, replacing our ICLUB objective with CDC loss \citep{Melnychuk2022CausalTF} or removing balancing increases errors. We also tested different MI lower bounds like NWJ \citep{nguyen2010estimating} and MINE \cite{belghazi2018mutual} for both CPC and InfoMax (Appendix \ref{subsect: extended_work_mi_ss}), finding that InfoNCE yielded the best results (Table \ref{tab: mimic_mi_estimators}). Full ablation results are in Appendices \ref{subsect: ablation_results_cancer} and \ref{subsect: ablation_results_mimic}.

\paragraph{Falsifiability Test}
This study assumes sequential ignorability, common in causal inference \cite{lim2018RMSM, bica2020estimatingCRN, Li2021GNetAR, Melnychuk2022CausalTF}. To assess robustness, we performed a falsifiability test by omitting certain confounders during training, while they remained in MIMIC-III data construction. As seen in Table \ref{tab: perf_mimic_hidden_conf}, violating sequential ignorability increased prediction errors for Causal CPC, CT, and CRN, though RMSN was less affected but underperformed at \(\tau \geq 2\). Despite this, Causal CPC maintained its lead at larger horizons, demonstrating strong encoding of long-term dependencies.
\vspace{-0.5cm}
\begin{table}[!htbp]
\centering
\caption{Results on the MIMIC III when sequential ignorability is violated reported by RMSEs}
\label{tab: perf_mimic_hidden_conf}
\resizebox{\textwidth}{!}{%
\begin{tabular}{|c |c |c |c |c |c |c |c |c |c |c|} 
\hline
Model  & $\tau = 1$   &   $\tau = 2$&  $\tau = 3$&   $\tau = 4$&  $\tau = 5$&   $\tau = 6$& $\tau = 7$&  $\tau = 8$&   $\tau = 9$ & $\tau = 10$\\ 
\hline  
Causal CPC  &  0.44$\pm $ 0.04&   0.56$ \pm$ 0.07& 0.66$\pm 0.07$& 0.73 $\pm$ 0.08&   0.78$\pm$ 0.08&  0.83$\pm$0.06&   \textbf{0.86$ \pm$ 0.10}& \textbf{0.88$\pm $ 0.08}& \textbf{0.91 $\pm$ 0.08}& \textbf{0.95$ \pm$ 0.07}\\ 
\hline
Causal Transformer &  \textbf{0.34$\pm$0.07}  &\textbf{0.48$\pm$0.07} & \textbf{0.60$\pm$0.07} & \textbf{0.68$\pm$0.06} &  \textbf{0.75$ \pm$ 0.06} & \textbf{0.80$\pm$0.07}&  \textbf{0.86$\pm$ 0.09}& 0.91$\pm 0.11$&0.95 $\pm 0.13$&  1.00 $\pm $0.15\\ 
\hline 
CRN& 0.40$\pm $ 0.07& 0.54$ \pm$ 0.09& 0.70$ \pm$ 0.09& 0.84$ \pm$ 0.09& 0.97$ \pm$ 0.09& 1.08$ \pm$ 0.13& 1.18$ \pm$ 0.16& 1.26$ \pm$ 0.19& 1.33$ \pm$ 0.21&1.39$ \pm$ 0.23\\
 \hline
RMSN & 0.38$ \pm$ 0.08& 0.67$ \pm$ 0.21& 0.78$ \pm$ 0.16& 0.84$ \pm$ 0.14& 0.91$ \pm$ 0.14& 0.98$ \pm$ 0.15& 1.04$ \pm$ 0.16& 1.09$ \pm$ 0.18& 1.15$ \pm$ 0.19&1.20$ \pm$ 0.23\\
  \hline
 \end{tabular}
 }
\end{table}
\vspace{-0.5cm}
\paragraph{Tightness of MI Upper Bounds}  
Estimating MI bounds for high-dimensional variables is challenging and expensive \citep{rainforth2018tighter, poole2019VBMI}, often limited to low-dimensional inputs or Gaussian assumptions. In MI-constrained models, batch size is crucial. As shown in Eq. (\ref{eq: LB_MI_InfoNCE}), increasing the batch size \(\mathcal{B}\) tightens the lower bound via \(\log(|\mathcal{B}|)\), and to balance memory and performance, we chose batch sizes of 256 for the encoder and 128 for the decoder. While these bounds may not be perfectly tight, mutual information and self-supervision biases significantly enhance performance (Appendix \ref{subsect: extended_work_mi_ss}), as confirmed by ablation studies. Other MI estimators like NWJ and MINE \citep{nguyen2010estimating, belghazi2018mutual} did not improve performance; our initial setup consistently performed better (Appendix \ref{subsect: detailed_results_mimic_ablation}).

\paragraph{Extending Causal CPC to Continuous Treatment}  
Our approach could be extended to continuous treatments by replacing the treatment classifier with a regressor. Since the method maximizes likelihood, the equilibrium in Theorem \ref{thm: blancing_iclub} remains valid. However, in practice, continuous treatments will be represented by a single dimension, unlike discrete treatments with \(K\)-dimensional one-hot encoding. This risks losing important treatment information in counterfactual predictions. A simpler adaptation to our model could involve discretizing continuous treatments. 

\paragraph{Conclusion} We introduced a novel approach to long-term counterfactual regression, combining RNNs with CPC to achieve SOTA results without relying on complex transformer models. Prioritizing computational efficiency, we incorporated contrastive loss-based regularization guided by mutual information (MI). Our method consistently outperforms existing models on both synthetic and real-world datasets, marking the first application of CPC in causal inference. Future work could focus on improving interpretability by integrating Shapley values into the causal framework \citep{heskes2020causalSahap}. Additionally, developing uncertainty-aware models tailored for longitudinal data is crucial for enhancing the reliability of predictions in our causal framework \citep{de2022predicting, jesson2020identifying, yin2024conformal}.

\newpage
\bibliography{refs}

\newpage
\appendix
\section{Impact Statements} 
\label{sect: impact}
Our paper seeks to advance the field of Trustworthy Machine Learning by focusing on the accurate estimation of counterfactual trajectories. This capability holds significant potential to enhance decision-making processes across various domains, particularly in healthcare, where clinicians can leverage models designed to mitigate bias and promote fairness. Additionally, by focusing on efficiency, our contributions extend beyond traditional machine learning considerations to address environmental concerns associated with energy consumption. By advocating for the prudent use of computational resources, especially in training complex models deployed in real-world scenarios, we aim to promote sustainability in developing and applying machine learning solutions.

\section{Causal assumptions}
\subsection{Identifiability Assumptions in Causal CPC}
\label{subsect: identif_ass_CCPC}
In this section, we detail the assumptions used for the identifiability of the counterfactual responses $\mathbb{E}(Y_{t+\tau}(\omega_{t+1:t+\tau})\mid \mathbf{H}_{t+1})$. As briefly stated in Section \ref{sect: related_work},  we follows similar assumptions to \cite{robins-time-varying-exposures, robins2000MSM, bica2020estimatingCRN, Melnychuk2022CausalTF}, namely 
\begin{assumption}[Consistency] 
For every time step $t$ and given any manner by which a unit $i$ receives the sequence of treatment $\omega_{i,\leq t}$, we always observe the potential outcome $Y_{it}(\omega_{i,\leq t})$. Formally:
\begin{equation*}
    W_{i,\leq t} = w_{i,\leq t} \implies Y_{it} = Y_{it}(w_{i,\leq t}).
\end{equation*}
\label{assp: consist}
\end{assumption}

\begin{assumption}[Sequential Ignorability]
Given any time step t, we have the conditional independence:
\begin{equation*}
    Y_{it}(\omega_{it}) \indep  W_{it}| \mathbf{H}_{it} = \mathbf{h}_{it} \quad \forall (\omega_{it},\mathbf{h}_{it})
\end{equation*}
    \label{assp: ignorability}
\end{assumption}

\begin{assumption}[Overlap/positivity]
Given any time step $t$, and for any possible historical context $\mathbf{h}_t$, the probability of observing any of the possible treatment regimes is strictly positive but not deterministic:
    \begin{equation*}
        p(\mathbf{h}_t) \neq 0 \implies  0 < p(W_t =\omega_{t}|\mathbf{h}_t) < 1
    \end{equation*}
\label{assp: overlap}
\end{assumption}

The three assumptions are sufficient for the identification of the counterfactual responses from observational data, which we formulate in the following proposition. 

\begin{proposition}
    Assuming consistency, overlap, and ignorability  (assumptions \ref{assp: consist}, \ref{assp: ignorability}, \ref{assp: overlap}), the causal quantity $\mathbb{E}(Y_{t+\tau}(\omega_{t+1:t+\tau})\mid \mathbf{H}_{t+1})$ is identifiable from observational data following 
    \begin{equation*}
        \mathbb{E}(Y_{t+\tau}(\omega_{t+1:t+\tau})\mid \mathbf{H}_{t+1}) = \mathbb{E}\left(Y_{t+\tau}\mid \mathbf{H}_{t+1}, W_{t+1:t+\tau} = \omega_{t+1:t+\tau}\right)
    \end{equation*}
\end{proposition}

\begin{proof}
    See \cite{robins-time-varying-exposures}
\end{proof}
\subsection{On the Causal Graph}
\label{subsect: causal_graph}
We repeat the causal graph introduced in Figure \ref{fig: causal_graph} to explain the data generation process. here, all of the past observed data encompassed in $\mathbf{H}_{t+1}$ confounds future treatments and responses, $W_{t+1}, W_{t+2}, \dots, W_{t_{max}}$ and $Y_{t+1}, Y_{t+2}, \dots, Y_{t_{max}}$, which create long-term dependencies. The fact that post-covariates are affected by past treatments creates time-dependent confounding. The static covariates are assumed to be affecting all of the time-varying variables. Since we suppose sequential ignorability, there are no possible exogenous noises affecting both treatments and responses. However, such noise may possibly affect responses, time-varying covariates, and response variables. 

In the figure, for simplicity, we represent past treatments as $W_{ \leq t}$ such that each element in that sub-sequence confounds the next treatment and response $W_{t+1}$ and $Y_{t+1}$. Idem for $Y_{ \leq t}$ and $\mathbf{X}_{ \leq t}$. The static covariates $\mathbf{V}$ are assumed to be affecting all the time-varying variables. We omit the representation of exogenous noise for simplicity. Interactions between $W_{ \leq t}$ , $\mathbf{X}_{ \leq t}$, and $Y_{ \leq t}$ were also omitted for simplicity.

\begin{figure}[!htbp]
\begin{center}
\includegraphics[scale = 0.7]{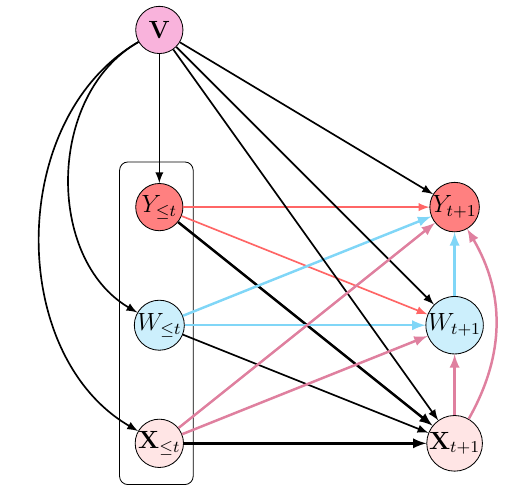}
\label{fig: appendix_causal_graph}
\end{center}
\end{figure}

\section{Extended related work} 
\label{sect: extended_work}
\subsection{Counterfactual regression over time: Methods overview }
\label{subsect: extended_work_cf}
\subsubsection{Methods included in experiments}
In this section, we give a brief overview of models included in our experiments: MSMs \cite{robins-time-varying-exposures}, RMSN \cite{lim2018RMSM}, CRN \cite{bica2020estimatingCRN}, G-Net \cite{Li2021GNetAR}, and CT \cite{Melnychuk2022CausalTF}. To delineate the differences between these models and Causal CPC, we detail in Table \ref{tab: overview_benchmark} the main design differences between all these models. 

\begin{table}[!htbp]
 \centering
  \caption{A summary of the methods included in our experiments}
\resizebox{\textwidth}{!}{%
\begin{tabular}{|p{3cm}|p{3cm}|p{3cm}|p{3cm}|p{3cm}|p{3cm}|p{3cm}|p{3cm}|}  
\hline  
Model &  Model Backbone  &  Tailored to long-term forecast? & Learning of long-term dependencies &  Use of contrastive learning &  Prediction of counterfactuals & handling selection bias &   Invertibility of representation  \\ \hline  

\textbf{Causal CPC (ours)}  & GRU &   yes  & CPC  & learn long-term relations &   Autoregressive &  Balanced representation  &  yes, contrast representation with input \\ \hline  
\textbf{CT} & 3 Transformers&	yes&	Transformer architecture&	N/A	& Autoregressive &	Balanced representation &N/A \\ 
\hline 
\textbf{G-Net} & LSTM& 	No	& N/A& 	N/A	& Autoregressive&  	G-Computation& 	Current covariates $\mathbf{X}_t$ \\ 
\hline  
\textbf{CRN} &  LSTM	&No	 &N/A	&N/A	&Autoregressive &	Balanced representation &	N/A \\ 
\hline 
\textbf{RMSN} &  LSTM& 	No& N/A& 	N/A& Autoregressive &  Weighting& 	N/A\\ 
\hline  
\textbf{MSM} &  Logistic+linear model& 	No& N/A& 	N/A& Autoregressive&  Weighting& 	N/A\\ 
\hline  
\end{tabular}
}
\label{tab: overview_benchmark}
\end{table}

\subsubsection{Methods Violating Our Assumptions}
\label{appendix: Methods_viol_assp}

Our work relies on the assumption of sequential ignorability, yet several alternative models operate under different assumptions, often addressing the presence of unobserved confounders. Some of these models are rooted in deconfounding theory \citep{lopezmultipletreatments, ranganath2018multiple, wang2019blessings}, which has been extended to time-varying settings. Deconfounding involves imposing a factor model on treatment assignment, where each treatment becomes conditionally independent given latent variables that act as proxies for unobserved confounders. Examples of this approach include \cite{Bica2020TimeSDecounf, hatt2021sequential, cao2023estimating}. Other models assume the presence of proxy variables, inferring a representation of unobserved confounders through probabilistic models based on these proxies \citep{veitch2020adapting, LuChengCMAproxynConf20121, kuzmanovic2021deconfounding}.

In contrast to our setting, which is governed by the three causal assumptions in Appendix \ref{subsect: identif_ass_CCPC}, many models assume a data-generating process similar to \citep{Soleimani2017TreatmentResponseMF,soleimani2017scalable, qian2021synctwin}. These methods, often non- or semi-parametric, tend to either ignore static covariates or handle them linearly, leading to computational inefficiencies and scalability issues. Nevertheless, some non- or semi-parametric approaches—such as \cite{schulam2017reliable, seedat2022continuous,de2022predicting, hizli2023causal}—align with our causal assumptions but extend them to continuous time, treating sequential ignorability in a continuous setting.

Additionally, models like \cite{JiangCF-GODE2023} incorporate continuous-time and assume interactions between units, where an individual’s outcome depends on both their treatment and the treatments of others. \cite{berrevoets2021disentangled}, focusing on binary treatment sequences, requires a stronger version of sequential ignorability—conditional on current covariates—whereas our model assumes a weaker version, conditioning on the entire history of covariates to account for long-lasting confounding effects. 

Furthermore, \cite{chen2023multi} focuses solely on binary treatments and targets the estimation of the average treatment effect on the treated (ATT). The authors assume a specific treatment regime, where individuals enter a post-treatment state after a defined point in time. This assumption is restrictive compared to our framework, which allows for complex, individualized treatment assignment mechanisms and non-binary treatments, where treatment values can fluctuate over time. As a result, \cite{chen2023multi} is incompatible with our causal assumptions.

Other methods, like \cite{wu2023counterfactual}, address high-dimensional counterfactual generation based on time-varying treatment plans under the same sequential ignorability assumption. However, they are not designed for causal forecasting over multiple time steps, as required in our setting. Similarly, \cite{frauen2023estimating} focuses on estimating the average causal effect and is not suited for predicting individual treatment effects or conditional counterfactual responses, as it targets marginal counterfactual expectations via g-computation.

\subsection{Mutual Information and Self-Supervision}
\label{subsect: extended_work_mi_ss}
\textbf{Self-Supervised Learning and Mutual Information} In self-supervised learning, Deep InfoMax \cite{DeepInfoMax2019} uses MI computation between input images and their representations, focusing on maximizing MI to improve reconstruction quality. Local MI between representations and image patches captures detailed patterns across regions, enhancing encoding. By maximizing average MI with local regions, Deep InfoMax significantly boosts downstream task performance, while global MI plays a key role in reconstructing the entire input from the representation.

CPC aligns with the MI-based approach seen in Deep InfoMax, emphasizing the maximization of MI between global and local representation pairs. Distinct from Deep InfoMax, CPC processes local features sequentially, constructing partial "summary features" to predict specific local features in the future. While classical self-supervised paradigms often focus on tasks like classification or reconstruction-based objectives—favoring either local or global MI maximization—integrating both approaches becomes essential for downstream tasks such as counterfactual regression over time, which justifies why Causal CPC is designed to support both local and global MI maximization to improve temporal predictions.

Several other methods share similarities with CPC, such as Contrastive Multiview Coding \cite{tian2020contrastive}. This method emphasizes maximizing mutual information between representations of different views of the same observation. Augmented Multiscale Deep InfoMax \cite{bachman2019learning}, akin to CPC, makes predictions across space but differs by predicting representations across layers in the model. While Instance Discrimination \cite{zhao2020makes} encourages representations capable of discriminating between individual examples in the dataset, our preference for CPC arises from its adaptability in processing sequential features in an ordered and autoregressive manner, which aligns with the requirements of our specific context, especially when dealing with counterfactual regression over time.

\textbf{Mutual Information and Inductive Bias.} Mutual information (MI) estimation success relies not only on MI's properties but also on the inductive biases from feature representation choices and MI estimator parameterization \cite{Tschannen2020OnMIMAX}. Experimental evidence shows that, although MI remains invariant under homeomorphisms, maximization with an invertible encoder during random initialization enhances downstream performance. While higher-capacity critics yield tighter MI bounds, findings consistent with \cite{rainforth2018tighter} suggest that simpler critics provide better representations, even with looser MI bounds. Accordingly, we selected a simple bilinear critic function for contrastive losses. In vision tasks, augmentations and contrastive loss properties are crucial for representation efficiency \citep{arora2019theoretical, tosh2021contrastive, haochen2021provable}, and \cite{saunshi2022CL_InductiveBias} highlights that inductive bias, via function class representation and optimizers, significantly affects downstream performance, offering theoretical, non-vacuous guarantees on representation quality.

\textbf{Variational Approaches and MI Estimation Challenges} The estimation of MI faces inherent challenges, particularly within variational lower bounds. These bounds often degrade as MI increases, creating a delicate trade-off between high bias and high variance. To address this, methods that utilize upper bounds on MI have been developed, attempting to mitigate challenges associated with variational bounds. One strategy for MI maximization involves computing gradients of a lower MI bound concerning the parameters of a stochastic encoder. This computational approach potentially eliminates the need for direct MI estimation, providing a more tractable solution. However, estimating MI from samples remains challenging, and traditional approaches encounter scalability issues in modern machine-learning problems.

It's crucial to note that higher estimated MI between observations and learned representations does not consistently translate to improved predictive performance in downstream supervised learning tasks. CPC is an example, exhibiting less variance but more bias, with estimates capped at $\log|\mathcal{B}|$. Strategies to reduce bias, such as increasing the batch size, introduce higher computational complexity, requiring additional evaluations for estimating each batch with the encoding function.

In our empirical approach, we adopt a specific sampling strategy for sequences, considering a one-time step per batch. This facilitates computing the InfoNCE between local summary features at time t and the future prediction of local features, leading to a reduction in algorithmic complexity for contrastive loss computation. Empirical observations demonstrate non-decreased representation quality and improved prediction of factual and counterfactual outcomes.

\textbf{Other MI lower bounds.} The Mutual Information Neural Estimator (MINE) \cite{belghazi2018mutual} leverages the relationship between MI and the Kullback-Leibler (KL) divergence. MI can be expressed as the KL divergence between the joint distribution and the product of marginals:
\[
I(X; Z) := D_{KL}(\mathbb{P}_{(X,Z)} \,\|\, \mathbb{P}_X \otimes \mathbb{P}_Z)
\]

MINE employs the Donsker-Varadhan representation \cite{donsker1983asymptotic} of the KL divergence:

\begin{equation}
D_{KL}(\mathbb{P} \,\|\, \mathbb{Q}) = \sup_{T: \Omega \rightarrow \mathbb{R}} \mathbb{E}_{\mathbb{P}}[T] - \log\left(\mathbb{E}_{\mathbb{Q}}[e^T]\right)
\label{eq: kl_divergence}
\end{equation}

Here, the supremum is taken over all functions \( T \) where the expectations exist. For a specific class of functions \( \mathcal{F} \), potentially represented by a class of neural networks, we obtain the lower bound:

\begin{equation}
D_{KL}(\mathbb{P} \,\|\, \mathbb{Q}) \geq \sup_{T \in \mathcal{F}} \mathbb{E}_{\mathbb{P}}[T] - \log\left(\mathbb{E}_{\mathbb{Q}}[e^T]\right)
\label{eq: kl_divergence_inequality}
\end{equation}

In practice, we maximize
\[
\hat{I}^{\text{MINE}}_{\gamma}(\mathbb{P} \,\|\, \mathbb{Q}) = \mathbb{E}_{\mathbb{P}}[T_{\gamma}] - \log\left(\mathbb{E}_{\mathbb{Q}}[e^{T_{\gamma}}]\right),
\]
where \( T_{\gamma} \) is a discriminator parameterized by \( \gamma \), representing neural network parameters. The MINE estimator is a strongly consistent estimator of the true MI (Theorem 2, \cite{belghazi2018mutual}).

Alternatively, the \textbf{f-divergence} representation of \( D_{KL} \) \cite{nowozin2016f} allows us to derive another MI lower bound, known as the Nguyen, Wainwright, and Jordan (NWJ) estimator \cite{nguyen2010estimating}:

\begin{equation}
D_{KL}(\mathbb{P} \,\|\, \mathbb{Q}) \geq \sup_{T \in \mathcal{F}} \mathbb{E}_{\mathbb{P}}[T] - \log\left(\mathbb{E}_{\mathbb{Q}}[e^{T-1}]\right)
\end{equation}

This results in the estimator:
\[
\hat{I}^{\text{NWJ}}_{\gamma}(\mathbb{P}, \mathbb{Q}) = \mathbb{E}_{\mathbb{P}}[T_{\gamma}] - \log\left(\mathbb{E}_{\mathbb{Q}}[e^{T_{\gamma}-1}]\right).
\]

Unlike the InfoNCE estimator, which exhibits high bias and low variance, the NWJ estimator has a low bias but high variance \cite{poole2019VBMI}.

\section{Experimental protocol}
\label{sect: exp_protocol}

All models were implemented using PyTorch \cite{Paszke2019PyTorchAI} and PyTorch Lightning \citep{Falcon_PyTorch_Lightning_2019}. In contrast to the approach in \cite{Melnychuk2022CausalTF}, we employed early stopping for all models. The stopping criterion was defined as the Mean Squared Error over factual outcomes for a dedicated validation dataset. Specifically, for the Causal CPC encoder, the stopping criterion was determined by the validation loss of the encoder.

While all models in the benchmark were trained using the Adam optimizer \cite{kingma2014adam}, we opted for training Causal CPC (encoder plus decoder without the treatment subnetwork) with AdamW \cite{loshchilov2017AdamW} due to its observed stability during training. Similar to the common practice in training GAN discriminators, the treatment subnetwork was optimized using SGD with momentum \cite{sutskever2013importance}.

The CT employed the Exponential Moving Average (EMA) \cite{yaz2018} of parameters to enhance training stability. However, this technique was not applied to Causal CPC, as experimental evidence suggested only marginal improvements. Weight decay was set to zero for all models.

For each experiment, the models were trained over five different seeds, and the reported performance metrics include the mean and standard deviation of the results.

The counterfactual trajectories are generated following two strategies: 
\begin{itemize}
    \item \textbf{Single sliding treatment} \citep{bica2020estimatingCRN, Melnychuk2022CausalTF}: Trajectories are generated with a single treatment per trajectory while the treatment slides over the forecasting range to generate multiple trajectories.  Similar to \cite{bica2020estimatingCRN}, we apply such a generation scheme to cancer simulation data.
    \item  \textbf{Random trajectories}: Trajectories are generated such that at each time step, treatment is generated randomly. We apply random trajectories to semi-synthetic MIMIC data.
\end{itemize}

For the falsifiability test on MIMIC III datset, we mask two confounders from the inputs of the benchmark models, namely sodium and glucose measurements.
\section{Experiments on synthetic data: Details} 
\subsection{Description of the Simulation Model}
\label{subsect: cancer_sim_descrip}

We present a tumor growth simulation model, focusing on the PharmacoKinetic-PharmacoDynamic (PK-PD) model as discussed in \cite{PK-PD}, a recent approach to predicting treatment responses in non-small cell lung cancer patients. This simulation models the evolution of tumor volume, denoted by \( V(t) \), in discrete time, where \( t \) represents the number of days since diagnosis:

\begin{equation}
V(t) = \left(1 + \underbrace{\Lambda \log \left(\frac{K}{V(t-1)}\right)}_{\text{Tumor Growth}} - \underbrace{\kappa_c C(t)}_{\text{Chemotherapy}} - \underbrace{\left(\kappa_{rd} Rd(t) + \upsilon Rd(t)^2\right)}_{\text{Radiation}} + \underbrace{e_t}_{\text{Noise}}\right) V(t-1)
\end{equation}

Here, the model parameters \( \Lambda, K, \kappa_c, \kappa_{rd}, \upsilon \) are sampled for each patient based on prior distributions from \cite{PK-PD}. Additionally, \( Rd(t) \) represents the radiation dose applied at time \( t \), and \( C(t) \) denotes the drug concentration. 

We introduce confounding into the assignment of radiotherapy/chemotherapy treatment by making it dependent on past tumor volume evolution. Treatment is simulated using a Bernoulli distribution with probability \( \sigma(\pi_t) \), where:

\begin{equation}
\pi_t = \frac{\gamma}{D_{\max}}\left(\bar{D}(t) - \delta\right)
\end{equation}

In this context:
\( \bar{D}(t) \) represents the average tumor diameter over the last 15 days,
\( D_{\max} = 13~\text{cm} \) is the maximum tumor diameter,
\( \delta \) is set to \( \delta = D_{\max}/2 \).

The parameter \( \gamma \) controls the level of time-dependent confounding, with a higher \( \gamma \) value assigning more weight to the history of tumor diameter in treatment assignment.

\subsection{Additional results}
\label{subsect: detailed_results_cancer}
\subsubsection{Comparison to benchmark models}
\label{subsect: benchmark_few_data_cancer}

We report in this section detailed counterfactual errors for Causal CPC and baselines over the cancer simulation dataset, which are responsible for Figure \ref{fig:combined_cancer_sim}.
\begin{table}[!htbp]
\centering
\caption{Results on the synthetic data set with \underline{sequence length 60}: mean$\pm$standard deviation of NRMSEs. The best value for each metric is given in bold: smaller is better.}
\resizebox{\textwidth}{!}{
\begin{tabular}{|c|c|c|c|c|c|c|c|c|c|c|}
\hline  
\multicolumn{11}{|c|}{$ \gamma =1$} \\
\hline
Model & $\tau = 1$ & $\tau = 2$& $\tau = 3$& $\tau = 4$& $\tau = 5$& $\tau = 6$& $\tau = 7$& $\tau = 8$& $\tau = 9$ & $\tau = 10$\\ 
\hline  
\textbf{Causal CPC (ours)} & \textbf{0.83$\pm$ 0.06}& \textbf{0.86 $\pm$ 0.09}& \textbf{0.94$\pm$ 0.09}& \textbf{0.97$\pm$ 0.08}& \textbf{1.03$\pm$ 0.10}& \textbf{1.07 $\pm$ 0.10}& \textbf{1.12$\pm$ 0.10}& \textbf{1.17 $\pm$ 0.09}& \textbf{1.22$\pm$ 0.08}& \textbf{1.26$\pm$ 0.08}\\
\hline  
\textbf{CT} & 0.99$\pm$ 0.13& 0.92 $\pm$ 0.14& 0.98$\pm$ 0.14& 1.05$\pm$ 0.15& 1.11 $\pm$ 0.18 & 1.11 $\pm0.11$& 1.21 $\pm$ 17& 1.26$\pm$ 0.16& 1.31 $\pm$ 0.005& 1.35$\pm$ 0.16\\ 
\hline 
 \textbf{G-Net} & 0.91$\pm$0.15& 1.1$\pm$0.16& 1.24$\pm$0
16& 1.33$\pm$0.17& 1.40$\pm$0.18& 1.47$\pm$0.19& 1.52$\pm$0.18& 1.57$\pm$0.22& 1.63$\pm$0.22&1.7$\pm$0.25\\ \hline 
 \textbf{CRN}    & 0.84$\pm$0.10& 0.83$\pm$0.09& 0.92$\pm$0.10& 1.00$\pm$0.11& 1.09$\pm$0.12& 1.17$\pm$0.14& 1.25$\pm$0.16& 1.32$\pm$0.18& 1.37$\pm$0.23&1.43$\pm$0.26\\ \hline 
 \textbf{RMSN}   & 0.99$\pm$0.13& 0.91$\pm$0.04& 1.30$\pm$0.65& 1.43$\pm$0.76& 1.56$\pm$0.83& 1.66$\pm$0.88& 1.73$\pm$0.91& 1.77$\pm$0.89& 1.81$\pm$0.88&1.84$\pm$0.86\\ \hline 
 \textbf{MSM}   & 1.20$\pm$0.10& 1.83$\pm$0.26& 2.07$\pm$0.44& 2.38$\pm$0.44& 2.54$\pm$0.45& 2.90$\pm$0.37& 3.01$\pm$.38& 3.06$\pm$0.36& 3.08$\pm$0.36&3.08$\pm$0.36\\\hline 
\multicolumn{11}{|c|}{$ \gamma = 2$} \\
\hline
\textbf{Causal CPC (ours)} & 1.16$\pm$ 0.22& \textbf{0.91 $\pm$ 0.10}& \textbf{0.95 $\pm$ 0.13}& \textbf{1.00 $\pm$ 0.15}& \textbf{1.07$\pm$ 0.19}& \textbf{1.17 $\pm$ 0.24}& \textbf{1.27 $\pm$ 0.25}& \textbf{1.38 $\pm$ 0.28}& \textbf{1.49 $\pm$ 0.30}& \textbf{1.60 $\pm$ 0.34}\\ \hline  

\textbf{CT} & 1.24$\pm$0.20& 1.13$\pm$0.15& 1.27$\pm$021& 1.36$\pm$0.28& 1.44$\pm$0.29& 1.55$\pm$0.27& 1.64$\pm$0.28& 1.69$\pm$0.20& 1.74$\pm$0.28& 1.77 $\pm$ 0.29\\ 
\hline
 \textbf{G-Net}& \textbf{1.05$\pm$0.21}& 1.05$\pm$0.08& 1.26$\pm$0.16& 1.38$\pm$0.23& 1.48$\pm$0.27& 1.57$\pm$0.31& 1.64$\pm$0.33& 1.70$\pm$0.36& 1.75$\pm$0.39&1.8$\pm$0.42\\\hline
 \textbf{CRN}& 1.25$\pm$0.25& 1.08$\pm$0.06& 1.14$\pm$0.12& 1.21$\pm$0.17& 1.30$\pm$ 0.21& 1.41$\pm$0.25& 1.54$\pm$0.32& 1.67$\pm$0.41& 1.8$\pm$0.51&1.92$\pm$0.63\\\hline
 \textbf{RMSN}& 1.47$\pm$0.27& 1.33$\pm$0.25& 1.30$\pm$0.23& 1.33$\pm$0.24& 1.38$\pm$0.26& 1.45$\pm$0.28& 1.52$\pm$0.31& 1.60$\pm$0.25& 1.67$\pm$0.38&1.75$\pm$0.42\\\hline
 \textbf{MSM}& 1.43$\pm$0.27& 2.22$\pm$0.53& 2.67$\pm$0.63& 2.98$\pm$0.70& 3.19$\pm$0.74& 3.33$\pm$0.77& 3.41$\pm$0.79& 3.44$\pm$0.25& 3.45$\pm$0.78&3.34$\pm$0.77\\\hline
 \multicolumn{11}{|c|}{$ \gamma = 3$} \\
 \hline
  \textbf{Causal CPC (ours)} & 1.37$\pm$0.31& \textbf{1.16$\pm$0.27}& \textbf{1.26$\pm$0.30}& \textbf{1.38$\pm$0.35}& \textbf{1.53$\pm$0.40}& \textbf{1.69$\pm$0.47}&\textbf{ 1.84$\pm$0.52}& \textbf{2.00$\pm$0.51}& 2.14$\pm$0.61& 2.28$\pm$0.66\\ 
   \hline
 \textbf{CT} & 1.36$\pm$0.32& 1.42$\pm$0.36& 1.62$\pm$0.46& 1.78$\pm$0.53& 1.89$\pm$0.58& 2.01$\pm$0.63& 2.13$\pm$0.66& 2.22$\pm$0.69& 2.31$\pm$0.69& 2.37$\pm$0.73\\ 
\hline
\textbf{G-Net} &\textbf{ 1.14$\pm$0.24}& 1.22$\pm$0.15& 1.54$\pm$0.26& 1.77$\pm$0.33& 1.94$\pm$0.36& 2.09$\pm$0.40& 2.23$\pm$0.43& 2.34$\pm$0.47& 2.44$\pm$0.52& 2.52$\pm$0.56\\ 
\hline
 \textbf{CRN}& 1.46$\pm$0.29& 1.54$\pm$0.38& 1.70$\pm$0.48& 1.79$\pm$0.53& 1.86$\pm$0.92& 1.92$\pm$0.58& 1.98$\pm$0.59& 2.04$\pm$0.61& \textbf{2.10$\pm$0.63}&\textbf{2.16$\pm$0.64} \\\hline
 
\textbf{RMSN}& 1.22$\pm$0.26& 1.28$\pm$0.29& 1.43$\pm$0.40& 1.56$\pm$0.48& 1.70$\pm$0.53& 1.83$\pm$0.57& 1.95$\pm$0.59& 2.06$\pm$0.61& 2.14$\pm$0.61&2.21$\pm$0.61\\\hline
  
 \textbf{MSM} & 1.70$\pm$0.35& 2.73$\pm$0.88& 3.22$\pm$1.03& 3.25$\pm$1.12& 3.71$\pm$1.18& 3.85$\pm$1.22& 3.91$\pm$1.23& 3.95$\pm$1.24& 3.96$\pm$1.24&3.94$\pm$1.23\\\hline
 \end{tabular}
 }
\end{table}

\begin{table}[!htbp]
\centering
\caption{Results on the synthetic data set with \underline{sequence length 40}: mean$\pm$standard deviation of NRMSEs. The best value for each metric is given in bold: smaller is better.}
\resizebox{\textwidth}{!}{
\begin{tabular}{|c|c|c|c|c|c|c|c|c|c|c|}
\hline  
\multicolumn{11}{|c|}{$ \gamma =1$} \\
\hline
Model & $\tau = 1$ & $\tau = 2$& $\tau = 3$& $\tau = 4$& $\tau = 5$& $\tau = 6$& $\tau = 7$& $\tau = 8$& $\tau = 9$ & $\tau = 10$\\ 
\hline  
\textbf{Causal CPC (ours)} & 1.21$\pm$0.07 & 1.13$\pm$0.12 & 1.22$\pm$0.12 & \textbf{1.27$\pm$0.17} & \textbf{1.35$\pm$0.19} & \textbf{1.43$\pm$0.21} & \textbf{1.48$\pm$0.22} & \textbf{1.54$\pm$0.22} & \textbf{1.58$\pm$0.23} & \textbf{1.62$\pm$0.23}\\
\hline  
\textbf{CT} & 1.21$\pm$0.09 & 1.34$\pm$0.10 & 1.43$\pm$0.14 & 1.52$\pm$0.19 & 1.59$\pm$0.22 & 1.65$\pm$0.23 & 1.71$\pm$0.23 & 1.76$\pm$0.21 & 1.78$\pm$0.21 & 1.80$\pm$0.18\\ 
\hline 
\textbf{G-Net} & 1.11$\pm$0.10 & 1.30$\pm$0.17 & 1.52$\pm$0.20 & 1.65$\pm$0.21 & 1.76$\pm$0.24 & 1.86$\pm$0.29 & 1.96$\pm$0.34 & 2.05$\pm$0.40 & 2.13$\pm$0.46 & 2.19$\pm$0.52\\
\hline 
\textbf{CRN} & \textbf{1.01$\pm$0.12} & \textbf{1.11$\pm$0.14} & \textbf{1.19$\pm$0.14} & 1.30$\pm$0.13 & 1.41$\pm$0.12 & 1.49$\pm$0.09 & 1.56$\pm$0.07 & 1.62$\pm$0.05 & 1.67$\pm$0.04 & 1.70$\pm$0.04\\
\hline 
\textbf{RMSN} & 1.14$\pm$0.03 & 1.20$\pm$0.08 & 1.24$\pm$0.07 & 1.32$\pm$0.08 & 1.40$\pm$0.08 & 1.49$\pm$0.08 & 1.57$\pm$0.08 & 1.64$\pm$0.06 & 1.71$\pm$0.05 & 1.76$\pm$0.04\\\hline 
\multicolumn{11}{|c|}{$ \gamma = 2$} \\
\hline
\textbf{Causal CPC (ours)} & 1.41$\pm$0.09 & 1.30$\pm$0.17 & 1.33$\pm$0.20 & \textbf{1.41$\pm$0.25} & \textbf{1.50$\pm$0.28} & \textbf{1.56$\pm$0.31} & \textbf{1.60$\pm$0.23} & \textbf{1.66$\pm$0.27} & \textbf{1.70$\pm$0.26} & \textbf{1.74$\pm$0.25}\\
\hline  
\textbf{CT} & 1.36$\pm$0.15 & 1.38$\pm$0.19 & 1.40$\pm$0.25 & 1.48$\pm$0.28 & 1.54$\pm$0.30 & 1.62$\pm$0.31 & 1.70$\pm$0.31 & 1.75$\pm$0.31 & 1.82$\pm$0.32 & 1.85$\pm$0.29\\ 
\hline 
\textbf{G-Net} & 1.26$\pm$0.09 & 1.47$\pm$0.20 & 1.74$\pm$0.28 & 1.90$\pm$0.34 & 2.02$\pm$0.40 & 2.12$\pm$0.44 & 2.22$\pm$0.50 & 2.31$\pm$0.55 & 2.39$\pm$0.61 & 2.47$\pm$0.68\\
\hline 
\textbf{CRN} & \textbf{1.24$\pm$0.22} & \textbf{1.21$\pm$0.18} & \textbf{1.31$\pm$0.22} & 1.42$\pm$0.25 & 1.55$\pm$0.29 & 1.67$\pm$0.34 & 1.79$\pm$0.41 & 1.91$\pm$0.50 & 2.03$\pm$0.59 & 2.13$\pm$0.70\\
\hline 
\textbf{RMSN} & 1.28$\pm$0.15 & 1.28$\pm$0.13 & 1.37$\pm$0.12 & 1.49$\pm$0.15 & 1.62$\pm$0.19 & 1.76$\pm$0.25 & 1.89$\pm$0.30 & 2.03$\pm$0.36 & 2.15$\pm$0.41 & 2.26$\pm$0.46\\
\hline
 \multicolumn{11}{|c|}{$ \gamma = 3$} \\
 \hline
\textbf{Causal CPC (ours)} & 1.52$\pm$0.19 & \textbf{1.19$\pm$0.16} & \textbf{1.26$\pm$0.23} & \textbf{1.38$\pm$0.25} & \textbf{1.52$\pm$0.24} & \textbf{1.66$\pm$0.25} & \textbf{1.79$\pm$0.27} & \textbf{1.91$\pm$0.32} & \textbf{2.02$\pm$0.38} & \textbf{2.11$\pm$0.46}\\
\hline  
\textbf{CT} & 1.56$\pm$0.18 & 1.68$\pm$0.33 & 1.77$\pm$0.54 & 1.95$\pm$0.75 & 2.05$\pm$0.83 & 2.19$\pm$0.92 & 2.29$\pm$1.00 & 2.35$\pm$1.04 & 2.43$\pm$1.09 & 2.43$\pm$1.07\\ 
\hline 
\textbf{G-Net} & 1.19$\pm$0.02 & 1.38$\pm$0.04 & 1.72$\pm$0.02 & 1.92$\pm$0.08 & 2.11$\pm$0.18 & 2.32$\pm$0.29 & 2.52$\pm$0.41 & 2.72$\pm$0.54 & 2.93$\pm$0.68 & 3.13$\pm$0.83\\
\hline 
\textbf{CRN} & 1.50$\pm$0.01 & 1.47$\pm$0.07 & 1.62$\pm$0.24 & 1.76$\pm$0.40 & 1.89$\pm$0.51 & 1.98$\pm$0.57 & 2.05$\pm$0.61 & 2.10$\pm$0.62 & 2.13$\pm$0.61 & 2.15$\pm$0.58\\
\hline 
\textbf{RMSN} & \textbf{1.48$\pm$0.19} & 1.40$\pm$0.13 & 1.66$\pm$0.35 & 1.86$\pm$0.52 & 2.01$\pm$0.61 & 2.14$\pm$0.65 & 2.24$\pm$0.67 & 2.32$\pm$0.67 & 2.37$\pm$0.66 & 2.39$\pm$0.65\\
\hline
 \end{tabular}
 }
\end{table}
\subsubsection{Ablation study}
\label{subsect: ablation_results_cancer}
We detail here the results of the ablation study conducted on the cancer simulation dataset (Table \ref{tab: ablation_study}).  The (full) Causal CPC model, as presented in the core paper, gives, in most cases, better results than any ablation configuration. 
\begin{table}[!htbp]
 \caption{Results of the ablation study on the synthetic data set: mean$\pm$standard deviation of Normalized Rooted Mean Squared Errors (NRMSEs). The best value for each metric is given in bold: smaller is better.}
 \centering
\resizebox{\textwidth}{!}{%
\begin{tabular}{|c |c |c |c |c |c |c |c |c |c |c|} 
\hline
Model  & $\tau = 1$   &   $\tau = 2$&  $\tau = 3$&   $\tau = 4$&  $\tau = 5$&   $\tau = 6$& $\tau = 7$&  $\tau = 8$&   $\tau = 9$ & $\tau = 10$\\ \hline  
\textbf{CAUSAL CPC (FULL)}  & \textbf{0.83 $\pm$ 0.06}  &\textbf{0.86 $\pm$ 0.06} &0.94$\pm$ 0.09&   \textbf{0.97 $\pm$ 0.08}&  \textbf{1.03 $\pm$ 0.10 }&   \textbf{1.07$\pm$0.10}& \textbf{1.12$\pm$ 0.10}& \textbf{1.17$\pm$ 0.06}&   \textbf{ 1.22 $\pm$ 0.08 }&   \textbf{ 1.26$\pm$ 0.08 }\\ 
\hline  
\textbf{Causal CPC (w/o $\mathcal{L}^{(InfoNCE)}$)}&  0.84$\pm$0.04&   0.91$\pm$0.07& 0.95$\pm$ 0.07&   0.99$ \pm$ 0.09& 1.03$\pm$0.10 & 1.10$ \pm$ 0.07& 1.15$\pm 0.14$& 1.20 $\pm$ 0.14&  1.23$ \pm $0.14&   1.28$ \pm 0.15$\\ 
\hline  
\textbf{Causal CPC (w/o $\mathcal{L}^{(InfoMax)}$) } &  0.84$\pm$0.04&   0.86$\pm$0.09& \textbf{0.91$\pm$0.08}&  0.99$\pm$0.10& 1.07$\pm$0.08&   1.16$ \pm$ 0.08& 1.24$ \pm$ 0.10& 1.31 $\pm$ 0.12&  1.38$ \pm $0.08&   1.46$ \pm 0.10$\\ 
\hline
 \textbf{Causal CPC (w CDC loss)}& 0.83$\pm$0.02& 0.89$\pm$0.07& 0.96$\pm$ 0.07& 1.03$\pm$ 0.07& 1.07$\pm$0.08& 1.10$ \pm$ 0.07& 1.13$\pm 0.10$& 1.18 $\pm$ 0.09& 1.24$ \pm $0.11&1.28$ \pm 0.11$\\\hline
 \textbf{Causal CPC (w Balancing}& 0.84$\pm$0.04& 0.88$\pm$0.05& 0.97$\pm$ 0.05& 1.04$\pm$ 0.07& 1.08$\pm$0.10& 1.13$ \pm$ 0.08& 1.15$\pm 0.14$& 1.20$\pm$ 0.10& 1.25$ \pm $0.08&1.29$ \pm 0.12$\\\hline 
 \end{tabular}
 }
\label{tab: cancer_ablation_details}
\end{table}


\section{Experiments on semi-synthetic data: Details} 

\subsection{Description of the simulation model}
\label{subsect: details_sim_mimic}

In this section, we provide a concise overview of the simulation model built upon the MIMIC III dataset, as introduced by \cite{Melnychuk2022CausalTF}. Initially, a cohort of 1,000 patients is extracted from the MIMIC III data, and the simulation proposed by \cite{Melnychuk2022CausalTF} extends the model of \cite{schulam2017reliable}.

Let \( d_y \) be the dimension of the outcome variable. In the case of multiple outcomes, untreated outcomes, denoted as \( \mathbf{Z}_t^{j,(i)} \) for \( j=1, \ldots, d_y \), are generated for each patient \( i \) within the cohort. The generation process is defined as follows:

\begin{equation}
\mathbf{Z}_t^{j,(i)} = \underbrace{\alpha_S^j \mathbf{B}\text{-spline}(t) + \alpha_g^j g^{j,(i)}(t)}_{\text{endogenous}} + \underbrace{\alpha_f^j f_Z^j\left(\mathbf{X}_t^{(i)}\right)}_{\text{exogenous}} + \underbrace{\varepsilon_t}_{\text{noise}}
\end{equation}

where:
the B-spline \(\mathbf{B}\text{-spline}(t)\) is an endogenous component,
\( g^{j,(i)}(\cdot) \) is sampled independently for each patient from a Gaussian process with a Matérn kernel 
and \( f_Z^j(\cdot) \) is sampled from a Random Fourier Features (RFF) approximation of a Gaussian process. 

To introduce confounding in the assignment mechanism, current time-varying covariates are incorporated via a random function \( f_Y^l\left(\mathbf{X}_t\right) \) and the average of the subset of the previous \( T_l \) treated outcomes, \( \bar{A}_{T_l}\left(\overline{\mathbf{Y}}_{t-1}\right) \). For \( d_a \) binary treatments \( \mathbf{A}_t^l \), where \( l=1, \ldots, d_a \), the assignment mechanism is modeled as:

\begin{equation*}
\begin{aligned}
    p_{\mathbf{A}_t^l} &= \sigma\left(\gamma_A^l \bar{A}_{T_l}\left(\overline{\mathbf{Y}}_{t-1}\right) + \gamma_X^l f_Y^l\left(\mathbf{X}_t\right) + b_l\right), \\
    \mathbf{A}_t^l &\sim \operatorname{Bernoulli}\left(p_{\mathbf{A}_t^l}\right).
\end{aligned}
\end{equation*}

Subsequently, treatments are applied to the untreated outcomes using the following expression:

\begin{equation}
E^j(t) = \sum_{i=t-w^l}^t \frac{\min_{l=1, \ldots, d_a} \mathbb{1}_{\left[\mathbf{A}_i^l=1\right]} p_{\mathbf{A}_i^l} \beta_{l j}}{\left(w^l-i\right)^2}
\end{equation}

The final outcome combines the treatment effect and the untreated simulated outcome:

\begin{equation}
Y^j_t = Z^j_t + E^j(t).
\end{equation}

\subsection{Additional results}
\label{subsect: detailed_results_mimic}
\subsubsection{Ablation study}
\label{subsect: detailed_results_mimic_ablation}
We detail here the results of the ablation study conducted on the MIMIC III semi-synthetic dataset (Table \ref{tab: ablation_study}).  The (full) Causal CPC model, as presented in the core paper, gives consistently better results than any ablation configuration.  
\label{subsect: ablation_results_mimic} 
\begin{table}[!htbp]
 \caption{Results on  MIMIC III semi-synthetic data set: mean$\pm$standard deviation of Normalized Rooted Mean Squared Errors (NRMSEs). The best value for each metric is given in bold: smaller is better.}
 \centering
\resizebox{\textwidth}{!}{%
\begin{tabular}{|c |c |c |c |c |c |c |c |c |c |c|} 
\hline
Model  & $\tau = 1$   &   $\tau = 2$&  $\tau = 3$&   $\tau = 4$&  $\tau = 5$&   $\tau = 6$& $\tau = 7$&  $\tau = 8$&   $\tau = 9$ & $\tau = 10$\\ 
\hline  
\textbf{Causal CPC (ful)}  &  \textbf{0.32$\pm $ 0.04}&   \textbf{0.45$ \pm$ 0.08}& \textbf{0.54$\pm 0.06$}& \textbf{0.61 $\pm$ 0.10}&   \textbf{0.66$\pm$ 0.10}&  \textbf{0.69$\pm$0.11}&   \textbf{0.71$ \pm$ 0.11}& \textbf{0.73$\pm $ 0.06}& \textbf{0.75 $\pm$ 0.05}& \textbf{0.77$ \pm$ 0.10} \\ 

\hline
 \textbf{Causal CPC (w/o $\mathcal{L}^{(InfoNCE)}$)}& 0.35$\pm $ 0.04& 0.50$ \pm$ 0.05& 0.59$ \pm$ 0.06& 0.66$ \pm$ 0.06& 0.71$ \pm$ 0.08& 0.75$ \pm$ 0.06& 0.77$ \pm$ 0.07& 0.79$ \pm$ 0.08& 0.81$ \pm$ 0.07&0.83$ \pm$ 0.07\\
 
 \hline
 \textbf{Causal CPC (w/o $\mathcal{L}^{(InfoMax)}$) } & 0.36$ \pm$ 0.02& 0.53$ \pm$ 0.03& 0.64$ \pm$ 0.04& 0.71$ \pm$ 0.05& 0.77$ \pm$ 0.05& 0.77$ \pm$ 0.05& 0.83$ \pm$ 0.05& 0.86$ \pm$ 0.05& 0.88$ \pm$ 0.08&0.90$ \pm$ 0.05\\
 
 \hline
 \textbf{Causal CPC (CDC loss) }& 0.36$ \pm$ 0.02& 0.54$ \pm$ 0.03& 0.65$ \pm$ 0.05& 0.72$ \pm$ 0.05& 0.77$ \pm$ 0.05& 0.70$ \pm$ 0.04& 0.83$ \pm$ 0.04& 0.85$ \pm$ 0.03& 0.86$ \pm$ 0.03&0.88$ \pm$ 0.08\\\hline 
 
\textbf{Causal CPC (w/o balancing)}& 0.35$ \pm$ 0.03& 0.50$ \pm$ 0.05& 0.60$ \pm$ 0.06& 0.67$ \pm$ 0.06& 0.72$ \pm$ 0.06& 0.76$ \pm$ 0.06& 0.78$ \pm$ 0.06& 0.80$ \pm$ 0.06& 0.83$ \pm$ 0.06&0.85$ \pm$ 0.06\\\hline 
 \end{tabular}
 }
\label{tab: mimic_ablation_details}
\end{table}

Furthermore, We replace the InfoNCE objective used to compute the CPC term and InfoMax terms with that of NWJ and MINE (Section \ref{subsect: extended_work_mi_ss}). We repeat the same MIMIC III experimentation while varying the objective used for CPC and InfoMax. Table \ref{tab: mimic_mi_estimators} shows the counterfactual errors for each configuration compared to the original formulation of Causal CPC. In all cases, The InfoNCE objective performs better with notable error reduction at large horizons.

\begin{table}[!htbp]
 \caption{Results of NWJ and MINE MI lower bounds when used for CPC and InfoMax for MIMIC III semi-synthetic data set: mean$\pm$standard deviation of Normalized Rooted Mean Squared Errors (NRMSEs). The best value for each metric is given in bold: smaller is better.}
 \centering
\resizebox{\textwidth}{!}{%
\begin{tabular}{|c |c |c |c |c |c |c |c |c |c |c|} 
\hline
Model  & $\tau = 1$   &   $\tau = 2$&  $\tau = 3$&   $\tau = 4$&  $\tau = 5$&   $\tau = 6$& $\tau = 7$&  $\tau = 8$&   $\tau = 9$ & $\tau = 10$\\ 
\hline  
\textbf{Original Model}  &  \textbf{0.34$\pm $ 0.04}&   \textbf{0.45$ \pm$ 0.08}& \textbf{0.54$\pm 0.06$}& \textbf{0.61 $\pm$ 0.10}&   \textbf{0.66$\pm$ 0.10}&  \textbf{0.69$\pm$0.11}&   \textbf{0.71$ \pm$ 0.11}& \textbf{0.73$\pm $ 0.06}& \textbf{0.75 $\pm$ 0.05}& \textbf{0.77$ \pm$ 0.10} \\ 
\hline
\textbf{CPC with NWJ}& 0.34$\pm $ 0.04& 0.48$ \pm$ 0.05& 0.58$ \pm$ 0.06& 0.66$ \pm$ 0.07& 0.71$ \pm$ 0.08& 0.75$ \pm$ 0.07& 0.78$ \pm$ 0.07& 0.81$ \pm$ 0.06& 0.84$ \pm$ 0.06&0.87$ \pm$ 0.06\\
 
\hline
\textbf{CPC with MINE } & 0.35$ \pm$ 0.03& 0.50$ \pm$ 0.05& 0.61$ \pm$ 0.04& 0.69$ \pm$ 0.04& 0.75$ \pm$ 0.04& 0.79$ \pm$ 0.03& 0.82$ \pm$ 0.03& 0.85$ \pm$ 0.02& 0.88$ \pm$ 0.02&0.91$ \pm$ 0.02\\
 \hline
\textbf{InfoMax with NWJ}& 0.42$ \pm$ 0.08& 0.56$ \pm$ 0.04& 0.69$ \pm$ 0.07& 0.77$ \pm$ 0.08& 0.83$ \pm$ 0.09& 0.87$ \pm$ 0.09& 0.90$ \pm$ 0.09& 0.92$ \pm$ 0.09& 0.94$ \pm$ 0.08&0.96$ \pm$ 0.08\\\hline 
 
\textbf{InfoMax with MINE}& 0.37$ \pm$ 0.05& 0.52$ \pm$ 0.03& 0.65$ \pm$ 0.06& 0.73$ \pm$ 0.8& 0.80$ \pm$ 0.10& 0.84$ \pm$ 0.11& 0.87$ \pm$ 0.11& 0.89$ \pm$ 0.10& 0.91$ \pm$ 0.10&0.93$ \pm$ 0.09\\\hline 
 \end{tabular}
 }
\label{tab: mimic_mi_estimators}
\end{table}

\subsubsection{Comparison to benchmark models: standard train/test split}
\label{subsect: results_orig_mimic}
As mentioned in Section \ref{subsect: core_exp_mimic}, We also tested Causal CPC on MIMIC III semi-synthetic data using the same experimental protocol as \cite{Melnychuk2022CausalTF}, namely by using the split of patients into train/validation/test as 800/200/200. As a result, baseline performances in Table \ref{tab: perf_mimic_orig_data} are exactly the same as in \cite{Melnychuk2022CausalTF}. 
\begin{table}[!htbp]
 \caption{Results over the MIMIC III semi-synthetic data set (same experimental protocol as in \cite{Melnychuk2022CausalTF}): mean$\pm$standard deviation of Rooted Mean Squared Errors (RMSEs). The best value for each metric is given in bold: smaller is better.}
 \centering
\resizebox{\textwidth}{!}{%
\begin{tabular}{|c |c |c |c |c |c  |c |c |c |c |c|} 
\hline
Model       & $\tau = 1$   &   $\tau = 2$& $\tau = 3$&   $\tau = 4$&  $\tau = 5$&   $\tau = 6$& $\tau = 7$&  $\tau = 8$&   $\tau = 9$ & $\tau = 10$\\ \hline  

\textbf{Causal CPC (ours)}  &   0.25 $\pm$ 0.03  & \textbf{0.37 $\pm$ 0.02}  &\textbf{0.40 $\pm$ 0.01} & \textbf{0.45 $\pm$ 0.01} & \textbf{0.49 $\pm$ 0.02} & \textbf{0.52 $\pm$ 0.02} & \textbf{0.55 $\pm$ 0.03}  &\textbf{0.56 $\pm$ 0.03 } &\textbf{0.58 $\pm$ 0.04} & \textbf{0.60 $\pm$ 0.03} \\ 
\hline  
\textbf{CT} & \textbf{ 0.20 $\pm$ 0.01 } & 0.38 $\pm$ 0.01  &0.45 $\pm$ 0.01 & 0.49 $\pm$ 0.01 & 0.52 $\pm$ 0.02 & 0.53 $\pm$ 0.02 & 0.55 $\pm$ 0.02  &\textbf{0.56 $\pm$ 0.02}  &\textbf{0.58 $\pm$ 0.02} & \textbf{0.59 $\pm$ 0.02}  \\ 
\hline  
\textbf{G-Net} & 0.34 $\pm$ 0.01& 0.67 $\pm$ 0.03& 0.83 $\pm$ 0.04& 0.94 $\pm$ 0.04& 1.03 $\pm$ 0.05 &1.10 $\pm$ 0.05 &1.16 $\pm$ 0.05 &1.21 $\pm$ 0.06& 1.25 $\pm$ 0.06 & 1.29 $\pm$ 0.06\\ 
\hline  
\textbf{CRN} &  0.30 $\pm$ 0.01& 0.48 $\pm$ 0.02 &0.59 $\pm$ 0.02& 0.65 $\pm$ 0.02 &0.68 $\pm$ 0.02& 0.71 $\pm$ 0.01 &0.72 $\pm$ 0.01 &0.74 $\pm$ 0.01 &0.76 $\pm$ 0.01& 0.78 $\pm$ 0.02\\ 
\hline  
\textbf{RMSN} &   0.24 $\pm$ 0.01& 0.47 $\pm$ 0.01& 0.60 $\pm$ 0.01& 0.70 $\pm$ 0.02& 0.78 $\pm$ 0.04 &0.84 $\pm$ 0.05 &0.89 $\pm$ 0.06& 0.94 $\pm$ 0.08 & 0.97 $\pm$ 0.09& 1.00 $\pm$ 0.11 \\ 
\hline  
\textbf{MSM} & 0.37 $\pm$ 0.01& 0.57 $\pm$ 0.03& 0.74 $\pm$ 0.06 &0.88 $\pm$ 0.03& 1.14 $\pm$ 0.10 &1.95 $\pm$ 1.48 &3.44 $\pm$ 4.57 &$>$ 10.0 & $>$ 10.0 & $>$ 10.0\\ 
\hline  
 \end{tabular}
 }
\label{tab: perf_mimic_orig_data}
\end{table}

\subsubsection{Running time and model complexity}
In this section, we complement the table about complexity and running time given for cancer simulation in the core paper by providing the exact same table but for MIMIC III semi-synthetic data. 
\begin{table}[!htbp]
\caption{The number of parameters to train for each model after hyper-parameters fine-tuning and the corresponding running time averaged over five seeds. Results are reported for semi-synthetic MIMIC III data; the processing unit is GPU - 1 x NVIDIA Tesla M60 .}
\label{tab: complexity_running_time_mimic}
\vskip 0.15in
\begin{center}
\begin{small}
\begin{sc}
\resizebox{0.7\textwidth}{!}{%
\begin{tabular}{lccccr}
\toprule
Model & trainable parameters (k) & Training time (min) & Prediction time (min) \\
\midrule
\textbf{Causal CPC (ours)} & \textbf{9.8}  & 12$\pm$2  & \textbf{4$\pm$1} \\
\textbf{CT } &12  & 14$\pm$1  & 38$\pm$2  \\
\textbf{G-Net} & 14.7 &  \textbf{7$\pm$1} &40$\pm$3 \\
\textbf{CRN} &  15.1 & 21$\pm$2  & 5$\pm$1  \\
\textbf{RMSN} & 20 & 48$\pm$4  & 5$\pm$1  \\
\bottomrule
\end{tabular}
}
\end{sc}
\end{small} 
\end{center}
\vskip -0.1in
\end{table}

\section{Proofs of theoretical results}
\label{sect: proofs}
\subsection{Relation between InfoNCE loss and mutual information}

\begin{proposition}
    $$I(\mathbf{U}_{t+j}, \mathbf{C}_{t}) \geq \log(|\mathcal{B}|) - \mathcal{L}^{(InfoNCE)}_j$$
\end{proposition}

\begin{proof}

In the following, we draw inspiration from the proof of \cite{oord2018representation}. The InfoNCE loss corresponds to the categorical cross-entropy of classifying the positive sample \( \mathbf{U}_{t+j} \) correctly, given the context \( \mathbf{C}_t \), with a probability:

\[
\frac{\exp(T_j(\mathbf{U}_{t+j},\mathbf{C}_{t}))}{ \sum_{l = 1}^{|\mathcal{B}|} \exp(T_j(\mathbf{U}_{l,t+j},\mathbf{C}_{t}))}.
\]

The positive sample \( \mathbf{U}_{t+j} \) is one element in the batch \( \mathcal{B} \), where the remaining elements serve as negative samples. Let \texttt{pos} \( \in \{1, \dots, |\mathcal{B}|\} \) be the indicator of the positive sample \( \mathbf{U}_{t+j} \). The optimal probability is given by:

\[
p(\text{Index} = \texttt{pos} \mid \mathcal{B}, \mathbf{C}_t) 
= \frac{p(\mathbf{u}_{\texttt{pos}, t+j} \mid \mathbf{C}_t) \prod_{l=1,\dots,|\mathcal{B}|; l \neq \texttt{pos}} p(\mathbf{u}_{l, t+j})}{\sum_{j=1}^{|\mathcal{B}|} \left[ p(\mathbf{u}_{j, t+j} \mid \mathbf{C}_t) \prod_{l=1,\dots,|\mathcal{B}|; l \neq j} p(\mathbf{u}_{l, t+j}) \right]}
= \frac{ \frac{p(\mathbf{u}_{\texttt{pos}, t+j} \mid \mathbf{C}_t)}{p(\mathbf{u}_{\texttt{pos}, t+j})} }{ \sum_{j=1}^{|\mathcal{B}|} \frac{p(\mathbf{u}_{j, t+j} \mid \mathbf{C}_t)}{p(\mathbf{u}_{j, t+j})} }.
\]

For the score \( \exp(T_j(\mathbf{U}_{t+j},\mathbf{C}_{t})) \) to be optimal, it should be proportional to \( \frac{p(\mathbf{u}_{\texttt{pos}, t+j} \mid \mathbf{C}_t)}{p(\mathbf{u}_{\texttt{pos}, t+j})} \). The mutual information (MI) lower bound arises from the fact that \( \exp(T_j(\mathbf{U}_{t+j},\mathbf{C}_{t})) \) estimates the density ratio \( \frac{p(\mathbf{u}_{\texttt{pos}, t+j} \mid \mathbf{C}_t)}{p(\mathbf{u}_{\texttt{pos}, t+j})} \).

\begin{equation}
\begin{aligned}
\mathcal{L}^{(InfoNCE)}_j & =-\mathbb{E}_{\mathcal{B}} \log \left[\frac{\frac{p\left(\mathbf{u}_{t+j} \mid \mathbf{c}_t\right)}{p\left(\mathbf{u}_{t+j}\right)}}{\frac{p\left(\mathbf{u}_{t+j} \mid \mathbf{c}_t\right)}{p\left(\mathbf{u}_{t+j}\right)}+\sum_{\mathbf{u}_{l, t+j} \in \mathcal{B}_{\text {neg }}} \frac{p\left(\mathbf{u}_{l, t+j} \mid \mathbf{c}_t\right)}{p\left(\mathbf{u}_{l, t+j}\right)}}\right] \\
& =\mathbb{E}_{\mathcal{B}} \log \left[1+\frac{p\left(\mathbf{u}_{t+j}\right)}{p\left(\mathbf{u}_{t+j} \mid \mathbf{c}_t\right)} \sum_{\mathbf{u}_{l, t+j} \in \mathcal{B}_{\text {neg }}} \frac{p\left(\mathbf{u}_{l, t+j} \mid \mathbf{c}_t\right)}{p\left(\mathbf{u}_{l, t+j}\right)}\right] \\
& \approx \mathbb{E}_{\mathcal{B} }\log \left[1+\frac{p\left(\mathbf{u}_{t+j}\right)}{p\left(\mathbf{u}_{l, t+j} \mid \mathbf{c}_t\right)}(|\mathcal{B}|-1) \mathbb{E}_{\mathbf{U}_{t+j}} \frac{p\left(\mathbf{u}_{l, t+j}\mid \mathbf{c}_t\right)}{p\left(\mathbf{u}_{l, t+j}\right)}\right] \\
& =\mathbb{E}_{\mathcal{B}} \log \left[1+\frac{p\left(\mathbf{u}_{t+j}\right)}{p\left(\mathbf{u}_{t+j} \mid \mathbf{c}_t\right)}(|\mathcal{B}|-1)\right] \\
& \geq \mathbb{E}_{\mathcal{B}} \log \left[\frac{p\left(\mathbf{u}_{t+j}\right)}{p\left(\mathbf{u}_{t+j}\mid \mathbf{c}_t\right)} |\mathcal{B}|\right] \\
& =-I\left(\mathbf{u}_{t+j}, \mathbf{c}_t\right)+\log (|\mathcal{B}|),
\end{aligned}
\label{eq: infonce_proof}
\end{equation}

The approximation in the third equation, Eq. \eqref{eq: infonce_proof}, becomes more precise as the batch size increases.
\end{proof}

\subsection{Relation between InfoMax and Input Reconstruction}

We now prove Proposition \ref{prop: lb_infomax}, which states that:
\[
I(\mathbf{C}_t^h,\mathbf{C}_t^f) \leq I(\mathbf{H}_{t}, (\mathbf{C}_t^h, \mathbf{C}_t^f)).
\]

\begin{proof}
This follows from two applications of the data processing inequality \cite{cover1999elementsInfo}, which states that for random variables \( A \), \( B \), and \( C \) satisfying the Markov relation \( A \rightarrow B \rightarrow C \), the inequality \( I(A;C) \leq I(A;B) \) holds.

First, since \( \mathbf{C}_t^h = \Phi_{\theta_{1},\theta_{2}}(\mathbf{H}_t^h) \) and \( \mathbf{C}_t^f = \Phi_{\theta_{1},\theta_{2}}(\mathbf{H}_t^f) \), we can write \( \mathbf{H}_t^h = \mathrm{trunc}_f(\mathbf{H}_t) \) and \( \mathbf{H}_t^f = \mathrm{trunc}_h(\mathbf{H}_t) \), where \( \mathrm{trunc}_f \) and \( \mathrm{trunc}_h \) truncate the future and history processes, respectively, given a splitting time \( t_0 \).

Thus, we have the Markov relation:
\[
\mathbf{C}_t^h \xleftarrow{\Phi_{\theta_{1},\theta_{2}} \circ \mathrm{trunc}_f} \mathbf{H}_t \xrightarrow[]{\Phi_{\theta_{1},\theta_{2}} \circ \mathrm{trunc}_h} \mathbf{C}_t^f,
\]
which is Markov equivalent to:
\[
\mathbf{C}_t^h \xrightarrow{\Phi_{\theta_{1},\theta_{2}} \circ \mathrm{trunc}_f} \mathbf{H}_t \xrightarrow[]{\Phi_{\theta_{1},\theta_{2}} \circ \mathrm{trunc}_h} \mathbf{C}_t^f.
\]
By the data processing inequality, this results in \( I(\mathbf{C}_t^h, \mathbf{C}_t^f) \leq I(\mathbf{H}_t, \mathbf{C}_t^h) \). On the other hand, we have the trivial Markov relation \( \mathbf{H}_t \rightarrow (\mathbf{C}_t^h, \mathbf{C}_t^f) \rightarrow \mathbf{C}_t^h \), which implies \( I(\mathbf{H}_t, \mathbf{C}_t^h) \leq I(\mathbf{H}_t, (\mathbf{C}_t^h, \mathbf{C}_t^f)) \). Combining these two inequalities proves the proposition.
\end{proof}

\subsection{Proof of Theorem \ref{thm: tightness_lb_info}}

To begin, we split the process history into two non-overlapping views (Figure \ref{fig: ccpc_archi}): $\mathbf{H}_t^h := \mathbf{U}_{1:t_0}$ and $\mathbf{H}_t^f := \mathbf{U}_{t_0+1:t}$, representing a historical subsequence and a future subsequence within the process history $\mathbf{H}_{t}$, respectively. We then computed representations of these two views denoted $\mathbf{C}_t^h$ and $\mathbf{C}_t^f$, respectively. This naturally gives rise to the Markov chain, as in showed in the proof of proposition \ref{prop: lb_infomax}:

\[
\mathbf{C}_t^h \xleftarrow[]{} \mathbf{H}_t \xrightarrow[]{} \mathbf{C}_t^f
\]

which is Markov equivalent to:

\[
\mathbf{C}_t^h \xrightarrow{} \mathbf{H}_t \xrightarrow[]{} \mathbf{C}_t^f
\]

Following this Markov chain, we can show that \cite{shwartz2024compress}:

\[
I(\mathbf{C}_t^f, \mathbf{C}_t^h) = I(\mathbf{H}_t, \mathbf{C}_t^h) - \mathbb{E}_{\mathbf{h}_t \sim \mathbb{P}_{\mathbf{H}_t}} \mathbb{E}_{\mathbf{c}_{t}^f \sim \mathbb{P}_{\mathbf{C}_{t}^f \mid \mathbf{h}_t}} \left[ D_{KL}[\mathbb{P}_{\mathbf{C}_{t}^h \mid \mathbf{h}_t} || \mathbb{P}_{\mathbf{C}_{t}^h \mid \mathbf{c}_{t}^f }] \right]
\]

On the other hand, by applying the chain rule of the mutual information to $I(\mathbf{H}_{t}; (\mathbf{C}_t^h, \mathbf{C}_t^f))$ we get:

\[
I(\mathbf{H}_{t}; (\mathbf{C}_t^f, \mathbf{C}_t^h)) = I(\mathbf{H}_t, \mathbf{C}_t^h) + I(\mathbf{H}_t; \mathbf{C}_t^f \mid \mathbf{C}_t^h) 
\]

Combining these equations, the tightness of our bounds can be written as:

\[
I(\mathbf{H}_{t}; (\mathbf{C}_t^f, \mathbf{C}_t^h)) - I(\mathbf{C}_t^f, \mathbf{C}_t^h) = I(\mathbf{H}_t; \mathbf{C}_t^f \mid \mathbf{C}_t^h) + \mathbb{E}_{\mathbf{h}_t \sim \mathbb{P}_{\mathbf{H}_t}} \mathbb{E}_{\mathbf{c}_{t}^f \sim \mathbb{P}_{\mathbf{C}_{t}^f \mid \mathbf{h}_t}} \left[ D_{KL}[\mathbb{P}_{\mathbf{C}_{t}^h \mid \mathbf{h}_t} || \mathbb{P}_{\mathbf{C}_{t}^h \mid \mathbf{c}_{t}^f }] \right]
\]

\subsection{On the Relation Between Conditional Entropy and Reconstruction}
\label{appendix: zero_cond_H}

We now prove the statement in the core paper, which asserts that the conditional entropy \( H(\mathbf{H}_{t} \mid (\mathbf{C}_t^h, \mathbf{C}_t^f)) \geq 0 \) is minimized if \( \mathbf{H}_{t} \) is almost surely a function of \( (\mathbf{C}_t^h, \mathbf{C}_t^f) \). The proof is adapted from \cite{cover1999elementsInfo}.

\begin{proposition}
\label{prop: zero_cond_H}
If \( H(\mathbf{A} \mid \mathbf{B}) = 0 \), then \( \mathbf{A} = f(\mathbf{B}) \) almost surely.
\end{proposition}

\begin{proof}
For simplicity, suppose \( \mathbf{A} \) and \( \mathbf{B} \) are discrete random variables. Assume, by contradiction, that there exists \( \mathbf{b}_0 \) and two distinct values \( \mathbf{a}_1 \) and \( \mathbf{a}_2 \) such that \( p(\mathbf{a}_1 \mid \mathbf{b}_0) > 0 \) and \( p(\mathbf{a}_2 \mid \mathbf{b}_0) > 0 \). Then, the conditional entropy is given by:

\begin{equation*}
    H(\mathbf{A} \mid \mathbf{B}) = - \sum_{\mathbf{b}} p(\mathbf{b}) \sum_{\mathbf{a}} p(\mathbf{a} \mid \mathbf{b}) \log p(\mathbf{a} \mid \mathbf{b}).
\end{equation*}

In particular, we have:
\[
H(\mathbf{A} \mid \mathbf{B}) \geq p(\mathbf{b}_0) \left(-p(\mathbf{a}_1 \mid \mathbf{b}_0) \log p(\mathbf{a}_1 \mid \mathbf{b}_0) - p(\mathbf{a}_2 \mid \mathbf{b}_0) \log p(\mathbf{a}_2 \mid \mathbf{b}_0) \right) > 0.
\]

Since \( -t \log t \geq 0 \) for \( 0 \leq t \leq 1 \) and is strictly positive for \( t \) not equal to 0 or 1, the conditional entropy \( H(\mathbf{A} \mid \mathbf{B}) = 0 \) if and only if \( \mathbf{A} \) is almost surely a function of \( \mathbf{B} \).
\end{proof}

\subsection{On the benefit of the InfoMax loss on inverting the data generation process}
\label{appendix: infomax_inverting_dgp}
To ensure identifiability in the latent space, we leverage recent advances in causal and disentangled representation learning. Suppose the true data-generating process is given by $\mathbf{H}_{t} = g(\mathbf{z}_t)$, where $\mathbf{z}_t$ represents the true latent factors. In the sequential context, we assume that the same function $g$ generates two historical subsequences:
\[
\mathbf{H}_t^f = g(\mathbf{z}_t^f), \quad \mathbf{H}_t^h = g(\mathbf{z}_t^h).
\]
We assume a general dependency of the form:
\[
p(\mathbf{z}_t^f \mid \mathbf{z}_t^h ) = \frac{Q(\mathbf{z}_t^f)}{Z(\mathbf{z}_t^h)} \exp(-d(\mathbf{z}_t^f, \mathbf{z}_t^h)).
\]
Here, $\mathbf{\Phi}$ is an encoder, and we use the InfoMax regularization term as follows:
\[
\mathcal{L}^{(InfoMax)}(\mathbf{\Phi}, d, \mathcal{B}) := -\mathbb{E}_{\mathcal{B}} \left[ \log \frac{\exp(-d(\mathbf{\Phi}(\mathbf{H}_t^f), \mathbf{\Phi}(\mathbf{H}_t^h)))}{\sum_{l = 1}^{|\mathcal{B}|} \exp(-d(\mathbf{\Phi}(\mathbf{H}_{l,t}^f), \mathbf{\Phi}(\mathbf{H}_t^h)))}. \right]
\]
According to \cite{matthes2023towards}, under certain conditions, if the encoder $f$ minimizes $\mathcal{L}^{(InfoMax)}$, then $h = g \circ f$ is a scaled permutation matrix. This result suggests that when the encoder achieves a minimizer for $\mathcal{L}^{(InfoMax)}$, the encoder function $f$ closely approximates an invertible transformation of $g$.

From a causal inference perspective, if $Y_{it}(\omega_{it}) \indep W_{it} \mid \mathbf{H}_{it}$ and $\mathbf{H}_{it} = g(\mathbf{Z}_{it})$, then an invertible function $g \circ f$ ensures that:
\[
Y_{it}(\omega_{it}) \indep W_{it} \mid g \circ f (\mathbf{H}_{it}).
\]
Thus, $Y_{it}(\omega_{it}) \indep W_{it} \mid g(\mathbf{C}_{it})$ and since $g$ is invertible, we have:
\[
Y_{it}(\omega_{it}) \indep W_{it} \mid \mathbf{C}_{it}.
\]
This demonstrates that the representation $\mathbf{C}_{it}$ retains the essential independence structure, facilitating accurate counterfactual inference.

\subsection{Proof of theorem \ref{thm: blancing_iclub}} 

To prove the Theorem \ref{thm: blancing_iclub}, we first prove the following lemma and proposition.
\begin{lemma}
    Let $\Phi$ be a fixed representation function. Given that $q(W_{t+1} \mid \Phi(\mathbf{H}_t)) $ is the conditional likelihood of observing the treatment $W_{t+1}$, denote the probability of observing each treatment value as $q^j=q(\Phi(\mathbf{H}_t)):= q(W_{t+1} = j \mid \Phi(\mathbf{H}_t))$ for $j \in \{0,1,\dots, K-1\}$. Then, the optimal treatment prediction function is such that 
    \begin{equation}
        q^{j, *}(\Phi(\mathbf{H}_t) ) = \frac{p(\Phi(\mathbf{H}_t)\mid W_{t+1} = j)}{\sum_{l=0}^{K-1}p(\Phi(\mathbf{H}_t)\mid W_{t+1} = l)p(W_{t+1} = l)}
    \end{equation}
    \label{lem: optimal_q}
\end{lemma}
\begin{proof}
    For a fixed representation $\Phi$, finding the optimal treatment probabilities amounts to solving the following optimization problem:
    \begin{equation}
        \max_{q} \mathbb{E}_{\mathbb{P}(\Phi(\mathbf{H}_t), W_{t+1})} \left[ \log q(W_{t+1} \mid \Phi(\mathbf{H}_t)) \right]  \quad \textrm{subject to}  \quad \sum_{l=0}^{K-1}q^{l}(\Phi(\mathbf{H}_t)) = 1
        \label{eq: max_treat_constrained}
    \end{equation}

First, we write the likelihood $q(W_{t+1} \mid \Phi(\mathbf{H}_t))$ using the conditional probabilities $q^j(\Phi(\mathbf{H}_t))$.
    \begin{equation*}
        q(W_{t+1} \mid \Phi(\mathbf{H}_t)) = \prod_{j=0}^{K-1}q^j(\Phi(\mathbf{H}_t))^{\mathds{1}_{\{W_{t+1}=j\}}}
    \end{equation*}
    
Then, the treatment likelihood can be written as   
\begin{equation*}
        \begin{aligned}
        \mathbb{E}_{\mathbb{P}(\Phi(\mathbf{H}_t), W_{t+1})} \left[ \log q(W_{t+1} \mid \Phi(\mathbf{H}_t)) \right] &=  \mathbb{E}_{\mathbb{P}(\Phi(\mathbf{H}_t), W_{t+1})}  \left[ \sum_{l=0}^{K-1} \log(q^{l}(\Phi(\mathbf{H}_t)))\mathds{1}_{\{W_{t+1}=j\}} \right]   \\ 
        &=  \sum_{l=0}^{K-1} \int\log(q^{l}(\Phi(\mathbf{H}_t))\mathds{1}_{\{W_{t+1}=j\}} p(W_{t+1}\mid \Phi(\mathbf{H}_t))p(\Phi(\mathbf{H}_t))dW_{t+1}d\Phi(\mathbf{H}_t) \\
         &= \sum_{l=0}^{K-1} \int \log(q^{l}(\Phi(\mathbf{H}_t)) p(W_{t+1} = l\mid \Phi(\mathbf{H}_t))p(\Phi(\mathbf{H}_t))d\Phi(\mathbf{H}_t) \\ 
         &= \sum_{l=0}^{K-1} \int\log(q^{l}(\Phi(\mathbf{H}_t)) p(\Phi(\mathbf{H}_t)\mid W_{t+1} = l ) p(W_{t+1} = l) d\Phi(\mathbf{H}_t) 
        \end{aligned}
    \end{equation*}

Let's denote $\alpha_l = p(W_{t+1} = l)$,  the marginal probability of observing the $l$-th treatment regime, and $p_{l}^{\Phi}(\mathbf{H}_t) = p(\Phi(\mathbf{H}_t)\mid W_{t+1} = l )$ with a corresponding probability distribution $\mathbb{P}_{l}^{\Phi}$.  We intend to maximize point-wise the objective in Eq. \eqref{eq: max_treat_constrained}. Plugging the latter formulation of the conditional likelihood in Eq. \eqref{eq: max_treat_constrained}  and writing the Lagrangian function, we get

\begin{equation}
    \max_{q} \sum_{l=0}^{K-1} \log(q^{l}(\Phi(\mathbf{H}_t))p_{j}^{\Phi}(\mathbf{H}_t) \alpha_l + \lambda (\sum_{l=0}^{K-1}q^{l}(\Phi(\mathbf{H}_t)) - 1)  
\end{equation}

Computing the gradient w.r.t $q^{l}(\Phi(\mathbf{H}_t))$ for  $l \in \{0,1,\dots, K-1\}$ and setting to zero, we have

\begin{equation}
    q^{l, *}(\Phi(\mathbf{H}_t)) = -\frac{\alpha_l p_{j}^{\Phi}(\mathbf{H}_t)}{\lambda}
\end{equation}
Then, by the equality constraint, we find that $\lambda = -\sum_{l=0}^{K-1}\alpha_l p_{j}^{\Phi}(\mathbf{H}_t)$.
\end{proof}
\begin{proposition}
 \label{prop: iclub_form}
    Let $\Phi$ be a fixed representation function.  The $I_{CLUB}$ objective when the treatment prediction function is optimal (i.e. $ q=q^{*})$ has the following form:  
     \begin{equation}
     \label{eq: i_club_optimal_q}
I_{CLUB} = \sum_{j=0}^{K-1}\alpha_l D_{KL}(\mathbb{P}_j^{\Phi}|| \sum_{l=0}^{K-1}\alpha_l \mathbb{P}_l^{\Phi}) 
+ \mathbb{E}_{\mathbb{P}_{\Phi(\mathbf{H}_t)}}\left[ D_{KL}(\mathbb{P}_{W_{t+1}}||\mathbb{P}_{W_{t+1}|\Phi(\mathbf{H}_t)} )  \right].
     \end{equation}
\end{proposition}
\begin{proof}
    First, recall that 
    \begin{equation*}
        I_{\text{CLUB}}(\Phi(\mathbf{H}_t), W_{t+1}; q^{*}) =  \mathbb{E}_{\mathbb{P}_{(\Phi(\mathbf{H}_{t+1}), W_{t+1})}}\left[ \log q^{*}(W_{t+1} \mid \Phi(\mathbf{H}_{t+1})) \right] - \left.\mathbb{E}_{\mathbb{P}_{\Phi(\mathbf{H}_{t+1})}}\mathbb{E}_{\mathbb{P}_{W_{t+1}}}\left( \log q^{*}(W_{t+1} \mid \Phi(\mathbf{H}_{t+1})) \right) \right]
    \end{equation*}
     \begin{equation*}
        I_{\text{CLUB}}(\Phi(\mathbf{H}_t), W_{t+1}; q^{*}) =  A - B
    \end{equation*}
    Let's detail $A$ and $B$ separately, 

    \begin{equation*}
        \begin{aligned}
        A &= \sum_{j=0}^{K-1} \int \alpha_j \log(q^{l, *}(\Phi(\mathbf{H}_t)) p_{j}^{\Phi}(\mathbf{H}_t) d\Phi(\mathbf{H}_t) \\ 
        &= \sum_{j=0}^{K-1} \int\alpha_j \log(\frac{\alpha_j p_{j}^{\Phi}(\mathbf{H}_t)}{\sum_{l=0}^{K-1}p_{l}^{\Phi}(\mathbf{H}_t)\alpha_l}) p_{j}^{\Phi}(\mathbf{H}_t) d\Phi(\mathbf{H}_t) \\ 
        &= \sum_{j=0}^{K-1} \int \alpha_j \log(\frac{p_{j}^{\Phi}(\mathbf{H}_t)}{\sum_{l=0}^{K-1}p_{l}^{\Phi}(\mathbf{H}_t)\alpha_l}) p_{j}^{\Phi}(\mathbf{H}_t) d\Phi(\mathbf{H}_t) +\log(\alpha_j)\alpha_j \\
         &= \sum_{j=0}^{K-1} \alpha_j D_{KL}(\mathbb{P}_j^{\Phi}|| \sum_{l=0}^{K-1}\alpha_l \mathbb{P}_l^{\Phi}) +  \sum_{j=0}^{K-1} \log(\alpha_j)\alpha_j \\ 
        \end{aligned}
    \end{equation*}

    Finally, we can write 
    \begin{equation}
    A = \sum_{j=0}^{K-1} \alpha_j D_{KL}(\mathbb{P}_j^{\Phi}|| \sum_{l=0}^{K-1}\alpha_l \mathbb{P}_l^{\Phi})  - H(W_{t+1}) 
        \label{eq: A_final_form}
    \end{equation}

    For the remaining term $B$, we have
     \begin{equation*}
        \begin{aligned}
        B &= \mathbb{E}_{\mathbb{P}_{\Phi(\mathbf{H}_t)}}\mathbb{E}_{\mathbb{P}_{W_{t+1}}}\left( \log q^{*}(W_{t+1} \mid \Phi(\mathbf{H}_{t+1})) \right] \\ 
        &= \sum_{j=0}^{K-1} \mathbb{E}_{\mathbb{P}_{\Phi(\mathbf{H}_t)}} \mathbb{E}_{\mathbb{P}_{W_{t+1}}} \left[ \log(q^{j}(\Phi(\mathbf{H}_t)))\mathds{1}_{\{W_{t+1}=j\}} \right] \\ 
        &= \sum_{j=0}^{K-1} \mathbb{E}_{\mathbb{P}_{\Phi(\mathbf{H}_t)}}\left[ \alpha_j \log(q^{j}(\Phi(\mathbf{H}_t)))\right] \\ 
        &= \sum_{j=0}^{K-1}\alpha_j \int \log\left[ \frac{\alpha_j p_{j}^{\Phi}(\mathbf{H}_t)}{\sum_{l=0}^{K-1}p_{l}^{\Phi}(\mathbf{H}_t)\alpha_l}\right]p(\Phi(\mathbf{H}_t))d\Phi(\mathbf{H}_t) \\ 
       &= \sum_{j=0}^{K-1}\alpha_j \int \log\left[ \frac{p(\Phi(\mathbf{H}_t))}{\sum_{l=0}^{K-1}p_{l}^{\Phi}(\mathbf{H}_t)\alpha_l} \frac{p(W_{t+1} =j \mid \Phi(\mathbf{H}_t))}{p(W_{t+1} =j)}\right]p(\Phi(\mathbf{H}_t))d\Phi(\mathbf{H}_t)\\ 
       &- H(W_{t+1})  \\ 
       &= \sum_{j=0}^{K-1}\alpha_j \int \underbrace{\log\left[ \frac{p(\Phi(\mathbf{H}_t))}{\sum_{l=0}^{K-1}p_{l}^{\Phi}(\mathbf{H}_t)\alpha_l}\right]}_{=0}p(\Phi(\mathbf{H}_t))d\Phi(\mathbf{H}_t) \\ 
        &+ \sum_{j=0}^{K-1}\alpha_j \int \log\left[\frac{p(W_{t+1} =j \mid \Phi(\mathbf{H}_t))}{p(W_{t+1} =j)}\right] p(\Phi(\mathbf{H}_t))d\Phi(\mathbf{H}_t) - H(W_{t+1}) \\
        \end{aligned}
    \end{equation*}
    The final form of $B$ is therefore 
    \begin{equation}
    B = - \int D_{KL}(\mathbb{P}_{W_{t+1}}| \mathbb{P}_{W_{t+1}| \Phi(\mathbf{H}_t)})p(\Phi(\mathbf{H}_t)) d\Phi(\mathbf{H}_t)
       - H(W_{t+1}) 
        \label{eq: B_final_form}
    \end{equation}

    The proposition follows immediately from Equations \eqref{eq: A_final_form} and \eqref{eq: B_final_form}.
\end{proof}

\begin{proof}(Theorem \ref{thm: blancing_iclub})
Since by lemma \ref{lem: optimal_q}, the $I_{CLUB}$ formulation in  proposition \ref{prop: iclub_form} holds, then to prove that the representation is balanced, it is enough to see that by the positivity of $D_{KL}$

\begin{equation}
    I_{CLUB} \geq  \mathbb{E}_{\mathbb{P}_{\Phi(\mathbf{H}_t)}}\left[ D_{KL}(\mathbb{P}_{W_{t+1}}||\mathbb{P}_{W_{t+1}|\Phi(\mathbf{H}_t)} )\right] \geq 0  
    \label{eq: iclub_ineq_at_optimal}
\end{equation}
$I_{CLUB}$ is minimal when $I_{CLUB} =0$, which happens if and only if for $j \in \{0,1,\dots, K-1\}$  $p(W_{t+1} =j ) = p(W_{t+1} =j \mid \Phi(\mathbf{H}_t))$ almost surely which, by Bayes rule is equivalent to say $p(\Phi(\mathbf{H}_t)) = p(\Phi(\mathbf{H}_t) \mid W_{t+1} =j) $.


        
    
\end{proof}
\section{Causal CPC Pseudo algorithm}
\label{sect: alg_ccpc}
In this section, we present a detailed overview of the training procedure for Causal CPC. Initially, we train the Encoder using only the contrastive terms, as outlined in Algorithm \ref{alg: enc_pretraining}. Our primary objective is to ensure that, for each time step $t$, the process history $\mathbf{H}_{t}$ is predictive of future local features $\mathbf{Z}_t$. However, calculating the InfoNCE loss for a batch across all possible time steps $t=0, \dots, t_{\text{max}}$ can be computationally demanding.

To address this, we adopt a more efficient approach by uniformly sampling a single time step $t$ per batch. Subsequently, the corresponding process history $\mathbf{H}_{t}$ is contrasted. The sampled $\mathbf{H}_{t}$ is then employed as input for the InfoMax objective and randomly partitioned into future $\mathbf{H}_t^f$ and past $\mathbf{H}_t^h$ sub-processes.
\begin{algorithm}
   \caption{Pretraining of the encoder}
   \label{alg: enc_pretraining}
\begin{algorithmic}
   \REQUIRE Encoder parameters $\theta_{1,2,3}$, learning rate $\mu$
   \STATE {\bfseries Input:} data $\{\mathbf{H}_{i, t_{max}}, i=1,\dots, N \}$ 
    \FOR{$p \in \{1, \dots,\mathrm{epoch_{max}}\}$} 
            \FOR{$\mathcal{B}=\{ \mathbf{H}_{i, t_{max}}, i=1,\dots, |\mathcal{B}| \}$}
                    \STATE $\mathbf{Z}_t = \Phi_{\theta_1}([\mathbf{X}_{t}, W_{t-1}, Y_{t-1}])$ for $t=0, \dots, t_{max}$.
                    \STATE Choose $t \sim \mathcal{U}([1, t_{max}-1])$.
                    \STATE Compute $\mathbf{C}_t = \Phi_{\theta_{1},\theta_{2}}(\mathbf{H}_{t})$.
                    \STATE  Compute $\mathcal{L}^{CPC}(\theta_{1},\theta_{2}, \{\Gamma_j\}_{j=1}^{\tau})$.
                    \STATE  Choose $t_0 \sim \mathcal{U}([1, t])$.
                    \STATE  Compute $\mathbf{C}_t^h=\Phi_{\theta_{1},\theta_{2}}(\mathbf{H}_t^h), \mathbf{C}_t^f=\Phi_{\theta_{1},\theta_{2}}(\mathbf{H}_t^f)$,
                    \STATE Compute $\mathcal{L}^{(InfoMax)}(\theta_{1},\theta_{2}, \gamma)$.
                    \STATE Update parameters
                    $$\theta_{1,2,3} \leftarrow  \theta_{1,2,3} - \mu  \left( \frac{\partial \mathcal{L}^{CPC}(\theta_{1},\theta_{2}, \{\Gamma_j\}_{j=1}^{\tau})}{\partial \theta_{1,2,3} } + \frac{\partial \mathcal{L}^{(InfoMax)}(\theta_{1},\theta_{2}, \gamma)}{\partial \theta_{1,2,3} } \right  ) $$
            \ENDFOR
    \ENDFOR
   \STATE \textbf{Return}: Trained encoder.
\end{algorithmic}
\end{algorithm}

The decoder is trained while taking the encoder as input (Algorithm \ref{alg: dec_training}), utilizing a lower learning rate compared to the untrained part of the decoder. It is trained autoregressively and without teacher forcing. This implies that for each time step $t$, our GRU-based decoder should predict the future sequence of treatments $\hat{Y}_{t+1:t+\tau}$ with its hidden state initialized to the representation $\mathbf{\Phi}_t$ of the historical process up to time $t$.

To enhance training efficiency, instead of predicting $\hat{Y}_{i,t+1:t+\tau}$ for all individuals $i$ in a batch and for all possible time steps $t$, we randomly select $m$ time indices $t_{i,1}, \dots, t_{i,m}$ for each individual $i$. From these indices, we compute future treatment response sequences $\hat{Y}_{i,t_{i,1}+1:t_{i,1}+\tau}, \dots, \hat{Y}_{i,t_{i,m}+1:t_{i,m}+\tau}$.  We found that is enough to train while selecting randomly 10\% of the time steps.

\begin{algorithm}
   \caption{Training of the decoder}
   \label{alg: dec_training}
\begin{algorithmic}
   \REQUIRE Encoder parameters $\theta_{1,2,3}$, Decoder parameters $\theta_{4}, \theta_{Y}, \theta_{W}$.
    \REQUIRE Encoder learning rate $\mu_{enc}$, Treatment learning rate $\mu_{W}$ ,  Outcome learning rate $\mu_{Y}$.
    \REQUIRE Number of random time indices $m$.
   \STATE {\bfseries Input:} data $\{\mathbf{H}_{i, t_{max}}, i=1,\dots, N \}$ 
    \FOR{$p \in \{1, \dots,\mathrm{epoch_{max}}\}$} 
            \FOR{$\mathcal{B}=\{ \mathbf{H}_{i, t_{max}}, i=1,\dots, |\mathcal{B}| \}$}
                    \STATE Compute $\mathbf{C}_{i, t} = \mathrm{encoder}(\mathbf{H}_{i, t})$ for $t=0, \dots, t_{max}$, $i=1,\dots, |\mathcal{B}| $.
                    \STATE Compute $\mathbf{\Phi}_t = \mathbf{\Phi}_{\theta_R}(\mathbf{H}_t)$.
                    \FOR{$i=1,\dots, |\mathcal{B}| $}
                            \STATE Choose $t_{i,1}, \dots, t_{i,m} \sim \mathcal{U}([1, t_{max}-\tau]) $.
                            \FOR{$t  \in  \{ t_{i,1}, \dots, t_{i,m}\}$}
                                    \STATE Compute $\hat{Y}_{i,t+1:t+\tau}, \hat{W}_{i,t+1:t+\tau}, \mathbf{\Phi}_{i,t+1:t+\tau-1} = \mathrm{decoder}(\mathbf{\Phi}_t, \mathbf{V}_i, W_{i,t}, Y_{i,t}, W_{i,t+1:t+\tau}) $
                            \ENDFOR
                    \ENDFOR
                    \STATE Compute  $\mathcal{L}_{dec}(\theta_{R}, \theta_{Y}, \theta_{W})$ and $\mathcal{L}_{W}(\theta_{W}, \theta_{R})$.
                    \STATE Update parameters in the order.
                    $$\theta_{1,2,3} \leftarrow  \theta_{1,2,3} - \mu_{enc}  \left( \frac{\partial \mathcal{L}_{dec}(\theta_{R}, \theta_{Y}, \theta_{W})}{\partial \theta_{1,2,3} }\right  ) $$
                    $$\theta_{4, Y} \leftarrow  \theta_{4, Y}  - \mu_{Y}  \left( \frac{\partial \mathcal{L}_{dec}(\theta_{R}, \theta_{Y}, \theta_{W})}{\partial \theta_{4, Y}  }\right  ) $$                     $$\theta_{W} \leftarrow  \theta_{W}  - \mu_{W}  \left( \frac{\partial  \mathcal{L}_{W}(\theta_{W}, \theta_{R}) }{\partial \theta_{W} }\right  ) $$
            \ENDFOR
    \ENDFOR
   \STATE \textbf{Return}: Trained decoder.
\end{algorithmic}
\end{algorithm}

\section{Causal CPC: Architecture details}
\label{sect: archi_ccpc}
\begin{table}[!htbp]
\centering
\begin{tabular}{c}
\hline \textbf{Inputs}: $[\mathbf{X}_{t}, W_{t-1}, Y_{t-1}]$\\
\hline Linear Layer\\
\hline WeightNorm\\
\hline SELU\\
\hline Linear Layer  \\
\hline WeightNorm\\
\hline \textbf{Outputs}: $\mathbf{Z}_t = \Phi_{\theta_1}([\mathbf{X}_{t}, W_{t-1}, Y_{t-1}])$\\
\end{tabular}
\caption{Architecture for learning local features $\mathbf{Z}_t$}
\label{tab: archi_loc_feat}
\end{table}

\begin{table}[!htbp]
\centering
\begin{tabular}{c}
\hline \textbf{Inputs}: $\mathbf{Z}_{\leq t}$\\
\hline GRU (1 layer)\\
\hline \textbf{Outputs}: Hidden state $\mathbf{C}_t = \Phi_{\theta_2}^{ar}(\mathbf{Z}_{\leq t})$\\
\end{tabular}
\caption{Architecture for learning context representation $\mathbf{C}_t$ }
\label{tab: archi_context }
\end{table}

\begin{table}[!htbp]
\centering
\begin{tabular}{c}
\hline \textbf{Inputs}: $[\mathbf{\Phi}_t, W_t]$\\
\hline Linear Layer\\
\hline WeightNorm\\
\hline SELU\\
\hline Linear Layer  \\
\hline WeightNorm\\
\hline \textbf{Outputs}: $\hat{Y}_t$\\
\end{tabular}
\caption{Architecture for outcome prediction}
\label{tab: archi_outc_pred}
\end{table}

\begin{table}[!htbp]
\centering
\begin{tabular}{c}
\hline \textbf{Inputs}: $\mathbf{\Phi}_t$\\
\hline Linear Layer\\
\hline SpectralNorm\\
\hline SELU\\
\hline Linear Layer  \\
\hline SpectralNorm\\
\hline \textbf{Outputs}: $\hat{W}_t$\\
\end{tabular}
\caption{Architecture for treatment prediction}
\label{tab: archi_treat_pred}
\end{table}
\section{Models hyperparameters}
\label{sect: huperparams_details}

In this section, we report the range of all hyperparameters to be fine-tuned, as well as fixed hyperparameters for all models and across the different datasets used in experiments. Best hyperparameter values are reported in the configuration files in the code repository.

\begin{table}[!htbp]
\centering
\caption{Hyper-parameters search range for RMSN}
\resizebox{0.9\textwidth}{!}{%
\begin{tabular}{|c|c|c|c|c|}
\hline
\textbf{Model} & \textbf{Sub-model} & \textbf{Hyperparameter} & \textbf{Cancer simulation} & \textbf{MIMIC III (SS)} \\
\hline
\multirow{7}{*}{RMSNs} 
& \multirow{7}{*}{Propensity Treatment Network} 
& LSTM layers & 1 & 1 \\
\cline{3-5}
& & Learning rate & $0.01, 0.005, 0.001, 0.0001$ & $0.01, 0.005, 0.001, 0.0001$ \\
\cline{3-5}
& & Batch size & $32, 64, 128$ & $32, 64, 128$ \\
\cline{3-5}
& & LSTM hidden units & $4, 6, \dots, 12$ & $4, 6, \dots, 30$ \\
\cline{3-5}
& & LSTM dropout rate & - & - \\
\cline{3-5}
& & Max gradient norm & $0.5, 1, 2$ & $0.5, 1, 2$ \\
\cline{3-5}
& & Early Stopping (min delta) & 0.0001 & 0.0001 \\
\cline{3-5}
& & Early Stopping (patience) & 30 & 30 \\
\hline
\multirow{6}{*}{Propensity History Network} 
& & LSTM layers & 1 & 1 \\
\cline{3-5}
& & Learning rate & $0.01, 0.005, 0.001, 0.0001$ & $0.01, 0.005, 0.001, 0.0001$ \\
\cline{3-5}
& & Batch size & $32, 64, 128$ & $64, 128, 256$ \\
\cline{3-5}
& & LSTM hidden units & $4, 6, \dots, 20$ & $4, 6, \dots, 30$ \\
\cline{3-5}
& & LSTM dropout rate & - & - \\
\cline{3-5}
& & Early Stopping (min delta) & 0.0001 & 0.0001 \\
\cline{3-5}
& & Early Stopping (patience) & 30 & 30 \\
\hline
\multirow{6}{*}{Encoder} 
& & LSTM layers & 1 & 1 \\
\cline{3-5}
& & Learning rate & $0.01, 0.005, 0.001, 0.0001$ & $0.01, 0.005, 0.001, 0.0001$ \\
\cline{3-5}
& & Batch size & $32, 64, 128$ & $32, 64, 128$ \\
\cline{3-5}
& & LSTM hidden units & $4, 6, \dots, 20$ & $4, 6, \dots, 30$ \\
\cline{3-5}
& & LSTM dropout rate & - & - \\
\cline{3-5}
& & Early Stopping (min delta) & 0.0001 & 0.0001 \\
\cline{3-5}
& & Early Stopping (patience) & 30 & 30 \\
\hline
\multirow{6}{*}{Decoder} 
& & LSTM layers & 1 & 1 \\
\cline{3-5}
& & Learning rate & $0.01, 0.005, 0.001, 0.0001$ & $0.01, 0.005, 0.001, 0.0001$ \\
\cline{3-5}
& & Batch size & $32, 64, 128$ & $128, 512, 1024$ \\
\cline{3-5}
& & LSTM hidden units & $4, 6, \dots, 20$ & $4, 6, \dots, 30$ \\
\cline{3-5}
& & LSTM dropout rate & - & - \\
\cline{3-5}
& & Max gradient norm & $0.5, 1, 2$ & $0.5, 1, 2$ \\
\cline{3-5}
& & Early Stopping (min delta) & 0.0001 & 0.0001 \\
\cline{3-5}
& & Early Stopping (patience) & 30 & 30 \\
\hline
\end{tabular}%
}
\end{table}

\begin{table}[!htbp]
\centering
\caption{Hyper-parameters search range for CRN}
\resizebox{0.7\textwidth}{!}{%
\begin{tabular}{|c|c|c|c|c|}
\hline
\textbf{Model} & \textbf{Sub-model} & \textbf{Hyperparameter} & \textbf{Cancer simulation} & \textbf{MIMIC III (SS)} \\
\hline
\multirow{8}{*}{CRN} 
& \multirow{8}{*}{Encoder} 
& LSTM layers & 1 & 1 \\
\cline{3-5}
& & Learning rate & $0.01, 0.005, 0.001, 0.0001$ & $0.01, 0.005, 0.001, 0.0001$ \\
\cline{3-5}
& & Batch size & $32, 64, 128$ & $32, 64, 128$ \\
\cline{3-5}
& & LSTM hidden units & $4, 6, \dots, 30$ & $4, 6, \dots, 30$ \\
\cline{3-5}
& & LSTM dropout rate & - & - \\
\cline{3-5}
& & BR size & $4, 6, \dots, 20$ & $4, 6, \dots, 30$ \\
\cline{3-5}
& & Early Stopping (min delta) & 0.0001 & 0.0001 \\
\cline{3-5}
& & Early Stopping (patience) & 30 & 30 \\
\hline
\multirow{7}{*}{Decoder} 
& & LSTM layers & 1 & 1 \\
\cline{3-5}
& & Learning rate & $0.01, 0.005, 0.001, 0.0001$ & $0.01, 0.005, 0.001, 0.0001$ \\
\cline{3-5}
& & Batch size & $128, 256, 512$ & $256, 512, 1024$ \\
\cline{3-5}
& & LSTM hidden units & $4, 6, \dots, 30$ & $4, 6, \dots, 30$ \\
\cline{3-5}
& & LSTM dropout rate & - & - \\
\cline{3-5}
& & BR size & $4, 6, \dots, 20$ & $4, 6, \dots, 30$ \\
\cline{3-5}
& & Early Stopping (min delta) & 0.0001 & 0.0001 \\
\cline{3-5}
& & Early Stopping (patience) & 30 & 30 \\
\hline
\end{tabular}%
}
\end{table}

\begin{table}[!htbp]
\centering
\caption{Hyper-parameters search range for G-Net}
\resizebox{0.7\textwidth}{!}{%
\begin{tabular}{|c|c|c|}
\hline
\textbf{Hyperparameter} & \textbf{Cancer simulation} & \textbf{MIMIC III (SS)} \\
\hline
LSTM layers & 1 & 1  \\
Learning rate & $0.01,0.005, 0.001, 0.0001$ &  $0.01,0.005, 0.001, 0.0001$ \\
Batch size & $32, 64, 128$ &  $32, 64, 128$ \\
LSTM hidden units & $4,6, \dots, 30$ & $4,6, \dots, 30$ \\
FC hidden units & $4,6, \dots, 30$ &  $4,6, \dots, 30$ \\
LSTM dropout rate & - & -  \\
R size& $4,6, \dots, 20$& $4,6, \dots, 30$ \\
MC samples & 10 & 10 \\
Early Stopping (min delta)& 0.0001&  0.0001\\
Early Stopping (patience)& 30& 30 \\
\cline{1-3}
\end{tabular}%
}
\end{table}

\begin{table}[!htbp]
\centering
\caption{Hyper-parameters search range for Causal Transfomer}
\resizebox{0.7\textwidth}{!}{%
\begin{tabular}{|c|c|c|}
\hline
\textbf{Hyperparameter} & \textbf{Cancer simulation} & \textbf{MIMIC III (SS)} \\
\hline
Transformer blocks & 1 & 1 \\
Learning rate & $0.01,0.005, 0.001, 0.0001$ & $0.01,0.005, 0.001, 0.0001$  \\
Batch size & $32, 64, 128$ & $32, 64, 128$ \\
Attention heads & $2$ & $2$ \\
Transformer units & $4,6, \dots, 20$& $4,6, \dots, 20$ \\
LSTM dropout rate & -  & - \\
BR size& $4,6, \dots, 20$&  $4,6, \dots, 20$ \\
FC hidden units & $4,6, \dots, 20$&   $4,6, \dots, 20$ \\
Sequential dropout rate & $0.1,0.2, 0.3$& $0.1,0.2, 0.3$ \\
Max positional encoding & $15$& $15$  \\
Early Stopping (min delta)& 0.0001& 0.0001 \\
Early Stopping (patience)& 30& 30 \\
\cline{1-3}
\end{tabular}%
}
\end{table}

\begin{table}[!htbp]
\centering
\caption{Hyper-parameters search range for Causal CPC}
\resizebox{0.9\textwidth}{!}{%
\begin{tabular}{|c|c|c|c|c|}
\hline
\textbf{Model} & \textbf{Sub-model} & \textbf{Hyperparameter} & \textbf{Cancer simulation} & \textbf{MIMIC III (SS)} \\
\hline
\multirow{9}{*}{Causal CPC} & \multirow{9}{*}{Encoder} 
& GRU layers & 1 & 1 \\
\cline{3-5}
& & Learning rate & $0.01, 0.005, 0.001, 0.0001$ & $0.01, 0.005, 0.001, 0.0001$ \\
\cline{3-5}
& & Batch size & $32, 64, 128$ & $64, 128, 256$ \\
\cline{3-5}
& & GRU hidden units & $4, 6, \dots, 30$ & $4, 6, \dots, 30$ \\
\cline{3-5}
& & GRU dropout rate & - & - \\
\cline{3-5}
& & Local features (LF) size & $4, 6, \dots, 20$ & $4, 6, \dots, 20$ \\
\cline{3-5}
& & Context Representation (CR) size & $4, 6, \dots, 20$ & $4, 6, \dots, 20$ \\
\cline{3-5}
& & Early Stopping (min delta) & 0.001 & 0.001 \\
\cline{3-5}
& & Early Stopping (patience) & 100 & 100 \\
\hline
\multirow{13}{*}{Decoder} & & GRU layers & 1 & 1 \\
\cline{3-5}
& & Learning rate (decoder w/o treatment sub-network) & $0.01, 0.005, 0.001, 0.0001$ & $0.01, 0.005, 0.001, 0.0001$ \\
\cline{3-5}
& & Learning rate (encoder fine-tuning) & $0.001, 0.0005, 0.0001, 0.00005$ & $0.001, 0.0005, 0.0001, 0.00005$ \\
\cline{3-5}
& & Learning rate (treatment sub-network) & $0.05, 0.01, 0.005, 0.0001$ & $0.05, 0.01, 0.005, 0.0001$ \\
\cline{3-5}
& & Batch size & $32, 64, 128$ & $32, 64, 128$ \\
\cline{3-5}
& & GRU hidden units & CR size & CR size \\
\cline{3-5}
& & GRU dropout rate & - & - \\
\cline{3-5}
& & BR size & CR size & CR size \\
\cline{3-5}
& & GRU layers (Treat Encoder) & 1 & 1 \\
\cline{3-5}
& & GRU hidden units (Treat Encoder) & 6 & 6 \\
\cline{3-5}
& & FC hidden units & $4, 6, \dots, 20$ & $4, 6, \dots, 20$ \\
\cline{3-5}
& & Random time indices (m) & 10\% & 10\% \\
\cline{3-5}
& & Early Stopping (min delta) & 0.001 & 0.001 \\
\cline{3-5}
& & Early Stopping (patience) & 50 & 50 \\
\hline
\end{tabular}%
}
\end{table}

\begin{table}[!htbp]
\centering
\caption{Hyper-parameters search range for Causal CPC}
\resizebox{0.9\textwidth}{!}{%
\begin{tabular}{|c|c|c|c|c|}
\hline
\textbf{Model} & \textbf{Sub-model} & \textbf{Hyperparameter} & \textbf{Cancer simulation} & \textbf{MIMIC III (SS)} \\
\hline
\multirow{9}{*}{Causal CPC} & \multirow{9}{*}{Encoder} 
& GRU layers & 1 & 1 \\
\cline{3-5}
& & Learning rate & $0.01, 0.005, 0.001, 0.0001$ & $0.01, 0.005, 0.001, 0.0001$ \\
\cline{3-5}
& & Batch size & $32, 64, 128$ & $64, 128, 256$ \\
\cline{3-5}
& & GRU hidden units & $4, 6, \dots, 30$ & $4, 6, \dots, 30$ \\
\cline{3-5}
& & GRU dropout rate & - & - \\
\cline{3-5}
& & Local features (LF) size & $4, 6, \dots, 20$ & $4, 6, \dots, 20$ \\
\cline{3-5}
& & Context Representation (CR) size & $4, 6, \dots, 20$ & $4, 6, \dots, 20$ \\
\cline{3-5}
& & Early Stopping (min delta) & 0.001 & 0.001 \\
\cline{3-5}
& & Early Stopping (patience) & 100 & 100 \\
\hline
\multirow{13}{*}{Decoder} & & GRU layers & 1 & 1 \\
\cline{3-5}
& & Learning rate (decoder w/o treatment sub-network) & $0.01, 0.005, 0.001, 0.0001$ & $0.01, 0.005, 0.001, 0.0001$ \\
\cline{3-5}
& & Learning rate (encoder fine-tuning) & $0.001, 0.0005, 0.0001, 0.00005$ & $0.001, 0.0005, 0.0001, 0.00005$ \\
\cline{3-5}
& & Learning rate (treatment sub-network) & $0.05, 0.01, 0.005, 0.0001$ & $0.05, 0.01, 0.005, 0.0001$ \\
\cline{3-5}
& & Batch size & $32, 64, 128$ & $32, 64, 128$ \\
\cline{3-5}
& & GRU hidden units & CR size & CR size \\
\cline{3-5}
& & GRU dropout rate & - & - \\
\cline{3-5}
& & BR size & CR size & CR size \\
\cline{3-5}
& & GRU layers (Treat Encoder) & 1 & 1 \\
\cline{3-5}
& & GRU hidden units (Treat Encoder) & 6 & 6 \\
\cline{3-5}
& & FC hidden units & $4, 6, \dots, 20$ & $4, 6, \dots, 20$ \\
\cline{3-5}
& & Random time indices (m) & 10\% & 10\% \\
\cline{3-5}
& & Early Stopping (min delta) & 0.001 & 0.001 \\
\cline{3-5}
& & Early Stopping (patience) & 50 & 50 \\
\hline
\end{tabular}%
}
\end{table}

\newpage

\section*{NeurIPS Paper Checklist}

\begin{enumerate}

\item {\bf Claims}
    \item[] Question: Do the main claims made in the abstract and introduction accurately reflect the paper's contributions and scope?
    \item[] Answer: \answerYes{} 
    \item[] Justification: The paper introduces a novel method combining RNNs with CPC for long-term counterfactual regression, leveraging MI objectives for efficient representation learning, and demonstrates state-of-the-art results on both synthetic and real-world data. These claims are substantiated by the detailed theoretical \ref{sect: causal_cpc_def} and empirical analyses \ref{sect: experiments} provided in the paper.
    \item[] Guidelines:
    \begin{itemize}
        \item The answer NA means that the abstract and introduction do not include the claims made in the paper.
        \item The abstract and/or introduction should clearly state the claims made, including the contributions made in the paper and important assumptions and limitations. A No or NA answer to this question will not be perceived well by the reviewers. 
        \item The claims made should match theoretical and experimental results, and reflect how much the results can be expected to generalize to other settings. 
        \item It is fine to include aspirational goals as motivation as long as it is clear that these goals are not attained by the paper. 
    \end{itemize}

\item {\bf Limitations}
    \item[] Question: Does the paper discuss the limitations of the work performed by the authors?
    \item[] Answer: \answerYes{}{} 
    \item[] Justification: While Causal CPC excels at large horizon predictions, it does not outperform SOTA models on short-term predictions (Table \ref{tab: perf_mimic_few_data}). 
    \item[] Guidelines:
    \begin{itemize}
        \item The answer NA means that the paper has no limitation while the answer No means that the paper has limitations, but those are not discussed in the paper. 
        \item The authors are encouraged to create a separate "Limitations" section in their paper.
        \item The paper should point out any strong assumptions and how robust the results are to violations of these assumptions (e.g., independence assumptions, noiseless settings, model well-specification, asymptotic approximations only holding locally). The authors should reflect on how these assumptions might be violated in practice and what the implications would be.
        \item The authors should reflect on the scope of the claims made, e.g., if the approach was only tested on a few datasets or with a few runs. In general, empirical results often depend on implicit assumptions, which should be articulated.
        \item The authors should reflect on the factors that influence the performance of the approach. For example, a facial recognition algorithm may perform poorly when image resolution is low or images are taken in low lighting. Or a speech-to-text system might not be used reliably to provide closed captions for online lectures because it fails to handle technical jargon.
        \item The authors should discuss the computational efficiency of the proposed algorithms and how they scale with dataset size.
        \item If applicable, the authors should discuss possible limitations of their approach to address problems of privacy and fairness.
        \item While the authors might fear that complete honesty about limitations might be used by reviewers as grounds for rejection, a worse outcome might be that reviewers discover limitations that aren't acknowledged in the paper. The authors should use their best judgment and recognize that individual actions in favor of transparency play an important role in developing norms that preserve the integrity of the community. Reviewers will be specifically instructed to not penalize honesty concerning limitations.
    \end{itemize}

\item {\bf Theory Assumptions and Proofs}
    \item[] Question: For each theoretical result, does the paper provide the full set of assumptions and a complete (and correct) proof?
    \item[] Answer: \answerYes{}{} 
    \item[] Justification: Detailed proofs are provided in Appendix \ref{sect: proofs}.
    \item[] Guidelines:
    \begin{itemize}
        \item The answer NA means that the paper does not include theoretical results. 
        \item All the theorems, formulas, and proofs in the paper should be numbered and cross-referenced.
        \item All assumptions should be clearly stated or referenced in the statement of any theorems.
        \item The proofs can either appear in the main paper or the supplemental material, but if they appear in the supplemental material, the authors are encouraged to provide a short proof sketch to provide intuition. 
        \item Inversely, any informal proof provided in the core of the paper should be complemented by formal proofs provided in appendix or supplemental material.
        \item Theorems and Lemmas that the proof relies upon should be properly referenced. 
    \end{itemize}

    \item {\bf Experimental Result Reproducibility}
    \item[] Question: Does the paper fully disclose all the information needed to reproduce the main experimental results of the paper to the extent that it affects the main claims and/or conclusions of the paper (regardless of whether the code and data are provided or not)?
    \item[] Answer: \answerYes{}{} 
    \item[] Justification:  All the experimental details related to the training and evaluation protocol (\ref{sect: exp_protocol}), datasets descriptions (\ref{subsect: cancer_sim_descrip} and \ref{subsect: details_sim_mimic}) are provided. A detailed description of the Causal CPC architecture is provided in \ref{sect: archi_ccpc}. Pseudo-algorithms for both the encoder and the decoder are also provided in \ref{sect: alg_ccpc}. Code is provided in the supplementary material.

    \item[] Guidelines:
    \begin{itemize}
        \item The answer NA means that the paper does not include experiments.
        \item If the paper includes experiments, a No answer to this question will not be perceived well by the reviewers: Making the paper reproducible is important, regardless of whether the code and data are provided or not.
        \item If the contribution is a dataset and/or model, the authors should describe the steps taken to make their results reproducible or verifiable. 
        \item Depending on the contribution, reproducibility can be accomplished in various ways. For example, if the contribution is a novel architecture, describing the architecture fully might suffice, or if the contribution is a specific model and empirical evaluation, it may be necessary to either make it possible for others to replicate the model with the same dataset, or provide access to the model. In general. releasing code and data is often one good way to accomplish this, but reproducibility can also be provided via detailed instructions for how to replicate the results, access to a hosted model (e.g., in the case of a large language model), releasing of a model checkpoint, or other means that are appropriate to the research performed.
        \item While NeurIPS does not require releasing code, the conference does require all submissions to provide some reasonable avenue for reproducibility, which may depend on the nature of the contribution. For example
        \begin{enumerate}
            \item If the contribution is primarily a new algorithm, the paper should make it clear how to reproduce that algorithm.
            \item If the contribution is primarily a new model architecture, the paper should describe the architecture clearly and fully.
            \item If the contribution is a new model (e.g., a large language model), then there should either be a way to access this model for reproducing the results or a way to reproduce the model (e.g., with an open-source dataset or instructions for how to construct the dataset).
            \item We recognize that reproducibility may be tricky in some cases, in which case authors are welcome to describe the particular way they provide for reproducibility. In the case of closed-source models, it may be that access to the model is limited in some way (e.g., to registered users), but it should be possible for other researchers to have some path to reproducing or verifying the results.
        \end{enumerate}
    \end{itemize}

\item {\bf Open access to data and code}
    \item[] Question: Does the paper provide open access to the data and code, with sufficient instructions to faithfully reproduce the main experimental results, as described in supplemental material?
    \item[] Answer: \answerYes{} 
    \item[] Justification: Code is provided in the supplementary material at submission.
    \item[] Guidelines:
    \begin{itemize}
        \item The answer NA means that paper does not include experiments requiring code.
        \item Please see the NeurIPS code and data submission guidelines (\url{https://nips.cc/public/guides/CodeSubmissionPolicy}) for more details.
        \item While we encourage the release of code and data, we understand that this might not be possible, so “No” is an acceptable answer. Papers cannot be rejected simply for not including code, unless this is central to the contribution (e.g., for a new open-source benchmark).
        \item The instructions should contain the exact command and environment needed to run to reproduce the results. See the NeurIPS code and data submission guidelines (\url{https://nips.cc/public/guides/CodeSubmissionPolicy}) for more details.
        \item The authors should provide instructions on data access and preparation, including how to access the raw data, preprocessed data, intermediate data, and generated data, etc.
        \item The authors should provide scripts to reproduce all experimental results for the new proposed method and baselines. If only a subset of experiments are reproducible, they should state which ones are omitted from the script and why.
        \item At submission time, to preserve anonymity, the authors should release anonymized versions (if applicable).
        \item Providing as much information as possible in supplemental material (appended to the paper) is recommended, but including URLs to data and code is permitted.
    \end{itemize}

\item {\bf Experimental Setting/Details}
    \item[] Question: Does the paper specify all the training and test details (e.g., data splits, hyperparameters, how they were chosen, type of optimizer, etc.) necessary to understand the results?
    \item[] Answer: \answerYes{}{} 
    \item[] Justification: Models' hyperparameter search range is provided in Appendix \ref{sect: huperparams_details}, the selection method is provided at the beginning of Section \ref{sect: experiments}, and remaining details about training and testing are provided in the experimental protocol (Appendix \ref{sect: exp_protocol}).

    \item[] Guidelines:
    \begin{itemize}
        \item The answer NA means that the paper does not include experiments.
        \item The experimental setting should be presented in the core of the paper to a level of detail that is necessary to appreciate the results and make sense of them.
        \item The full details can be provided either with the code, in appendix, or as supplemental material.
    \end{itemize}

\item {\bf Experiment Statistical Significance}
    \item[] Question: Does the paper report error bars suitably and correctly defined or other appropriate information about the statistical significance of the experiments?
    \item[] Answer: \answerNo{} 
    \item[] Justification: It is computationally demanding to compute errors bars for all neural network models in our benchmark. However, we reported the mean and standard deviation of metrics for each experiment, computed from multiple runs.
    \item[] Guidelines:
    \begin{itemize}
        \item The answer NA means that the paper does not include experiments.
        \item The authors should answer "Yes" if the results are accompanied by error bars, confidence intervals, or statistical significance tests, at least for the experiments that support the main claims of the paper.
        \item The factors of variability that the error bars are capturing should be clearly stated (for example, train/test split, initialization, random drawing of some parameter, or overall run with given experimental conditions).
        \item The method for calculating the error bars should be explained (closed form formula, call to a library function, bootstrap, etc.)
        \item The assumptions made should be given (e.g., Normally distributed errors).
        \item It should be clear whether the error bar is the standard deviation or the standard error of the mean.
        \item It is OK to report 1-sigma error bars, but one should state it. The authors should preferably report a 2-sigma error bar than state that they have a 96\% CI, if the hypothesis of Normality of errors is not verified.
        \item For asymmetric distributions, the authors should be careful not to show in tables or figures symmetric error bars that would yield results that are out of range (e.g. negative error rates).
        \item If error bars are reported in tables or plots, The authors should explain in the text how they were calculated and reference the corresponding figures or tables in the text.
    \end{itemize}

\item {\bf Experiments Compute Resources}
    \item[] Question: For each experiment, does the paper provide sufficient information on the computer resources (type of compute workers, memory, time of execution) needed to reproduce the experiments?
    \item[] Answer: \answerYes{}{} 
    \item[] Justification: We provided the computation resources used in the title of Table \ref{tab: complexity_running_time_cancer} as well as the time of execution in the same table for cancer simulation data. A similar table is provided in Appendix \ref{subsect: detailed_results_mimic} for MIMIC III data (Table \ref{tab: complexity_running_time_mimic}).
    \item[] Guidelines:
    \begin{itemize}
        \item The answer NA means that the paper does not include experiments.
        \item The paper should indicate the type of compute workers CPU or GPU, internal cluster, or cloud provider, including relevant memory and storage.
        \item The paper should provide the amount of compute required for each of the individual experimental runs as well as estimate the total compute. 
        \item The paper should disclose whether the full research project required more compute than the experiments reported in the paper (e.g., preliminary or failed experiments that didn't make it into the paper). 
    \end{itemize}
    
\item {\bf Code Of Ethics}
    \item[] Question: Does the research conducted in the paper conform, in every respect, with the NeurIPS Code of Ethics \url{https://neurips.cc/public/EthicsGuidelines}?
    \item[] Answer: \answerYes{} 
    \item[] Justification:  Authors acknowledge conducting research in conformity with the NeurIPS Code of Ethics.
    \item[] Guidelines:
    \begin{itemize}
        \item The answer NA means that the authors have not reviewed the NeurIPS Code of Ethics.
        \item If the authors answer No, they should explain the special circumstances that require a deviation from the Code of Ethics.
        \item The authors should make sure to preserve anonymity (e.g., if there is a special consideration due to laws or regulations in their jurisdiction).
    \end{itemize}

\item {\bf Broader Impacts}
    \item[] Question: Does the paper discuss both potential positive societal impacts and negative societal impacts of the work performed?
    \item[] Answer: \answerYes{} 
    \item[] Justification: An impact statement is included in Appendix \ref{sect: impact}.
    \item[] Guidelines:
    \begin{itemize}
        \item The answer NA means that there is no societal impact of the work performed.
        \item If the authors answer NA or No, they should explain why their work has no societal impact or why the paper does not address societal impact.
        \item Examples of negative societal impacts include potential malicious or unintended uses (e.g., disinformation, generating fake profiles, surveillance), fairness considerations (e.g., deployment of technologies that could make decisions that unfairly impact specific groups), privacy considerations, and security considerations.
        \item The conference expects that many papers will be foundational research and not tied to particular applications, let alone deployments. However, if there is a direct path to any negative applications, the authors should point it out. For example, it is legitimate to point out that an improvement in the quality of generative models could be used to generate deepfakes for disinformation. On the other hand, it is not needed to point out that a generic algorithm for optimizing neural networks could enable people to train models that generate Deepfakes faster.
        \item The authors should consider possible harms that could arise when the technology is being used as intended and functioning correctly, harms that could arise when the technology is being used as intended but gives incorrect results, and harms following from (intentional or unintentional) misuse of the technology.
        \item If there are negative societal impacts, the authors could also discuss possible mitigation strategies (e.g., gated release of models, providing defenses in addition to attacks, mechanisms for monitoring misuse, mechanisms to monitor how a system learns from feedback over time, improving the efficiency and accessibility of ML).
    \end{itemize}
    
\item {\bf Safeguards}
    \item[] Question: Does the paper describe safeguards that have been put in place for responsible release of data or models that have a high risk for misuse (e.g., pretrained language models, image generators, or scraped datasets)?
    \item[] Answer: \answerNA{}{} 
    \item[] Justification:
    \item[] Guidelines:
    \begin{itemize}
        \item The answer NA means that the paper poses no such risks.
        \item Released models that have a high risk for misuse or dual-use should be released with necessary safeguards to allow for controlled use of the model, for example by requiring that users adhere to usage guidelines or restrictions to access the model or implementing safety filters. 
        \item Datasets that have been scraped from the Internet could pose safety risks. The authors should describe how they avoided releasing unsafe images.
        \item We recognize that providing effective safeguards is challenging, and many papers do not require this, but we encourage authors to take this into account and make a best faith effort.
    \end{itemize}

\item {\bf Licenses for existing assets}
    \item[] Question: Are the creators or original owners of assets (e.g., code, data, models), used in the paper, properly credited and are the license and terms of use explicitly mentioned and properly respected?
    \item[] Answer: \answerYes{} 
    \item[] Justification: Authors of models and datasets are appropriately cited in the paper in the introduction (\ref{sect: intro}) and experiment (\ref{sect: experiments}) sections. Original owners of some model implementations are properly credited in our code. 
    \item[] Guidelines:
    \begin{itemize}
        \item The answer NA means that the paper does not use existing assets.
        \item The authors should cite the original paper that produced the code package or dataset.
        \item The authors should state which version of the asset is used and, if possible, include a URL.
        \item The name of the license (e.g., CC-BY 4.0) should be included for each asset.
        \item For scraped data from a particular source (e.g., website), the copyright and terms of service of that source should be provided.
        \item If assets are released, the license, copyright information, and terms of use in the package should be provided. For popular datasets, \url{paperswithcode.com/datasets} has curated licenses for some datasets. Their licensing guide can help determine the license of a dataset.
        \item For existing datasets that are re-packaged, both the original license and the license of the derived asset (if it has changed) should be provided.
        \item If this information is not available online, the authors are encouraged to reach out to the asset's creators.
    \end{itemize}

\item {\bf New Assets}
    \item[] Question: Are new assets introduced in the paper well documented and is the documentation provided alongside the assets?
    \item[] Answer: \answerYes{} 
    \item[] Justification: Code for Causal CPC is provided in the supplementary material.
    \item[] Guidelines:
    \begin{itemize}
        \item The answer NA means that the paper does not release new assets.
        \item Researchers should communicate the details of the dataset/code/model as part of their submissions via structured templates. This includes details about training, license, limitations, etc. 
        \item The paper should discuss whether and how consent was obtained from people whose asset is used.
        \item At submission time, remember to anonymize your assets (if applicable). You can either create an anonymized URL or include an anonymized zip file.
    \end{itemize}

\item {\bf Crowdsourcing and Research with Human Subjects}
    \item[] Question: For crowdsourcing experiments and research with human subjects, does the paper include the full text of instructions given to participants and screenshots, if applicable, as well as details about compensation (if any)? 
    \item[] Answer: \answerNA{}{} 
    \item[] Justification:
    \item[] Guidelines:
    \begin{itemize}
        \item The answer NA means that the paper does not involve crowdsourcing nor research with human subjects.
        \item Including this information in the supplemental material is fine, but if the main contribution of the paper involves human subjects, then as much detail as possible should be included in the main paper. 
        \item According to the NeurIPS Code of Ethics, workers involved in data collection, curation, or other labor should be paid at least the minimum wage in the country of the data collector. 
    \end{itemize}

\item {\bf Institutional Review Board (IRB) Approvals or Equivalent for Research with Human Subjects}
    \item[] Question: Does the paper describe potential risks incurred by study participants, whether such risks were disclosed to the subjects, and whether Institutional Review Board (IRB) approvals (or an equivalent approval/review based on the requirements of your country or institution) were obtained?
    \item[] Answer: \answerNA{}{} 
    \item[] Justification:
    \item[] Guidelines:
    \begin{itemize}
        \item The answer NA means that the paper does not involve crowdsourcing nor research with human subjects.
        \item Depending on the country in which research is conducted, IRB approval (or equivalent) may be required for any human subjects research. If you obtained IRB approval, you should clearly state this in the paper. 
        \item We recognize that the procedures for this may vary significantly between institutions and locations, and we expect authors to adhere to the NeurIPS Code of Ethics and the guidelines for their institution. 
        \item For initial submissions, do not include any information that would break anonymity (if applicable), such as the institution conducting the review.
    \end{itemize}

\end{enumerate}

\end{document}